\documentclass[aos]{imsart}
\usepackage{newtxtext}
\usepackage{fix-cm}

\RequirePackage{amsthm,amsmath,amsfonts,amssymb}
\RequirePackage[colorlinks, linkcolor=blue, citecolor=blue, urlcolor=blue]{hyperref}

\usepackage[authoryear]{natbib}

\RequirePackage{graphicx}
\usepackage{float}

\usepackage{pgfplots}
\usepackage{bbm}
\usepackage{latexsym}
\usepackage{caption}
\usepackage[margin=20pt]{subcaption}
\usepackage{hyperref, enumerate, xcolor}

\usepackage{tikz}
\usetikzlibrary{positioning}
\usepgflibrary{plotmarks} 
\usepgflibrary[plotmarks] 
\usetikzlibrary{plotmarks} 
\usetikzlibrary[plotmarks] 

\startlocaldefs

\theoremstyle{plain}

\newtheorem{theorem}{Theorem}
\newtheorem{lemma}{Lemma}
\newtheorem{proposition}{Proposition}
\newtheorem{corollary}{Corollary}

\theoremstyle{remark}
\newtheorem{definition}{Definition}
\newtheorem{example}{Example}
\newtheorem{assumption}{Assumption}
\newtheorem{remark}{Remark}
\newtheorem{condition}{Condition}

\newcommand{\pr}{\mathbb{P}}
\newcommand{\E}{\mathbb{E}}

\newcommand\y[1]{%
  \colorbox{green!10}{$#1$}%
}

{%
  \ifodd\value{page}
    \leftmark 
  \else
    \rightmark 
  \fi
}
\renewcommand{\leftmark}{J. FAN, C. GAO AND J. M. KLUSOWSKI}
\renewcommand{\rightmark}{ROBUST TRANSFER LEARNING}

\usepackage{color}
\usepackage{listings}

\DeclareMathOperator*{\argmin}{arg\,min}

\endlocaldefs
\pgfplotsset{compat=1.18}
\begin{document}

\begin{frontmatter}
\title{Robust Transfer Learning with Unreliable Source Data}


\begin{aug}
\author{\fnms{Jianqing}~\snm{Fan}\ead[label=1]{jqfan@princeton.edu}},
\author{\fnms{Cheng}~\snm{Gao}\ead[label=2]{chenggao@princeton.edu}}
\and
\author{\fnms{Jason}~\snm{M. Klusowski}\ead[label=3]{jason.klusowski@princeton.edu}}

\address{Department of Operations Research and Financial Engineering, Princeton University \\
\printead[presep={}]{1,2,3}}
\end{aug}

\begin{abstract}
This paper addresses challenges in robust transfer learning stemming from ambiguity in Bayes classifiers and weak transferable signals between the target and source distribution. We introduce a novel quantity called the ``ambiguity level'' that measures the discrepancy between the target and source regression functions, propose a simple transfer learning procedure, and establish a general theorem that shows how this new quantity is related to the transferability of learning in terms of risk improvements.   Our proposed ``Transfer Around Boundary'' (TAB) method, with a threshold balancing the performance of target and source data, is shown to be both efficient and robust, improving classification while avoiding negative transfer. Moreover, we demonstrate the effectiveness of the TAB model on non-parametric classification and logistic regression tasks, achieving upper bounds which are optimal up to logarithmic factors. Simulation studies lend further support to the effectiveness of TAB. We also provide simple approaches to bound the excess misclassification error without the need for specialized knowledge in transfer learning.
\end{abstract}



\end{frontmatter}

\section{Introduction}
\label{sec:intro}
Previous experiences can offer valuable information for learning new tasks. Human learners often transfer their existing knowledge gained from previous tasks to new and related ones. \emph{Transfer learning} refers to statistical learning tasks where a portion of the training data is generated from a similar but non-identical distribution to the data distribution for which we seek to make inferences about. The objective is then to transfer knowledge from such source data to improve learning in the related target task. Such problems, where there is a divergence between the data-generating distributions, arise in many real application problems, including computer vision \citep{li2020TLinCV,tzeng2017TLinCV}, natural language processing \citep{ruder-etal-2019-transfer,wang2015TLinNLP}, speech recognition \citep{huang2013TLinSpeechRecognition}, and genre classification \citep{choi2017TransferLF}. See \cite{storkey2009tlSurvey}, \cite{pan2009tlSurvey}, and \cite{weiss2016tlSurvey} for an overview. Also, similar problems have been studied by both statisticians and many other communities under different names, including label noise \citep{frenay2014labelNoise,scott2013labelNoise,cannings2020labelNoise,blanchard2017labelNoise,reeves2019labelNoise, scottZhang2019labelNoise}, domain adaptation \citep{Scott2018AGN, bendavid2010domainAdaptation,bendavid2010domainAdaptation2,mansour2009domainAdaptation}, multi-task learning \citep{caruana1997multitaskLearning,maurer2016multitaskLearning}, or distributional robustness \citep{sinha2018distributionalRobustness,christiansen2020distributionalRobustness}.

We focus here on the transfer learning setting in the context of binary classification since it is not only fundamental in statistical learning and has been extensively investigated in diverse contexts, but also because it provides a framework that is particularly conducive to algorithms that seek to exploit relationships between the target and source distributions. For transfer learning theory of linear regression models, see \cite{chen2013tlRegression,bastani2020tlRegression} under the setting of finite covariate dimensions or \cite{gross2016dataSharedLasso, ollier2017dataSharedLasso,ttcai2020highDimensional} under the high-dimensional regime with lasso-based penalties. See also \cite{tian2022tlGLMs} for generalized linear models and \cite{ttcai2022nonParametricRegression} for non-parametric regression.

To set up the framework, suppose that a labeled data sample (relatively \emph{small} in size, typically) is drawn from the $Q$, the target distribution we wish to make statistical inferences about. Also, let $P$ be the source distribution from which we wish to transfer knowledge. We suppose a labeled data sample (relatively \emph{large} in size, typically) is drawn from $P$. The corresponding random pairs of $Q$ and $P$ distributions are denoted by $(X,Y)$ and $(X^P,Y^P)$ on $\mathbb{R}^d\times \{0,1\}$ respectively. Let $\eta^Q,\eta^P:\mathbb{R}^d\rightarrow[0,1]$ denote the target and source regression functions, i.e.
$$\eta^Q(x)=Q(Y=1 \mid X=x)\text{ and }\eta^P(x)=P(Y^P=1 \mid X^P=x).$$
The key question is how much information or knowledge can be transferred from $P$ to $Q$ given observations from both distributions. Note that the Bayes classifier $f^*_Q(\cdot)=\mathbf{1}\{\eta^Q(\cdot)\geq \frac{1}{2}\}$ minimizes the misclassification rate $Q(Y\neq f(X))$ over all classifiers. Therefore, we define the excess risk of any empirical classifier $\hat{f}$ as $$\mathcal{E}_Q(\hat{f})=Q(Y\neq \hat{f}(X))-Q(Y\neq f^*_Q(X)).$$ 
The aforementioned key question leads us to the task of constructing an empirical classifier that accelerates the convergence of excess risk to zero  by utilizing labeled data samples from \emph{both} $Q$ and $P$.

\subsection{Related Literature}
Recent research has aimed to bridge the gap between limited theoretical understanding and significant practical achievements in transfer classification learning. To set up the problem, it is necessary to make some assumptions about the similarity between the target and source distributions, which are both useful for theory and practical implementation. Various approaches have been proposed and explored in the literature to measure this similarity, including divergence bounds, covariate shift, and label shift. Some methods focus on test error bounds that rely on measures of discrepancy between $Q$ and $P$, such as modified total-variation or R\'{e}nyi divergence, between the target and source distributions \citep{bendavid2010domainAdaptation,bendavid2010domainAdaptation2,mansour2009domainAdaptation,mansour2009renyiDivergence,germain2013domainAdaptation,cortes2019domainAdaptation}. This line of work produces distribution-free risk rates, primarily expressed in terms of $n_P$ alone, but the obtained rates do not converge to zero with increasing sample size. In other words, these frameworks cannot produce a faster convergence of excess risk if their proposed divergences are non-negligible. Nonetheless, consistent classification is proved achievable regardless of a non-negligible divergence even when $n_Q=0$, provided that certain additional structures on the target and source distributions are present \citep{bendavid2010domainAdaptation2}.

Two common additional structures include covariate shift and label shift (or posterior shift). Covariate shift \citep{gretton2008covariateShift,candela2009covariateShift,kpotufe2018covariateShift} considers scenarios where the conditional distributions of the response given the covariates are identical across $Q$ and $P$ or  $\eta^Q=\eta^P$, but marginal distributions of covariates are different. Label shift, on the other hand, assumes an identical or similar target and source marginal distributions $Q_X$ and $P_X$, but the conditional probabilities $\eta^Q$ and $\eta^P$ differ.

Most previous work in label shift can be divided into two branches. On one hand, some frameworks do not require identical Bayes classifiers, i.e., $(\eta^P-1/2)(\eta^Q-1/2)\geq 0$, but impose very specific relations between $\eta^Q$ and $\eta^P$, such as the literature on \emph{label noise} \citep{frenay2014labelNoise,scott2013labelNoise,cannings2020labelNoise,blanchard2017labelNoise,reeves2019labelNoise, nagarajan2018labelNoise, scottZhang2019labelNoise}. This literature typically considers the labels of the source domain as ``corrupted'' in a certain way compared to the true labels in the target distribution, with the corrupted labels usually coupled with the true labels given the same feature data $X$. For instance, a common assumption \citep{reeves2019labelNoise, nagarajan2018labelNoise} is that the $Y^P \mid X^P=x$ is equal to $Y \mid X=x$ up to a constant probability of label flipping, i.e., in our terminology
$$P(Y^P=1 \mid Y=0,X^P=X=x)=\pi_0\ \text{ and }\ P(Y^P=0 \mid Y=1,X^P=X=x)=\pi_1,$$
for some constants $\pi_0,\pi_1\in(0,1)$. Under this setting, $\eta^P=(1-\pi_0-\pi_1)\eta^Q+\pi_0$, which is linear in $\eta^Q$. Given knowledge of the specific form of $\eta^P$, it is feasible to modify learning algorithms to efficiently infer the target Bayes classifier. In another special type of label shift problem, \cite{maity2022labelShift} assumes that $P_{X\mid Y}=Q_{X\mid Y}$, which is convenient for estimating the joint distribution of $(X, Y)$. However, all their proposed estimators are tailored to fit specific assumptions and may not be applicable to more general relations between $\eta^Q$ and $\eta^P$. On the other hand, recent work such as \cite{ttcai2020classification} and \cite{hanneke2019onTheTargetValue} have introduced more general label shift settings that impose relatively mild and general conditions on the relation between $\eta^P$ and $\eta^Q$, in addition to assuming identical Bayes classifiers. \cite{ttcai2020classification} requires a lower-bounded signal strength of $\eta^P$ relative to $\eta^Q$, besides the assumption of identical Bayes classifiers, i.e.,
$$
\Big(\eta^P-\frac{1}{2}\Big) \Big(\eta^Q-\frac{1}{2}\Big) > 0, \qquad
\Big|\eta^P-\frac{1}{2}\Big|\geq C_\gamma\Big|\eta^Q-\frac{1}{2}\Big|^\gamma,$$ for some positive $\gamma$ and $C_\gamma$, and derive a faster risk convergence rate with the transfer exponent $\gamma$. 

The aforementioned approaches have limitations in that they rely on specific and often untestable assumptions. More importantly, they may not be effective in situations where both ambiguities of the source data hold, i.e., there are no strong relations between $\eta^Q$ and $\eta^P$ and discrepancies between the Bayes classifiers of the target and source domains still exist. The only other work of which we are aware that allows such general ambiguous source data is \cite{reeve2021adpative}. \cite{reeve2021adpative} assumes that $\eta^P$ can be well approximated by a set of regression functions $g_1(\eta^Q),\dots,g_{L^*}(\eta^Q)$ that are no less informative than a linear transformation of $\eta^Q$.

\subsection{Main Contribution} In contrast to existing work, this paper considers a scenario where the Bayes classifiers can arbitrarily differ without imposing further conditions on the source and target distributions. We now introduce an important quantity for capturing the relative information transferred from the source data, which we refer to as the signal strength:

$$
  s(x):=
\begin{cases}
 |\eta^P(x)-\frac{1}{2}|,\quad &\mathrm{sgn}\left(\eta^Q(x)-\frac{1}{2}\right)\times(\eta^P(x)-\frac{1}{2})\geq 0,\\
0,\quad &\text{otherwise.}   
\end{cases}
$$
Our main assumption involves measuring the ambiguity level of the source data based on the expected signal strength  around the classification boundary $\{x: |\eta^Q(x)-\frac{1}{2}|\leq z\}$ for small $z$:
$$ 
    \E_{(X,Y)\sim Q}\left[\Big|\eta^Q(X)-\frac{1}{2}\Big|\mathbf{1}\left\{s(X)\leq C_\gamma\Big|\eta^Q(X)-\frac{1}{2}\Big|^\gamma, \Big|\eta^Q(X)-\frac{1}{2}\Big|\leq z\right\}\right] \leq \varepsilon(z),
$$ 
given some constants $\gamma,C_\gamma>0$.   This quantity captures the inevitable risk from hard-to-classify boundary points, which have $\eta^Q$ values too close to $1/2$ and $\eta^P$ and show weak signal relative to $\eta^Q$. The explicit form of $\varepsilon(z)$ is provided for some special examples in Section \ref{sec:ambiguityBound}.
If $s(x)\geq C_\gamma|\eta^Q(x)-1/2|^\gamma$ covers all of $\Omega$, then $\varepsilon(z) = 0$ and the setting reduces to that of \cite{ttcai2020classification} with a strong relative signal of $\eta^P$ across the entire feature space $\Omega$. 

Given the measure above, we propose a simple but effective classifier that can surprisingly adapt to any level of ambiguity in the signal conveyed by $\eta^P$, named the \emph{Transfer Around Boundary} (TAB) classifier, or the TAB model:
$$
\hat{f}_{TAB}(x)=
\begin{cases}
\mathbf{1}\{\hat{\eta}^Q(x)\geq \frac{1}{2}\},\quad &\text{if }|\hat{\eta}^Q(x) - \frac{1}{2}|\geq \tau,\\
\hat{f}^P(x),\quad &\text{otherwise},
\end{cases}
$$
where $\hat{\eta}^Q$ is an estimate of $\eta^Q$ obtained by the target data and $\hat{f}^P$ is a classifier obtained by the source data. To clarify our decision-making process, we rely on $\hat{\eta}^Q$ as the final prediction when it deviates from $1/2$. Otherwise, we switch to the prediction made by the source data. Our TAB classifier is easy to implement and does not require careful classifier design.

In Section \ref{sec:convergenceResultByConcentration}, we present two related general convergence theorems to showcase this classifier, with a proper choice of $\tau>0$, utilizing both the relative signal of $\eta^P$ and the ambiguity level $\varepsilon(z)$ from the source data. The target data involved in this classifier is noteworthy for its dual role: not only does it conserve a classical upper bound on the excess risk if the source data is unreliable, but it also helps to rule out the source data ambiguity when $\eta^Q$ is far from $1/2$, considering only necessary ambiguity in the risk bounds.

We then apply our general convergence results to non-parametric classification as well as logistic regression to illustrate the potential of our transfer learning method on parametric classification. The results for these two important cases show that our method is indeed nearly minimax optimal and our measure of the ambiguity level is effective, extending results in previous literature.

\vspace{\baselineskip}
\noindent\textbf{Non-parametric Classification:} Suppose that $\eta^Q$ is $\beta$-smooth (Condition \ref{con:smooth}), margin assumption holds for $\eta^Q$ with parameter $\alpha$ (Assumption \ref{assum:margin}), and strong density condition (Condition \ref{con:density}) holds for $Q_X$ and $P_X$. Let $\Pi^{NP}$ denote the set of all such distribution pairs $(Q, P)$. 

At the moment, we consider the scenario where $\eta^P$ is sufficiently smooth. Suppose that $\Pi_S^{NP}$ is a subset of $\Pi^{NP}$ such that $\eta^P$ is $\beta_P$-smooth with $\beta_P \geq \gamma\beta$. Then, we show that the minimax excess risk satisfies
$$
\begin{aligned}
 &\left( n_P^{-\frac{\beta(1+\alpha)}{2\gamma\beta+d}}+\varepsilon \Bigl  (cn_Q^{-\frac{\beta}{2\beta+d}}\Bigr )\right)\land n_Q^{-\frac{\beta(1+\alpha)}{2\beta+d}}\lesssim \inf_{\hat{f}}\sup_{(Q,P)\in\Pi_S^{NP}}\E\mathcal{E}_Q(\hat{f})\\
 \lesssim&\left( n_P^{-\frac{\beta(1+\alpha)}{2\gamma\beta+d}}+\varepsilon\Bigl (c\log (n_Q\lor n_P) n_Q^{-\frac{\beta}{2\beta+d}}\Bigr )\right)\land \left(\log^{1+\alpha} (n_Q\lor n_P) n_Q^{-\frac{\beta(1+\alpha)}{2\beta+d}}\right),  
\end{aligned}
$$
for some constant $c>0$ that is independent of $\alpha$.

With the optimal choice of $\tau\asymp \log(n_Q\lor n_P) n_Q^{-\frac{\beta}{2\beta+d}}$, the upper bound we obtain using the TAB classifier with $K$-NN classifier components is optimal up to some logarithmic terms of $n_Q\lor n_P$. This indicates that a necessary and sufficient condition for the source data to improve the excess risk rate is a large source data sample size such that $n_P\gg n_Q^{\frac{2\gamma\beta+d}{2\beta+d}}$, paired with a small ambiguity level $\varepsilon(z)\ll z^{1+\alpha}$.

Additionally, we provide the optimal rate considering the ``band-like" ambiguity condition between $\eta^Q$ and $\eta^P$. A particularly interesting example is to assume that $\sup_{x\in\Omega}|\eta^P(x)-\eta^Q(x)|\leq \Delta$, as it allows for effective transfer learning  by ensuring only slight deviations across domains. In this case, the minimax optimal excess risk becomes $\left(n_P^{-\frac{\beta(1+\alpha)}{2\beta+d}}+\Delta^{1+\alpha}\right)\land n_Q^{-\frac{\beta(1+\alpha)}{2\beta+d}}$, up to some logarithmic terms of $n_Q\lor n_P$. Note that we do not require any smoothness condition on $\eta^P$ in the band-like ambiguity case.

\vspace{\baselineskip}
\noindent\textbf{Logistic Regression:} If the covariates follow the standard normal random design (see \eqref{eq:logisticFamily} for a precise definition),
and the target and source logistic regression coefficient pair $(\beta_Q,\beta_P)$ belongs to
$$
    \Theta(s,\Delta)=\left\{(\beta_Q,\beta_P):\|\beta_Q\|_0\leq s,\angle ( \beta_Q,\beta_P )\leq \Delta \right\},
$$
for some $0\leq\Delta\leq \pi/2$, where $\angle ( \beta_Q,\beta_P )$ denotes the angle between directions $\beta_Q$ and $\beta_P$, we show that the minimax optimal excess risk satisfies
$$
\left(\frac{s\log d}{n_P} + \Delta^2\right)\land \frac{s\log d}{n_Q}\lesssim \inf_{\hat{f}}\sup_{(Q,P)\in\Pi^{LR}}\E\mathcal{E}_Q(\hat{f})\lesssim \left(\frac{s\log d}{n_P} + \Delta^2\right)\land \left(\frac{s\log d}{n_Q}\log^2(n_Q\lor n_P)\right).
$$
Thus, when $n_P\gg n_Q$ and $\Delta\ll \frac{s\log d}{n_Q}$, the source data improves the convergence rate of the excess risk.  Importantly, we only assume a ``small cone condition'' between $\beta_Q$ and $\beta_P$, in contrast to the ``small contrast condition'' in previous works \citep{ttcai2020highDimensional,tian2022tlGLMs}, i.e., $\|\beta_Q-\beta_P\|_q$ is small for some $q\in[0,1]$. In addition, unlike the small contrast condition, which implicitly assumes sparsity patterns of $\beta_P$ through $l_q$ norms with $q \leq 1$, our constructed parametric space does not impose any sparsity conditions on $\beta_P$. In the setting without access to source data, our upper bound $\frac{s\log d}{n_Q}$ is tighter than the $(\frac{s\log d}{n_Q})^{\frac{2}{3}}$ bound obtained in Theorem 7 of \cite{abramovich2019logistic}, assuming the same margin parameter $\alpha=1$. 

As evidenced by the above two examples, the TAB classifier maintains the performance of the target data against an unreliable source and enhances the performance, if the source data sample size is large and the ambiguity is relatively small, i.e., $n_P\gg n_Q$ and $\Delta\ll \sqrt{\frac{s\log d}{n_Q}}$.

While the target excess risk provided by the target data, $\mathcal{E}_Q(\hat{f}^Q)$, has been extensively studied in the literature, much less is known about the target excess risk $\mathcal{E}_Q(\hat{f}^P)$ provided by the source data, which is one key component in deducing the excess risk bound. To fill this gap and gain better insight into the contribution of the source data, we present a general result in Section \ref{sec:extension} that provides a direct upper bound on $\mathcal{E}_Q(\hat{f}^P)$ in terms of $\mathcal{E}_P(\hat{f}^P)$, which can be obtained through conventional theoretical analysis. Therefore, it is feasible to bound the excess risk using conventional statistical learning tools, without requiring specialized expertise in transfer learning. By providing such a general and accessible theoretical framework, our approach has the potential to make transfer learning more accessible.

\subsection{Notation and Organization}
We introduce some notation to be used throughout the paper. For any $q\in[0,\infty]$ and vector $x=(x_1,\cdots,x_d)\in\mathbb{R}^d$, we write $\|x\|_q$ for its $l_q$ norm. We write $\|x\|$ or $\|x\|_2$ for the Euclidean norm of $x$, and, given $r>0$, we write $B(x,r)$ or $B_d(x,r)$ for the closed Euclidean sphere of radius $r$ centered at $x$. For two probability measures $\mu,\nu$ on any general space, if $\mu$ is absolutely continuous with respect to $\nu$, we write $\frac{d\mu}{d\nu}$ for the Radon-Nikodyn derivative of $\mu$ with respect to $\nu$. Write $\lambda$ as the Lebesgue measure on $\mathbb{R}^d$. Let $a\land b$ and $a\lor b$ denote the minimum and maximum of $a$ and $b$, respectively. Let $a_{n}\lesssim b_{n}$ denote $|a_{n}|\leq c|b_{n}|$ for some constant $c>0$ when $n$ is large enough. Let $a_{n}\gtrsim b_{n}$ denote $|a_{n}|\geq c|b_{n}|$ for some constant $c>0$ when $n$ is large enough. Let $a_{n}\asymp b_{n}$ denote $1/c\leq |a_{n}|/|b_{n}|\leq c$ for some constant $c>1$ when $n$ is large enough. Let $a_{n}\overset{n\rightarrow\infty}{\longrightarrow} \infty$ denote that $a_{n}$ tends to infinity with $n$ growing to infinity. Let $a_{n}\ll b_{n}$ denote $|a_{n}|/|b_{n}|\rightarrow 0$ when $n$ is large enough. Let $a_{n}\gg b_{n}$ denote $|b_{n}|/|a_{n}|\rightarrow 0$ when $n$ is large enough. Let $\lfloor a\rfloor$ be the maximum integer that is less equal than $a$ for any real value $a$. Finally, we assume $0^0=0$ for simplicity.

We state our main working assumptions and measure of ambiguity, called the \emph{ambiguity level}, in Section \ref{sec:model}. We provide two general convergence results on the risk of transfer learning in Section \ref{sec:convergenceResultByConcentration}. Section \ref{sec:extension} presents an approach to bound the signal transfer risk, a crucial part of general convergence results, in terms of the excess risk rate studied in the conventional statistical learning literature.
In Sections \ref{sec:nonPara} and \ref{sec:para}, we apply our results to non-parametric classification and logistic regression, respectively, and provide upper and lower bounds on the excess risk. In Section \ref{sec:simulation}, we present simulation results for non-parametric classification and logistic regression, supporting the theoretical properties of our proposed method.

\section{Model}
\label{sec:model}
\subsection{Problem Formulation}
\label{sec:problem}
For two Borel-measurable distributions $P$ and $Q$, both taking values in $\mathbb{R}^d\times\{0,1\}$, we observe two independent random samples, the \emph{source} data $\mathcal{D}_P=\{(X_1^P,Y_1^P)$, $\cdots,(X_{n_P}^P,Y_{n_P}^P)\}\overset{\text{iid}}{\sim}P$ and the \emph{target} data $\mathcal{D}_Q=\{(X_1,Y_1),\cdots,(X_{n_Q},Y_{n_Q})\}
\overset{\text{iid}}{\sim}Q$. Suppose that $n_P\overset{n_Q\rightarrow\infty}{\longrightarrow} \infty$. Our goal is to improve the target data empirical classifier by transferring useful information from the source data. Consider the marginal probability distributions of $X$ for the $P$ and $Q$ distributions, denoted by $P_X$ and $Q_X$ respectively. Let $\Omega_P:=\mathrm{supp}(P_X)$ and $\Omega:=\mathrm{supp}(Q_X)$ represent the support sets of $P_X$ and $Q_X$. The regression functions for the source and target distributions are respectively defined as follows:
\begin{equation}
\eta^P(x)=
\begin{cases}
P(Y^P=1 \mid X^P=x), & x \in \Omega_P,\\[6pt]
1/2, & \text{otherwise},
\end{cases}
\quad
\eta^Q(x)=
\begin{cases}
Q(Y^P=1 \mid X^P=x), & x \in \Omega,\\[6pt]
1/2, & \text{otherwise}.
\end{cases}
\label{eq:RegFuncPQ}
\end{equation}

The goal of a classification model is to forecast the label $Y$ based on the value $X$. The effectiveness of a decision rule $f:\mathbb{R}^d\rightarrow \{0,1\}$ is evaluated by its misclassification rate with respect to the target distribution, which is defined as follows:
$$R(f):=Q(Y\neq f(X)).$$
The Bayes classifier (or Bayes estimator, Bayes decision rule) $f^*_Q(x)=\mathbf{1}\{\eta^Q(x)\geq 1/2\}$ is the minimizer of $R(f)$ over all Borel functions defined on $\mathbb{R}^d$ and taking values in $\{0,1\}$. We similarly define the Bayes decision rule for $\eta^P$ that is $f^*_P(x)=\mathbf{1}\{\eta^P(x)\geq 1/2\}$. Since the Bayes decision rule $f^*_Q(x)$ is a minimizer of the misclassification rate $R(f)$, the performance of any empirical classifier $\hat{f}:\mathbb{R}^d\rightarrow\{0,1\}$ can be then measured by the excess risk (on the target distribution): 
\begin{equation}
    \mathcal{E}_Q(\hat{f})=R(\hat{f})-R(f^*_Q)=2\E_{(X,Y)\sim Q}\left[
    \Big|\eta_Q(X)-\frac{1}{2} \Big|\mathbf{1}\left\{\hat{f}(X)\neq f_Q^*(X)\right\}\right].
    \label{eq:excessRisk}
\end{equation}
The last equality is the dual representation of the excess risk \citep{gyorfi1978dual}. Given the excess risk defined in \eqref{eq:excessRisk}, the objective of transferring useful information from the source data can be reformulated as the task of constructing an empirical classifier that improves the excess risk, which is accomplished by utilizing labeled data samples drawn from both the target and source distributions.

The rate at which the excess risk $\mathcal{E}_Q(\hat{f})$ converges to zero depends on the assumptions made about the target and source distributions $(Q,P)$. In classification, a common assumption is the margin assumption \citep{atsybakov2007fastPlugin, tsybakov2004smoothDiscriminationAnalysis}. This assumption is used to measure the behavior of $Q_X$ with respect to the distance between $\eta^Q(X)$ and $1/2$, which is essential for determining the convergence rate of the excess risk.

\begin{assumption}[Margin]
There exists some constant $\alpha\geq 0,\ C_\alpha>0$ such that for any $t>0$, we have $Q_X(0<|\eta^Q(X)-1/2|\leq t)\leq C_\alpha t^\alpha$.
\label{assum:margin}
\end{assumption}

Instead of assuming identical Bayes classifiers for $\eta^Q$ and $\eta^P$ over $x\in\Omega$ as in some existing literature \citep{hanneke2019onTheTargetValue,ttcai2020classification}, we allow them to differ. This is a natural and necessary relaxation because, in almost every real-world application, although the target and source domains may exhibit high similarity, they usually still have slightly different decision boundaries. Therefore, it is crucial to assess the impact of an unreliable source distribution on the optimal rates of classification.

\subsection{Source Data Ambiguity}
\label{subsec:ambiguity}
In this subsection, we provide a detailed discussion of the condition that characterizes the ambiguity of an unreliable source distribution. Here we introduce the signal strength function, which measures the relative signal of $\eta^P$ compared with $\eta^Q$. This function is crucial in capturing the efficacy of the source data for the classification task under $Q$.

\begin{definition}[Signal Strength] \label{def:signalStrength}
The signal strength of $\eta^P$ relative to $\eta^Q$ is defined as
$$
\begin{aligned}
  s(x):&=\left\{\mathrm{sgn}\Big(\eta^Q(x)-\frac{1}{2}\Big)\times\Big(\eta^P(x)-\frac{1}{2}\Big)\right\}\lor 0\\  
  &=
\begin{cases}
 |\eta^P(x)-\frac{1}{2}|,\quad &\mathrm{sgn}\left(\eta^Q(x)-\frac{1}{2}\right)\times(\eta^P(x)-\frac{1}{2})\geq 0,\\
0,\quad &\text{otherwise},   
\end{cases}
\end{aligned}
$$
for any $x\in\Omega$.
\end{definition}

It is reasonable to consider the signal strength as non-zero only when $(\eta^P(x)-1/2)(\eta^Q(x)-1/2)\geq 0$, indicating that the target and source data provide consistent information about the Bayes classification boundary. In this case, the source data $\eta^P(x)$ is beneficial for classifying the target data, and the signal strength is measured by $|\eta^P(x)-1/2|$. Conversely, when $(\eta^P(x)-1/2)(\eta^Q(x)-1/2)< 0$, the source data does not provide useful information for classifying $x$ in the target data, and the signal strength at $x$ is zero.

Next, we present the main working assumption that measures the ambiguity level of the source data, based on the signal strength. 

\begin{assumption}[Ambiguity Level] For some given \(\gamma ,C_\gamma>0\), the following condition holds for a specific continuous function \(\varepsilon(z;\gamma,C_\gamma)\) that is monotone increasing with \(z \in [0,1/2]\):
\begin{equation}
 \E_{(X,Y)\sim Q}\left[\Big|\eta^Q(X)-\frac{1}{2}\Big|\mathbf{1}\left\{s(X)\leq C_\gamma\Big|\eta^Q(X)-\frac{1}{2}\Big|^\gamma, \Big|\eta^Q(X)-\frac{1}{2}\Big|\leq z\right\}\right]\leq \varepsilon(z;\gamma,C_\gamma).
 \label{eq:ambiguity}
\end{equation}
We abbreviate $\varepsilon(z;\gamma,C_\gamma)$ as $\varepsilon(z)$ when there is no need to specify $\gamma$ and $C_\gamma$. 
\label{assum:ambiguity}
\end{assumption}

The expression in the expectation operator of \eqref{eq:ambiguity} can be divided into two factors. The first factor represents the distance between $\eta^Q(X)$ and $1/2$, which corresponds to the dual representation of the excess risk \eqref{eq:excessRisk}.  The second factor is crucial for understanding transfer learning ambiguity. On one hand, the constraint $s(X)\leq C_\gamma|\eta^Q(X)-1/2|^\gamma$ indicates a lack of strong signal from the source data relative to the target data. On the other hand, the constraint $|\eta^Q(X)-1/2|\leq z$ describes the hard-to-classify boundary points of the target distribution. The target data in this region cannot significantly aid in classification. Therefore, our indicator precisely captures the challenging boundary region where classification becomes difficult using either the target or source data.  As will be seen later, only the behavior around $z=0$ matters for the asymptotics of transfer learning.

The ambiguity level $\varepsilon(\cdot)$ allows for the presence of unreliable source data, as it accounts for situations where $\eta^P$ may not consistently provide a strong signal compared to $\eta^Q$. Additionally, $\varepsilon(\cdot)$ controls the ambiguity with respect to the distance from $1/2$. The larger the ambiguity level, the harder the classification task. Note that by the margin assumption, a trivial upper bound of the ambiguity level is $\varepsilon(z)=C_\alpha z^{1+\alpha}$.

\subsection{On the Ambiguity Level}
\label{sec:ambiguityBound}
To help clarify the ambiguity level assumption, we provide some examples where explicit formulas for the ambiguity level $\varepsilon(\cdot)$ are available, with a proper choice of $\gamma$ and $C_\gamma$.  We defer the case of two logistic regression models across the target and source distributions to Section \ref{sec:para}. 

\begin{example}[Perfect Source] \label{eg:perfectSource} Assume the condition of \emph{relative signal exponent} proposed in \cite{ttcai2020classification}, which amounts to $$s(x)\geq C_\gamma|\eta^Q(x)-\frac{1}{2}|^\gamma \quad \forall x\in\Omega.$$ Then Assumption \ref{assum:ambiguity} holds with $\varepsilon(z)\equiv 0$.

Specifically, if $\eta^P=\eta^Q$, it is straightforward to set $\gamma=1$ and $\varepsilon(\cdot)=0$. In this scenario, the Bayes classifiers are identical and $\eta^P$ gives strong signal compared to $\eta^Q$ over the whole support $\Omega$, so the ambiguity level is set as zero.
\end{example}

\begin{example}[Strong Signal over $\Omega_P$] \label{eg:STS} 
Instead of having a strong signal over the entire region as in Example \ref{eg:perfectSource}, here we have a strong signal over only a given region:
$$
s(x)\geq C_\gamma|\eta^Q(x)-\frac{1}{2}|^\gamma \quad \forall x\in\Omega_P.$$ 
Then Assumption \ref{assum:ambiguity} holds by taking
$$
    \varepsilon(z)\geq\E_{(X,Y)\sim Q}\left[\Big|\eta^Q(X)-\frac{1}{2}\Big|\mathbf{1}\left\{|\eta^Q(X)-\frac{1}{2}|\leq z, X\notin \Omega_P\right\}\right],
$$ 
i.e., the ambiguity level is controlled by the risk within the complement $\Omega/\Omega_P$. This corresponds to the common case where the source data is collected from a subpopulation with respect to the target distribution. 
\end{example}

\begin{example}[Strong Signal with Imperfect Transfer]  \label{eg:RSS}
Suppose that the transfer signal is strong, but the Bayes classifiers may differ, i.e.,
$$
   \Big|\eta^P(x)-\frac{1}{2}\Big|\geq C_\gamma \Big|\eta^Q(x)-\frac{1}{2}\Big|^\gamma,\quad \forall x\in \Omega,
$$ 
$$
    \Omega_R:=\left\{x\in\Omega:\Big(\eta^P(x)-\frac{1}{2}\Big)\Big(\eta^Q(x)-\frac{1}{2}\Big)<0\right\}\neq \emptyset.
$$
Then Assumption \ref{assum:ambiguity} holds if $$\varepsilon(z)\geq\E_{(X,Y)\sim Q}\left[\Big|\eta^Q(X)-\frac{1}{2}\Big|\mathbf{1}\left\{\Big|\eta^Q(X)-\frac{1}{2}\Big|\leq z, X\in \Omega_R\right\}\right].$$ In other words, the ambiguity level is determined by the risk associated with the region where different Bayes classifiers exist. A specific scenario that aligns with this example and utilizes $\gamma=C_\gamma=1$ can be expressed as $$\eta^P(x)=\eta^Q(x)\mathbf{1}\{x\notin \Omega_R\}+\left(1-\eta^Q(x)\right)\mathbf{1}\{x\in \Omega_R\},\quad \forall x\in \Omega,$$ where only the response values within $\Omega_R$ are flipped.
\end{example}

\begin{example}[Band-like Ambiguity] \label{eg:BA}
A further noteworthy scenario is when the probability distribution $\eta^P$ concentrates around a ``band" that is centered on an informative curve with respect to $\eta^Q$, but with some small deviation. A related situation is studied in \cite{reeve2021adpative}, where $\eta^P$ is approximated by a linear transfer function of $\eta^Q$. Suppose that there exists some band error constant $\Delta\geq 0$, which represents the deviation level, such that
\begin{equation}
 s(x)\geq C_\gamma\Big|\eta^Q(x)-\frac{1}{2}\Big|^\gamma-\Delta,
 \label{eq:BA}
\end{equation}
for any $x\in\Omega$. Then Assumption \ref{assum:ambiguity} holds with
\begin{equation}
\varepsilon(z;\gamma,C_\gamma/2)=\Big(C_\alpha z^{1+\alpha}\Big)\land \Big(2^{\frac{1+\alpha}{\gamma}}C_\alpha C_\gamma^{-\frac{1+\alpha}{\gamma}}\Delta^{\frac{1+\alpha}{\gamma}}\Big).
  \label{eq:BAAmbiguityLevel}
\end{equation}
\end{example}

Notably, the case of $\Delta=0$ degenerates into the perfect source scenario in Example \ref{eg:perfectSource}. For the proof of the statement in Example \ref{eg:BA}, see the Supplementary Material \citep{supp}. For simplicity, we assume that \eqref{eq:BA} holds over the entire feature space $\Omega$.


Specifically, if we further assume
\begin{equation}
   \sup_{x\in\Omega}|\eta^P(x)-\eta^Q(x)|\leq \Delta, 
   \label{eq:etaPetaQClose}
\end{equation}
then the band-like ambiguity condition \eqref{eq:BA} holds with $\gamma=C_\gamma=1$. This condition is common and meaningful in reality because it reflects a scenario where, despite differences in data collection contexts, the underlying relationship between features and outcomes remains nearly identical across domains. By \eqref{eq:BAAmbiguityLevel}, Assumption \ref{assum:ambiguity} holds in this special case with
$$\varepsilon(z,1,1/2)=\Big(C_\alpha z^{1+\alpha}\Big)\land \Big(2^{1+\alpha}C_\alpha C_\gamma^{-(1+\alpha)}\Delta^{1+\alpha}\Big).$$

\section{General Convergence Results}
\label{sec:generalConvergenceResult}
Let $\Pi$ be any given subset of distributions $(Q,P)$ satisfying that Assumptions \ref{assum:margin} and \ref{assum:ambiguity} with parameter $\alpha\geq 0$ and $\gamma, C_\alpha, C_\gamma>0$, and ambiguity level $\varepsilon(\cdot)$.  Our analysis focuses on the performance of any classifier when the target and source distribution pair $(Q,P)$ belongs to $\Pi$. This framework captures the essential information in the source data and how it can be used to improve the convergence rate of the excess risk.

\subsection{Performance of the TAB model}
\label{sec:convergenceResultByConcentration}
In preparation for the general result, it is necessary to define the risk learned by the source data over the region $s(x)\geq C_\gamma |\eta^Q(x)-1/2|^\gamma$ with strong signal strength.

\begin{definition}[Signal Transfer Risk] \label{def:STR}
Define the \emph{signal transfer risk} of the classifier $f$ with respect to parameters $\gamma,C_\gamma>0$ as
\begin{equation}
\xi(f;\gamma,C_\gamma):=\E_{(X,Y)\sim Q}\left[\Big|\eta^Q(X)-\frac{1}{2}\Big|\mathbf{1}\left\{f(X)\neq f_Q^*(X),s(X)\geq C_\gamma \Big|\eta^Q(X)-\frac{1}{2}\Big|^\gamma\right\}\right].
    \label{eq:STR}
\end{equation}
We abbreviate it as $\xi(f)$ when there is no need to specify $\gamma$ and $C_\gamma$.
\end{definition}

The signal transfer risk results from the classification of points belonging to the area where $s(x)\geq C_\gamma |\eta^Q(x)-1/2|^\gamma$. For simplicity, we assume the mild condition that $\sup_{(Q,P)\in\Pi}\E_{\mathcal{D}_P}\xi(\hat{f}^P)\gtrsim n_P^{-c}$ for some constant $c>0$ to prevent a convergence rate faster than polynomial. Due to the strong signals offered by the source data within this area, it is intuitively expected that the signal transfer risk can be reduced with the aid of the source data sample $\mathcal{D}_P$.
 
In this paper, the classifier derived from the target data is assumed to be a \emph{plug-in rule}, of the form $\mathbf{1}\{\hat{\eta}^Q(x)\geq 1/2\}$, where $\hat{\eta}^Q$ is an estimator of the regression function $\eta^Q$. By introducing a novel strategy called the \emph{TAB model}, the following result demonstrates a general approach to achieving a faster convergence rate of the excess risk.

\begin{theorem}
Let $\hat{\eta}^Q$ be an estimate of the regression function $\eta^Q$ and $\hat{f}^P$ be a classifier obtained by $\mathcal{D}_P$. Suppose there exist two sequences $\delta_Q, \delta_f$ such that $\delta_Q^{1+\alpha}\gtrsim n_Q^{-c}$ for some constant $c>0$ and $\alpha$ defined by the margin Assumption~\ref{assum:margin} and that with probability at least $1-\delta_f$,
    for any $x\in\Omega$ we have $\forall t>0$,
    \begin{equation}
        \sup_{(Q,P)\in\Pi}\pr_{\mathcal{D}_Q}\left(|\hat{\eta}^Q(x)-\eta^Q(x)| \geq t|X_{1:n_Q}\right)\leq C_1\exp\Big(- \Big(\frac{t}{\delta_Q}\Big)^2\Big),
        \label{eq:localConcentrationQ}
    \end{equation}
    for some constant $C_1>0$.
    Given the choice of $\tau\gtrsim \log(n_Q\lor n_P)\delta_Q$, the \emph{TAB classifier}
\begin{equation}
\hat{f}_{TAB}(x)=
\begin{cases}
\mathbf{1}\{\hat{\eta}^Q(x)\geq \frac{1}{2}\},\quad &\text{if }|\hat{\eta}^Q(x) - \frac{1}{2}|\geq \tau,\\
\hat{f}^P(x),\quad &\text{otherwise.}
\end{cases}
\label{eq:modelSelectedPlugInClassifier}
\end{equation}
satisfies, under the margin Assumption~\ref{assum:margin},
\begin{equation}
    \sup_{(Q,P)\in\Pi}\E_{(\mathcal{D}_Q,\mathcal{D}_P)} \mathcal{E}_Q(\hat{f}_{TAB})\lesssim \left(\sup_{(Q,P)\in\Pi}\E_{\mathcal{D}_P}\xi(\hat{f}^P)+ \varepsilon(2\tau)\right)\land \tau^{1+\alpha} +\delta_f.
    \label{eq:TLrate}
\end{equation}
\label{thm:general}
\end{theorem}

We now explain every important term in the upper bound \eqref{eq:TLrate}.

\begin{itemize}
    \item The term $\sup_{(Q,P)\in\Pi}\E_{\mathcal{D}_P}\xi(\hat{f}^P)$ captures the risk transferred by the source data using a classifier $\hat{f}^P$, excluding the ambiguity component. It often exhibits a faster convergence rate than $\delta_Q^{1+\alpha}$ and quantifies the benefits of transfer learning.
    \item The term $\varepsilon(2\tau)$ quantifies the ambiguity level \emph{only when $\eta^Q$ is close to $1/2$}.
    Thus, the target data not only establishes an upper bound on the excess risk but also plays a critical role in reducing the risk caused by ambiguity in the source data.
    \item The threshold $\tau$ balances the classification ability of the target and source data by filtering out points easily classified with the target data.
For $\tau=0$, our approach mirrors \cite{atsybakov2007fastPlugin}, achieving an asymptotically lower excess risk than $\delta_Q^{1+\alpha}$ without source data.
Alternatively, for $\tau=1/2$, only the source data is used to construct the classifier, which is often the case in domain adaptation when no target data is available.
Our choice of $\tau$ between these extremes combines the advantages of the target and source data.
    \item The final term $\delta_f$ represents the probability of the concentration inequality \eqref{eq:localConcentrationQ} failing due to extreme realizations of $\mathcal{D}_Q$. Generally, it is significantly smaller than the other terms that upper bound the excess risk. For instance, it decays exponentially with respect to $n_Q$ in non-parametric classification using K-NN estimators (See Lemma 9.1 of \cite{ttcai2020classification}).
\end{itemize}

With the tight choice of \(\tau \asymp \log (n_Q \lor n_P) \delta_Q\) to asymptotically minimize the right-hand side of \eqref{eq:TLrate}, our upper bound cannot be worse than \(\delta_Q^{1+\alpha}\) in \cite{atsybakov2007fastPlugin}, up to some logarithmic terms.  This can happen when the transfer signal is too weak so that $\tau^{1+\alpha}$ in \eqref{eq:TLrate}.

Occasionally, the concentration of $\hat{\eta}^Q$ may not be exponential, as in \eqref{eq:localConcentrationQ}. To overcome this limitation, the following theorem generalizes Theorem \ref{thm:general} to allow any type of concentration property of $\hat{\eta}^Q$.

\begin{theorem}
Let $\hat{\eta}^Q$ be an estimate of the regression function $\eta^Q$, and let $\hat{f}^P$ be a classifier obtained from $\mathcal{D}_P$. Suppose that for some $\tau>0$, there exist functions $\delta(\cdot,\cdot),\ \delta_Q(\cdot,\cdot)$ such that for any $(Q,P)\in\Pi$, the concentration property
\begin{equation}
\pr_{\mathcal{D}_Q}\left(|\hat{\eta}^Q(x)-\eta^Q(x)| \geq \tau\right)\leq \delta_Q(n_Q,\tau)
\label{eq:generalConcentrationQ}
\end{equation}
holds for any $x\in\Omega^*\subset\Omega$ with $Q(\Omega^*)\geq 1-\delta(n_Q,\tau)$. Then the \emph{TAB classifier}
\begin{equation}
\hat{f}_{TAB}(x)=
\begin{cases}
\mathbf{1}\{\hat{\eta}^Q(x)\geq \frac{1}{2}\},\quad &\text{if }|\hat{\eta}^Q(x) - \frac{1}{2}|\geq \tau,\\
\hat{f}^P(x),\quad &\text{otherwise.}
\end{cases}
\end{equation}
satisfies
\begin{equation}
    \sup_{(Q,P)\in\Pi}\E_{(\mathcal{D}_Q,\mathcal{D}_P)} \mathcal{E}_Q(\hat{f}_{TAB})\lesssim \left(\sup_{(Q,P)\in\Pi}\E_{\mathcal{D}_P}\xi(\hat{f}^P)+ \varepsilon(2\tau)\right)\land \tau^{1+\alpha} +\delta_Q(n_Q,\tau) + \delta(n_Q,\tau).
    \label{eq:generalTLrate}
\end{equation}
\label{thm:moreGeneral}
\end{theorem}

This theorem holds for any given $\tau$.
The quantity $\delta_Q(n_Q,\tau)$ in \eqref{eq:generalConcentrationQ} controls the failure probability bound of the concentration inequality concerning $\mathcal{D}_Q$. In Theorem \ref{thm:general}, it is further incorporated into the probability of extreme realizations of $X_{1:n_Q}$ \eqref{eq:localConcentrationQ}, and the subsequent concentration inequality for the failure probability related to $Y_{1:n_Q}$. 
To reduce the more general and abstract Theorem \ref{thm:moreGeneral} to Theorem \ref{thm:general}, it suffices to set $\delta_Q(n_Q,\tau)=\delta_f+C_1\exp(- (\tau/\delta_Q)^2)$ and $\Omega^*\equiv \Omega$, where $\delta_f$ and $C_1$ are the notations used in Theorem \ref{thm:general}.

Readers may be concerned that the exponential concentration \eqref{eq:localConcentrationQ} might not hold for all points over $\Omega$. Indeed, this is particularly true when the support $\Omega$ is not compact, and the concentration inequalities only hold within a part of $\Omega$ with high probability. However, we can address this by adding a small failure probability with exponential concentration, i.e., $1-Q_X(\Omega^*)\leq \delta(n_Q,\tau)$, to the risk bound given in \eqref{eq:generalTLrate}.

As the final part of this section, we list some recommendations of data-driven strategies for choosing the threshold parameter $\tau$:
\begin{enumerate}
    \item \textbf{Simple Rule of Thumb:} We recommend any value between \(0.5n_Q^{-1/2}\) and \(n_Q^{-1/2}\).
    \item \textbf{Adaptive Selection of Base Model Parameters:} If \(\delta_Q\) is represented by parameters in specific base models, adjust \(\tau\) based on these parameters. For example, Sections \ref{sec:nonPara} and \ref{sec:para} show that \(\delta_Q\) is directly related to the choice of nearest neighbors and the lasso penalty parameter.
    \item \textbf{Empirical Risk Minimization:} Select \(\tau\) by minimizing the empirical risk over the target data, choosing the one with the best empirical loss from a candidate list of \(\tau\).
\end{enumerate}
A detailed discussion of these approaches, including a rigorous analysis of the theoretical behavior of the ERM approach and the computationally feasible ERM algorithm, is provided in the Supplementary Material \citep{supp}.
\subsection{Simple Approach to Bounding Signal Transfer Risk}
\label{sec:extension}

The majority of the existing literature focuses on bounding the target excess risk using only the target data. While the excess risk of a classifier $\hat{f}^P$ with respect to the source distribution, i.e., $$2\E_{(X,Y)\sim P}\left[\Big|\eta^P(X)-\frac{1}{2}\Big|\mathbf{1}\left\{\hat{f}^P(X)\neq f_P^*(X)\right\}\right],$$ 
is well-studied, its signal transfer risk $\xi(\hat{f}^P)$, given by \eqref{eq:STR},
is less explored.

To incorporate the vast literature of traditional statistical learning into our framework, we present a result that directly bounds the signal transfer risk $\xi(\hat{f}^P)$ in terms of the excess risk of the source distribution. This result involves a more refined version of the signal transfer risk, and more significantly, does not rely on any concentration property often required by plug-in rules.

Some additional assumptions are needed. Condition \ref{con:absolutelyContinuity} ensures the boundedness of the Radon-Nikodym derivative \(\frac{dQ_X}{dP_X}\) over \(\Omega \cap \Omega_P\). This ensures that the source distribution provides enough coverage over \(\Omega \cap \Omega_P\) to enable accurate learning of every point over the common supporting region relative to \(Q_X\).

\begin{condition}[Absolutely Continuity]
$Q_X$ is absolutely continuous with respect to $P_X$. Moreover, there exists some constant $M>0$ such that the Radon-Nikodym derivative satisfies $\frac{dQ_X}{dP_X}(x)\leq M$ for any $x\in\Omega\cap \Omega_P$.
\label{con:absolutelyContinuity}
\end{condition}
Denote all such marginal distribution pairs $(Q_X,P_X)$ that satisfy Condition \ref{con:absolutelyContinuity} with parameter $M>0$ by $\mathcal{A}(M)$. 
We are in a position to present the upper bound for the signal transfer risk with respect to the target distribution in terms of the source excess risk. 

\begin{theorem} Define the source excess risk as $$\varepsilon_P:=\sup_{(Q,P)\in\Pi\cap \mathcal{A}(M)}\mathbb{E}_{\mathcal{D}_P} \E_{(X,Y)\sim P}\left[\Big|\eta^P(X)-\frac{1}{2}\Big|\mathbf{1}\left\{\hat{f}^P(X)\neq f_P^*(X)\right\}\right].$$
Then, the signal transfer risk satisfies
$$\sup_{(Q,P)\in\Pi\cap \mathcal{A}(M)}\mathbb{E}_{\mathcal{D}_P}\xi(\hat{f}^P)\leq 
\begin{cases}
    2M^{\frac{1+\alpha}{\gamma+\alpha}}C_\alpha^{\frac{\gamma-1}{\gamma+\alpha}}C_\gamma^{-\frac{1+\alpha}{\gamma+\alpha}}\varepsilon_P^{\frac{1+\alpha}{\gamma+\alpha}},\ &\gamma\geq 1,\\
    2^{\gamma-1}MC_\gamma^{-1}\varepsilon_P,\ &\gamma< 1.
\end{cases}
$$
\label{thm:STRgeneral}
\end{theorem}

Although the result may be sub-optimal when $\gamma<1$ and additionally requires that $(Q,P)\in\mathcal{A}(M)$, Theorem \ref{thm:STRgeneral} shows that the signal transfer risk can be bounded by the source excess risk $\varepsilon_P$, and inherits the performance and possible consistency with respect to classification on the source data. As long as $\hat{f}^P$ learns the source distribution sufficiently well, the signal transfer risk will converges faster than the rate obtained with only the target data, namely, $\delta_Q^{1+\alpha}$.

Theorem \ref{thm:STRgeneral} only requires knowledge of  conventional statistical learning techniques, which are widely studied and well-understood in the literature. Therefore, researchers and practitioners can easily apply our framework to a wide range of problems, without the need for specialized knowledge or expertise in transfer learning or related fields. Conventional theory on the excess risk can thus be applied to transfer learning.

\section{Applications in Non-parametric Classification}
\label{sec:nonPara}
In this section, we aim to apply the general result of Theorem \ref{thm:general} to non-parametric classification settings. Here we design the TAB classifier by combining plug-in rules over the target and source data, and obtain minimax optimal rates under non-parametric settings. See \cite{atsybakov2007fastPlugin} for a comprehensive overview of theoretical properties of plug-in rules.

We adopt $K$-nearest neighbor classifiers as plug-in rules for both $\hat{\eta}^Q$ and $\hat{\eta}^P$. Our analysis in this section then builds on prior work on rates for $K$-nearest neighbor classification (e.g. \cite{hall2008knn,samworth2012knn,gadat2016KNN,celisse2018knn,cannings2020KNN}). For a review of early work on the theoretical properties of the $K$-NN classifier, see \cite{devroye1997probabilistic}. Also, in the literature of non-parametric classification, see \cite{jfan1993lpr} and \cite{fan1996local} for local polynomial regression as an alternative to $K$-NN methods.

If the classifier $\hat{f}^P(\cdot)=\mathbf{1}\{\hat{\eta}^P(\cdot)\geq 1/2\}$ is a general plug-in rule, we also provide an explicit upper bound for the signal transfer risk $\xi(\hat{f}^P)$ based on  the point-wise misclassification rate of $\hat{\eta}^P(x)$. See the Supplementary Material \citep{supp} for the details of this bound and related results.

\subsection{K-Nearest Neighbor TAB Classifier}
Given a query point $x\in\mathbb{R}^d$, we first reorder the target data pairs as $(X_{(1)},  Y_{(1)}), \dots, (X_{(n_Q)}, Y_{(n_Q)})$ based on the Euclidean distances of the $X_i$'s to $x$, i.e.,
$$\|X_{(1)}-x\|_2\leq\dots\leq\|X_{(n_Q)}-x\|_2.$$
Then, we define the $K$-NN estimate $\hat{\eta}_k^Q(x)$ as the simple average of the response values of the $k_Q$ nearest neighbors of $x$ in the target data: 
$$\hat{\eta}^Q_{k_Q}=\frac{1}{k_Q}\sum_{i=1}^{k_Q}Y_{(i)}(x).$$
Similarly, we define the $K$-NN estimate $\hat{\eta}_{k_P}^P(x):=k_P^{-1}\sum_{i=1}^{k_P}Y^P_{(i)}(x)$ for the source data pairs $\mathcal{D}_P$. Finally, we plug these estimates into the TAB $K$-NN classifier to obtain
$$\hat{f}^{NN}_{TAB}(x)=
\begin{cases}
\mathbf{1}\{\hat{\eta}^{Q}_{k_Q}(x)\geq \frac{1}{2}\},\quad &\text{if }|\hat{\eta}^{Q}_{k_Q}(x) - \frac{1}{2}|\geq \tau,\\
\mathbf{1}\{\hat{\eta}^{P}_{k_P}(x)\geq \frac{1}{2}\},\quad &\text{otherwise.}
\end{cases}
$$
As for the threshold parameter, we choose 
$$\tau\asymp \log  (n_Q\lor n_P)k_Q^{-\frac{1}{2}}.$$ Further elaboration on the rationale behind choosing $\tau$, as well as the well-studied optimal selection of $(k_Q, k_P)$, can be found in Section \ref{sec:nonParaRate}.

\subsection{Non-parametric Classification Setting}
\label{subsec:nonPara}
We are now in a position to state the applications of our proposed TAB classifier in non-parametric classification under the finite dimension regime.
In addition to the margin assumption and the ambiguity level assumption, this paper considers the non-parametric classification problem when the following smoothness condition holds.

\begin{condition}[Smoothness] For any $\beta\in[0,1]$, the $(\beta,C_{\beta})$-\emph{Holder} is the class of functions $g:\mathbb{R}^d\rightarrow \mathbb{R}$ satisfying, for any $x,x'\in\mathbb{R}^d$
$$|g(x)-g(x')|\leq C_{\beta}\|x-x'\|^{\beta},$$ for some constant $C_\beta>0$.
We denote this class of functions by $\mathcal{H}(\beta,C_{\beta})$.
\label{con:smooth}
\end{condition}

Previous works, including \cite{ttcai2020classification} and \cite{reeve2021adpative}, do not typically require any smoothness assumption for $\eta^P$. While this simplification is justified when $\eta^P$ is closely related to $\eta^Q$, it may overlook valuable information from a sufficiently smooth $\eta^P$, leading to a phase-transition in the upper and lower bounds of the excess risk (for a detailed discussion, refer to the Supplementary Material \citep{supp}).  In contrast, our approach considers the smoothness of \emph{both} the target and source regression functions. Specifically, we assume that $$\eta^Q\in\mathcal{H}(\beta,C_{\beta}),\quad \eta^P\in\mathcal{H}(\beta_P,C_{\beta_P}).$$ This condition allows us to obtain a more refined upper bound that depends on the smoothness of both functions.

Our next condition concerns the mass of the source and target distributions in the sense that the density functions with respect to $Q_X$ and $P_X$ are bounded from zero and infinity.  A similar condition has been imposed in \citep{atsybakov2007fastPlugin} and \cite{ttcai2020classification}. We require that both $Q_X$ and $P_X$ satisfy the following condition:

\begin{condition}[Strong Density]
The marginal distribution $Q_X$ is absolutely continuous with respect to the Lebesgue measure $\lambda$ on its \emph{compact support} (denoted by $\Omega$). Furthermore, we have that
$$
    \frac{dQ_X}{d\lambda}(x)\in[\mu^-,\mu^+],\quad \frac{\lambda(B(x,r)\cap \Omega)}{\lambda(B(x,r))}\geq c_\mu,\quad \forall \; 0<r<r_u,\ x\in\Omega,
$$
Denote the set of such marginal distributions $Q_X$ by $\mathcal{S}(\mu)$ with positive parameters $\mu=(\mu^+,\mu^-,c_\mu,r_\mu)$.
\label{con:density}
\end{condition}

Taking all the conditions above into account, we consider the subset of source and target distribution pairs in $\Pi$ that satisfies Assumptions \ref{assum:margin}, \ref{assum:ambiguity} and Conditions \ref{con:smooth}, \ref{con:density}:

$$
\begin{aligned}
    \Pi^{NP}:=&\{(Q,P)\in\Pi:
    \eta^Q\in\mathcal{H}(\beta,C_{\beta}), \eta^P\in\mathcal{H}(\beta_P,C_{\beta_P}), Q_X,P_X\in\mathcal{S}(\mu)\}.
\end{aligned}
$$

We further impose the mild assumption $\alpha\beta\leq d$ to rule out the ``super-fast" rates of convergence mentioned in \cite{audibert2011fast}. This is guaranteed to hold when $\eta^Q$ equals $1/2$ at an interior point of $\Omega$ (See Proposition 3.4 of \cite{audibert2011fast}).

In the following part of this section, we explore three types of additional conditions on the space $\Pi^{NP}$ and analyze the respective excess risks.

\begin{enumerate}
    \item \textbf{Band-like Ambiguity:} We consider the scenario of band-like ambiguity described in Example \ref{eg:BA}. Define the focal parametric space as
    $$\Pi^{NP}_{BA}:=\left\{(Q,P)\in \Pi^{NP}:s(x)\geq C_\gamma\Big|\eta^Q(x)-\frac{1}{2}\Big|^\gamma-\Delta,\ \forall x\in\Omega\right\}.$$
    Recall that in Example \ref{eg:BA}, we had shown that this implies $$\varepsilon(z;\gamma,C_\gamma/2)=\Big(C_\alpha z^{1+\alpha}\Big)\land \Big(2^{\frac{1+\alpha}{\gamma}}C_\alpha C_\gamma^{-\frac{1+\alpha}{\gamma}}\Delta^{\frac{1+\alpha}{\gamma}}\Big).$$
    Furthermore, in this case, we set $\beta_P=0$, which means that no additional smoothness condition is imposed on $\eta^P$. Instead, we allow $\eta^P$ to arbitrarily fluctuate within a small band whose width is measured by $\Delta$. When $\Delta=0$, our setting reduces to the one in \cite{ttcai2020classification}. 

    \item \textbf{Smooth Source:} We add the condition $\beta_P \geq \gamma\beta$ for scenarios where $\eta^P$ is smooth. While the ambiguity level $\varepsilon(\cdot)$ is arbitrary, this condition ensures that points with strong signal strength have neighboring data points with strong signal strength of $\eta^P$ as well, enhancing the classification of source data points with strong signal strength.
    
    Define
    $\Pi^{NP}_{S}$ as the subset of $\Pi^{NP}$ such that $\beta_P \geq \gamma\beta$.
    When $\eta^Q$ and $\eta^P$ share the same smoothness degree $\beta$, we set $\gamma=1$.

    \item \textbf{Strong Signal with Imperfect Transfer:} We consider the scenario in Example \ref{eg:RSS} where a strong signal strength exists but the direction may be reversed. This condition ensures that the region $s(x)\leq C_\gamma|\eta^Q(x)-1/2|^\gamma$ is smooth, which further ensures the availability of a sufficient number of neighboring source data points with the strong signal. In fact, its boundary is a part of the decision boundary $\{ x\in\Omega:\eta^Q(x)=1/2\}$ whose smoothness is guaranteed by the smoothness of $\eta^Q$ (see the proof of Theorem \ref{thm:nonParaSpecial}).
    
    Define
    $$\Pi^{NP}_{I}:=\left\{(Q,P)\in \Pi^{NP}:\eta^P \text{ is continuous, }\Big|\eta^P(x)-\frac{1}{2}\Big|\geq C_\gamma\Big|\eta^Q(x)-\frac{1}{2}\Big|^\gamma,\ \forall x\in\Omega\right\}.$$
    It is worth noting that $s(x) \geq C_\gamma |\eta^Q(x) - 1/2|^\gamma$ if and only if $x\notin \Omega_R$, i.e., $(\eta^P(x) - 1/2)(\eta^Q(x) - 1/2) \geq 0$. Therefore, the ambiguity level
    $$\varepsilon(z)=\E_{(X,Y)\sim Q}\left[\Big|\eta^Q(X)-\frac{1}{2}\Big|\mathbf{1}\left\{\Big|\eta^Q(X)-\frac{1}{2}\Big|\leq z, X\in \Omega_R\right\}\right]$$
    precisely captures the risk caused by different Bayes classifiers between the target and source data.
    We also let $\beta_P=0$ in this case, but assume that $\eta^P$ is continuous to make sure that the region $s(x)\leq C_\gamma|\eta^Q(x)-1/2|^\gamma$ has a continuous boundary.

\end{enumerate}

The analysis of the excess risk rate with respect to the wider class $\Pi^{NP}$ is provided in the Supplementary Material \citep{supp}.

\subsection{Optimal Rate of Excess Risk}
\label{sec:nonParaRate}
Before presenting the results regarding the optimal rate of $\mathcal{E}_Q(\hat{f}^{NN}_{TAB})$, we first discuss parameter selection. We choose the number of nearest neighbors as follows: $$k_Q=\lfloor c_Qn_Q^{\frac{2\beta}{2\beta+d}}\rfloor,\ k_P=\lfloor c_Pn_P^{\frac{2\gamma\beta}{2\gamma\beta+d}}\rfloor$$ where $c_Q$ and $c_P$ are positive constants. This choice is motivated by previous work such as \cite{gadat2016KNN} and \cite{cannings2020KNN}, where similar choices are made in the context of nearest-neighbor methods. Notably, the choice of $k_P$ is similarly derived by seeing $\gamma\beta$ as the ``smoothness" parameter for the source data, and our choice coincides with the classical optimal choice when $\beta_P=\gamma\beta$. In addition, we assume $\gamma$ and $\beta$ are known here for convenience. For adaptive and rate-optimal approaches to determining the number of nearest neighbors, see \cite{lepski1993,reeve2021adpative}.

As for the threshold in the TAB classifier $\hat{f}^{NN}_{TAB}$, we choose 
$$\tau\asymp \log  (n_Q\lor n_P)k_Q^{-\frac{1}{2}}.$$
This choice is consistent with the concentration property of $\hat{\eta}^Q_{k_Q}$, of which the ``uncertainty'' level $\delta_Q$ in \eqref{eq:localConcentrationQ} is proportional to $k_Q^{-\frac{1}{2}}$, since $\hat{\eta}^Q_{k_Q}$ is the average of $k_Q$ random variables. By the definition of $k_Q$, we see that
$$\tau\asymp \log(n_Q\lor n_P) n_Q^{-\frac{\beta}{2\beta+d}}.$$

Given a family of the target and source distributions, the next theorem gives a provable upper bound on the excess risk of the TAB $K$-NN classifier $\hat{f}^{NN}_{TAB}$, with the proper parameter choices.

\begin{theorem}[Non-parametric Classification Upper Bound] Suppose that  $n_Q^{\frac{d}{2\beta+d}}\exp (-c_Qn_Q^{\frac{2\beta}{2\beta+d}})\lesssim n_P^{-\frac{\beta(1+\alpha)}{2\gamma\beta+d}}$. Then the TAB $K$-NN classifier $\hat{f}^{NN}_{TAB}(x)$ satisfies:
\begin{enumerate}
    \item \textbf{Band-like Ambiguity:}
    \begin{equation}
        \sup_{(Q,P)\in\Pi_{BA}^{NP}}\E_{(\mathcal{D}_Q,\mathcal{D}_P)} \mathcal{E}_Q(\hat{f}^{NN}_{TAB})\lesssim\left(n_P^{-\frac{\beta(1+\alpha)}{2\gamma\beta+d}}+\Delta^{\frac{1+\alpha}{\gamma}}\right)\land \left(\log^{1+\alpha} (n_Q\lor n_P) n_Q^{-\frac{\beta(1+\alpha)}{2\beta+d}}\right).
        \label{eq:nonParaUpperBA}
    \end{equation}

    \item \textbf{Smooth Source:}
    \begin{equation}
        \sup_{(Q,P)\in\Pi_{S}^{NP}}\E_{(\mathcal{D}_Q,\mathcal{D}_P)} \mathcal{E}_Q(\hat{f}^{NN}_{TAB})\lesssim \left( n_P^{-\frac{\beta(1+\alpha)}{2\gamma\beta+d}}+\varepsilon(2\tau)\right)\land \left(\log^{1+\alpha} (n_Q\lor n_P) n_Q^{-\frac{\beta(1+\alpha)}{2\beta+d}}\right).
        \label{eq:nonParaUpperS}
    \end{equation}

    \item \textbf{Strong Signal with Imperfect Transfer:}
    \begin{equation}
        \sup_{(Q,P)\in\Pi_{I}^{NP}}\E_{(\mathcal{D}_Q,\mathcal{D}_P)} \mathcal{E}_Q(\hat{f}^{NN}_{TAB})\lesssim \left( n_P^{-\frac{\beta(1+\alpha)}{2\gamma\beta+d}}+\varepsilon(2\tau)\right)\land \left(\log^{1+\alpha} (n_Q\lor n_P) n_Q^{-\frac{\beta(1+\alpha)}{2\beta+d}}\right).
        \label{eq:nonParaUpperF}
    \end{equation}
\end{enumerate}
\label{thm:nonParaSpecial}
\end{theorem}

Theorem \ref{thm:nonParaSpecial} is obtained simply by verifying the conditions in Theorem \ref{thm:general}, as detailed in the Supplementary Material \citep{supp}. The Supplementary Material presents a more refined and technical approach to analyze the signal transfer risk of plug-in rules compared to Theorem \ref{thm:STRgeneral}. The mild condition $n_Q^{\frac{d}{2\beta+d}}\exp (-c_Qn_Q^{\frac{2\beta}{2\beta+d}})\lesssim n_P^{-\frac{\beta(1+\alpha)}{2\gamma\beta+d}}$ controls the probability of failure $\delta_f$ in the terminology of Theorem \ref{thm:general}.

Our risk bounds reveal that transfer learning leads to faster convergence rates of excess risk when $n_P$ is large compared to $n_Q$ and the ambiguity level is small; specifically,
$$
\begin{aligned}
\Pi_{BA}^{NP}&: n_P\gg n_Q^{\frac{2\gamma\beta+d}{2\beta+d}},\ \Delta\ll n_Q^{-\frac{\gamma\beta}{2\beta+d}};\\
\Pi^{NP}_{S},\Pi^{NP}_{I}&: n_P\gg n_Q^{\frac{2\gamma\beta+d}{2\beta+d}},\ \varepsilon(2\tau)\ll \tau^{1+\alpha}.
\end{aligned}
$$
On the other hand, if $n_P$ is small compared to $n_Q$, then the term $n_Q^{-\frac{\beta(1+\alpha)}{2\beta+d}}$ dominates the upper bound, and reduces to the risk rate in the conventional setting with only target data and the strong density assumption \citep{atsybakov2007fastPlugin}, up to logarithmic factors.

\begin{remark}
    (a) When $\Delta=0$ in \eqref{eq:nonParaUpperBA}, or $\varepsilon(\cdot)\equiv 0$ in \eqref{eq:nonParaUpperF}, our upper bound reduces to Theorem 2 of \cite{ttcai2020classification}, i.e., $n_P^{-\frac{\beta(1+\alpha)}{2\gamma\beta+d}}\land n_Q^{-\frac{\beta(1+\alpha)}{2\beta+d}}$, up to logarithmic factors. Importantly, our method demonstrates robustness against a positive band error $\Delta$, in constrast to the weighted $K$-NN estimator proposed in their work. In addition, our proposed TAB method is simple and does not require the careful design of the classifier, such as the weighting or Lepski's method used in \cite{ttcai2020classification} and \cite{reeve2021adpative}. We refer the reader to Figure \ref{fig:band} in Section \ref{sec:simulation} for a numerical comparison.
    
    (b) Our result \eqref{eq:nonParaUpperS} supplements these previous works \citep{ttcai2020classification,reeve2021adpative} by allowing for an arbitrary type of ambiguity after imposing a smoothness condition on $\eta^P$.
\end{remark}

We next provide lower bounds on the minimax excess risk, which show that the TAB $K$-NN classifier achieves the minimax optimal rates, even when we add the constraint $\Omega=\Omega_P$. All lower bound results are constructed over a subset of $\Pi^{NP}$, which is also a subset of the distribution pair family satisfying Assumptions \ref{assum:margin} and \ref{assum:ambiguity} with specified ambiguity level $\varepsilon(\cdot)$.

\begin{theorem}[Non-parametric Classification Lower Bound] Fix the parameters $\alpha\beta\leq d$ and set $\tau\asymp \log(n_Q\lor n_P) n_Q^{-\frac{\beta}{2\beta+d}}$. We have that 
\begin{enumerate}
    \item \textbf{Band-like Ambiguity:}
    \begin{equation}
        \inf_{\hat{f}}\sup_{\substack{(Q,P)\in\Pi_{BA}^{NP}\\\Omega=\Omega_P}}\E_{(\mathcal{D}_Q,\mathcal{D}_P)} \mathcal{E}_Q(\hat{f})\gtrsim\left(n_P^{-\frac{\beta(1+\alpha)}{2\gamma\beta+d}}+\Delta^{\frac{1+\alpha}{\gamma}}\right)\land n_Q^{-\frac{\beta(1+\alpha)}{2\beta+d}}.
        \label{eq:nonParaLowerBA}
    \end{equation}

    \item \textbf{Smooth Source:} For some constant $c>0$ that is independent of $\alpha$,
    \begin{equation}
        \inf_{\hat{f}}\sup_{\substack{(Q,P)\in\Pi_{S}^{NP}\\\Omega=\Omega_P}}\E_{(\mathcal{D}_Q,\mathcal{D}_P)} \mathcal{E}_Q(\hat{f})\gtrsim \left( n_P^{-\frac{\beta(1+\alpha)}{2\gamma\beta+d}}+\varepsilon(cn_Q^{-\frac{\beta}{2\beta+d}})\right)\land n_Q^{-\frac{\beta(1+\alpha)}{2\beta+d}}.
        \label{eq:nonParaLowerS}
    \end{equation}

    \item \textbf{Strong Signal with Imperfect Transfer:} For some constant $c>0$ that is independent of $\alpha$,
    \begin{equation}
        \inf_{\hat{f}}\sup_{\substack{(Q,P)\in\Pi_{I}^{NP}\\\Omega=\Omega_P}}\E_{(\mathcal{D}_Q,\mathcal{D}_P)} \mathcal{E}_Q(\hat{f})\gtrsim \left( n_P^{-\frac{\beta(1+\alpha)}{2\gamma\beta+d}}+\varepsilon(cn_Q^{-\frac{\beta}{2\beta+d}})\right)\land n_Q^{-\frac{\beta(1+\alpha)}{2\beta+d}}.
        \label{eq:nonParaLowerF}
    \end{equation}
\end{enumerate}
\label{thm:nonParaSpecialLowerBound}
\end{theorem}

In the special case where $\sup_{x\in\Omega}|\eta^Q(x) - \eta^P(x)| \leq \Delta$, we can determine that the minimax optimal excess risk is $\left(n_P^{-\frac{\beta(1+\alpha)}{2\beta+d}}+\Delta^{1+\alpha}\right)\land n_Q^{-\frac{\beta(1+\alpha)}{2\beta+d}}$, up to logarithmic factors of $n_Q\lor n_P$. As long as
$$n_P\gg n_Q,\ \Delta\ll n_Q^{-\frac{\beta}{2\beta+d}},$$ the classifier will benefit from the source data. The proof of Theorem \ref{thm:nonParaSpecialLowerBound} reveals that condition \eqref{eq:etaPetaQClose}, although slightly stronger than the band-like ambiguity condition \eqref{eq:BA}, remains compatible with the lower bound construction as it ensures $\sup_{x\in\Omega}|\eta^Q(x) - \eta^P(x)| \leq \Delta$.

\section{Applications in Logistic Regression}
\label{sec:para}
Besides non-parametric classification, we also investigate the use of transfer learning in logistic regression models, which are a commonly used parametric approach in classification. Previous works such as \cite{zheng2019logistic} have studied the ``data enriched model" for logistic regression under a single-source setting, \cite{abramovich2019logistic} have explored sparse logistic regression in high-dimensional settings, and \cite{tian2022tlGLMs} have considered transfer learning in generalized linear models. Our goal is to reveal how incorporating an additional source logistic regression model with a different linear regression coefficient can enhance the convergence of the excess misclassification rate.

In this section, we extend to the high-dimensional regime, allowing \( d \to \infty \) as \( n_Q \to \infty \). Suppose the source and target distributions are logistic regression models given by
\begin{equation}
    \begin{aligned}
    &\text{Target data model: } \eta^Q(x)=\sigma({\beta_Q}^Tx)\\
    &\text{Source data model: } \eta^P(x)=\sigma({\beta_P}^Tx),
\end{aligned}
\label{eq:logisticModel}
\end{equation}
where two independent samples $(X_1,Y_1),\dots,(X_{n_Q},Y_{n_Q})
\overset{\text{iid}}{\sim}Q$ and $(X_1^P,Y_1^P),\dots, (X_{n_P}^P,Y_{n_P}^P)\overset{\text{iid}}{\sim}P$ are observed. To simplify the theoretical analysis, we assume that the marginal distributions $Q_X$ and $P_X$ are both $N(0,I_d)$, the $d$-dimension standard normal distribution. This marginal distribution is convenient when working with the restricted strong convexity condition (See \cite{negahban2009RSC}).  It can be extended to the more general distributions studied by \cite{dobriban2016regularity}.

Let $\angle ( \alpha,\beta )$ be the angle of two vectors $\alpha$ and $\beta$,  in the range of $0$ and $\pi/2$. Consider the following parametric space of the coefficient pair $(\beta_Q,\beta_P)$:
\begin{equation}
    \Theta(s,\Delta)=\left\{(\beta_Q,\beta_P):\|\beta_Q\|_0\leq s,\angle ( \beta_Q,\beta_P )\leq \Delta \right\},
    \label{eq:Theta}
\end{equation}
for some $s>0$ and $\Delta\in[0,\pi/2]$. The corresponding family of distribution pairs is then
\begin{equation}
\begin{aligned}
  \Pi^{LR}=\Pi^{LR}(s,\Delta,M)=\{&(Q,P):X,X^P\sim N(0,I_d),\eta^Q(x)=\sigma(\beta_Q^Tx),\\
  &\eta^P(x)=\sigma(\beta_P^Tx),(\beta_Q,\beta_P)\in\Theta(s,\Delta)
  \}.  
\end{aligned}
\label{eq:logisticFamily}
\end{equation}
To ensure the control of the ambiguity level, we impose a constraint on the angle between $\beta_Q$ and $\beta_P$ in \eqref{eq:Theta}, which must be smaller than a constant $\Delta$. Importantly, our constructed parametric space does not impose any sparsity conditions on $\beta_P$.

Given the family of logistic distribution pairs $\Pi^{LR}$, we show that $\Pi^{LR}$ is a subset of the overall distribution pair space $\Pi$ with $\alpha=1$ and $\varepsilon(z,1,m/\pi)\lesssim z^2\land \Delta^2$, provided that $\|\beta_P\|\geq m \|\beta_Q\|$ for some constant $m>0$. See the Supplementary Material \citep{supp} for a detailed proof.   

For model fitting, we minimize the negative Bernoulli likelihood function with lasso regularization terms to obtain $\hat{\beta}_Q$ and $\hat{\beta}_P$, i.e.,
\begin{equation}
\begin{aligned}
&\hat{\beta}_Q=\argmin_{\beta\in\mathbb{R}^p} \frac{1}{n_Q}\sum_{i=1}^{n_Q} \left\{\log(1+e^{X_i^T\beta})-Y_iX_i^T\beta\right\}+\lambda_Q\|\beta\|_1\\
&\hat{\beta}_P=\argmin_{\beta\in\mathbb{R}^p} \frac{1}{n_P}\sum_{i=1}^{n_P} \left\{\log(1+e^{{X_i^P}^T\beta})-Y_i^P{X_i^P}^T\beta\right\}+\lambda_P\|\beta\|_1,
\label{eq:minimization}
\end{aligned}
\end{equation}
where $\lambda_*\asymp \sqrt{\frac{\log d}{n_*}}$ for $*\in\{Q,P\}$. The corresponding target and source plug-in classifiers are then $\mathbf{1}\{\sigma(\hat{\beta}_Q^Tx)\geq 1/2\}=\mathbf{1}\{\hat{\beta}_Q^Tx\geq 0\}$ and $\mathbf{1}\{\hat{\beta}_P^Tx\geq 0\}$. Hence, the TAB logistic lasso classifier becomes
$$
\hat{f}^{LR}_{TAB}(x)=
\begin{cases}
\mathbf{1}\{\hat{\beta}_Q^Tx\geq 0\},\quad &\text{if }|\sigma(\hat{\beta}_Q^Tx) - 1/2|\geq \tau,\\
\mathbf{1}\{\hat{\beta}_P^Tx\geq 0\},\quad &\text{otherwise.}
\end{cases}
$$

By setting $\lambda_Q $, $\lambda_P$, and $\tau$ properly, the following theorem gives an excess risk upper bound.

\begin{theorem}
Assume that $L\leq m \|\beta_Q\|_2\leq \|\beta_P\|_2\leq U$ for some constants $L,U>0$ and $0<m\leq 1$. Suppose that for some constant $K>0$, we have $d^{K}\gtrsim \frac{n_Q\lor n_P}{s\log d},\ n_Q\gg \log \frac{n_P}{s\log d}$, and $n_Q\land n_P\gg s\log d$. Let $\hat{\beta}_Q,\hat{\beta}_P$ be obtained in \eqref{eq:minimization} with $$\lambda_Q= c_Q\sqrt{\frac{\log d}{n_Q}},\ \lambda_P= c_P\sqrt{\frac{\log d}{n_P}},$$ for some constants $c_Q,c_P\geq \sqrt{(K+1)}$. The TAB lasso classifier
with threshold $$\tau= c_\tau\sqrt{\frac{s\log d}{n_Q}}\log(n_Q\lor n_P),$$ for some constant $c_\tau>0$, satisfies
\begin{equation}
\begin{aligned}
    &\sup_{(Q,P)\in\Pi^{LR}}\E_{(\mathcal{D}_Q,\mathcal{D}_P)} \mathcal{E}_Q(\hat{f}^{LR}_{TAB})
    \lesssim \left(\frac{s\log d}{n_P}+\Delta^2\right)\land \left(\frac{s\log d}{n_Q}\log^2(n_Q\lor n_P)\right).
\end{aligned}
\label{eq:logisticUpper}
\end{equation}
\label{thm:logisticUpper}
\end{theorem}

The term $\frac{s\log d}{n_Q}$ is the classical risk term with access to only the target data, and $\frac{s\log d}{n_P}$ is the risk term transferred by the source data with an additional term $\Delta^2$ measuring the angle discrepancy between $\beta_Q$ and $\beta_P$. We see that knowledge from the source data can significantly improve the learning performance when $n_P$ is large and $\Delta$ is small, namely, 
$$
    n_P\gg n_Q,\quad \Delta\ll \sqrt{\frac{s\log d}{n_Q}}.
$$ 
It is worth noting that when \(\|\beta_Q\|\) and \(\|\beta_P\|\) are bounded away from zero, the small angle condition between $\beta_Q$ and $\beta_P$ considered in this paper is a more general assumption than the contrast assumption in \cite{tian2022tlGLMs}, which requires that $\beta_Q$ is sparse and the $l_q$-norm of the difference between $\beta_Q$ and $\beta_P$ is small for some $q\in[0,1]$.
In contrast, our result is applicable to a broader class of parameter spaces, allowing for $\beta_P$ to not only be non-sparse, but also differ significantly in norm from $\beta_Q$.

We choose this $\tau$ to ensure that, with high probability, we have $\|\hat{\beta}_Q-\beta_Q\|_2\lesssim \sqrt{s}\lambda_Q$. This, in turn, implies $$\Big|\sigma(\hat{\beta}_Q^Tx)-\sigma(\beta_Q^Tx)\Big|\lesssim \sqrt{s}\lambda_Q$$ with high probability with respect to $Q_X=N(0,I_d)$. The proof of Theorem \ref{thm:logisticUpper} then comes from verifying that $\sup_{(Q,P)\in \Pi^{LR}}\E_{\mathcal{D}_P}\xi(\mathbf{1}\{\hat{\beta}_P^Tx\geq 0\})\asymp \frac{s\log d}{n_P}+\Delta^2$ and $\varepsilon(z)\asymp z^2\land \Delta^2$ in Theorem \ref{thm:moreGeneral}. This indicates that even when $\beta_P$ is non-sparse, we can still obtain a reliable estimate of $\beta_P$ by incorporating a lasso regularizer, if in addition $\Delta$ is sufficiently small.

Theorem \ref{thm:logisticLower} below shows that the upper bound \eqref{eq:logisticUpper} in Theorem \ref{thm:logisticUpper} is optimal up to logarithmic factors of $n_Q\lor n_P$.
\begin{theorem} Suppose that $\frac{s\log d}{n_Q\lor n_P}\lesssim 1$. We have that
$$\inf_{\hat{f}}\sup_{(Q,P)\in\Pi^{LR}}\E\mathcal{E}_Q(\hat{f})\gtrsim
\left(\frac{s\log d}{n_P} + \Delta^2\right)\land \frac{s\log d}{n_Q}.$$

\label{thm:logisticLower}
\end{theorem}

The derivation of the lower bound involves two terms. The term $\frac{s\log d}{n_P} \land \frac{s\log d}{n_Q}$ represents the optimal convergence rate when the target and source distributions are identical, i.e., $\beta_P=\beta_Q$. The term $\Delta^2\land \frac{s\log d}{n_Q}$ corresponds to the choice $\beta_P=(1,0,0,\dots,0)$ in the lower bound construction, which imposes a sparsity constraint on $\beta_Q$ within a small cone.

It is worth noting that traditional lower bounds, typically derived from Fano's lemma, only consider minimax rates with respect to a distance metric. To show our lower bound of the excess risk $\E\mathcal{E}_Q(\hat{f}) $, we introduce a novel transformation that relates the excess risk to the angle difference of the linear coefficients. By applying Fano's lemma to this transformed quantity, we obtain the desired lower bound. We refer the reader to the Supplementary Material \citep{supp} for a detailed explanation.

\section{Simulation Studies}
\label{sec:simulation}
As mentioned earlier, the TAB classifier offers benefits in scenarios where $n_P$ is large and the ambiguity level is small, and prevents negative transfer when the source data lacks sufficient information to aid the classification. In this section, we present simulation studies to demonstrate the practical benefits of transfer learning and the TAB classifier. We separately consider the non-parametric classification and logistic regression settings.

\subsection{Non-parametric Classification Setting}
The setting we considered is as follows: $d=2,\ \Omega=\Omega_P=[0,1]^2,\ Q_X=P_X=\mathrm{Uniform}([0,1]^2)$, the uniform distribution over the square. For the target regression function, let $\eta^Q(x)=\eta^Q(x_1,x_2)=\frac{1}{2}+\frac{1}{10}\sin(2\pi(x_1+x_2))$. We have $\alpha=\beta=1$ for some constants $C_\alpha$ and $C_\beta$. Next, we consider two different non-parametric regression scenarios, by specifying the source regression function $\eta^P$ with different types of ambiguity.
\begin{enumerate}
\item \textbf{Band-like Ambiguity:} For $\Delta\in\{0,0.1,0.2,0.3,0.4,0.5,0.6\},\ \gamma\in\{0.5,1\}$:
    $$
    \eta^P(x)=
    \frac{1}{2}+ \left \{2\Big(\eta^Q(x)- \frac{1}{2}\Big)^\gamma-\Delta \right \}  \mbox{sgn}\left (\eta^Q(x)\geq \frac{1}{2} \right ).
    $$
Here, $\eta^P$ concentrates around an informative curve with respect to $\eta^Q$ with some ambiguity $\Delta$. In this case, we have $s(x)\geq |\eta^Q(x)-1/2|^\gamma-\Delta$.

\item \textbf{Partially Flipped Sine Functions:} For $r\in\{0, 0.05,0.1,0.15,0.2,0.25,0.3,0.35\},\ \gamma\in\{0.5,1\}$:
$$
\eta^P(x)=\eta^P(x_1,x_2)=
\begin{cases}
\frac{1}{2}-\frac{1}{5}\sin\Big(2\pi\frac{\{x_1+x_2\}}{r}\Big)^\gamma,\quad &\text{if }\{2x_1+2x_2\}\in[0,r],\\
\frac{1}{2}+\frac{1}{5}\sin\Big(2\pi\frac{\{x_1+x_2\}-r}{1-r}\Big)^\gamma,\quad &\text{if }\{2x_1+2x_2\}\in(r,1],
\end{cases}
$$
where $\{a\}=a-\lfloor a\rfloor$ represents the fractional part of a real value $a$. The following graph illustrates our setup of $\eta^P$. While keeping $\eta^P$ continuous with $\beta_P=1$, the ratio parameter $r$ creates an area where the Bayes classifier differs from the target distribution. The positive classification regimes are identical when $r=0$ and completely opposite when $r=1$. See Figure \ref{fig:visual} for a visualization.
\begin{figure}

\begin{center}\begin{tikzpicture}[scale=4]
    \draw[->] (0, 0.5) -- (2.2, 0.5) node[right] {$x_1+x_2$};
    \draw[->] (0, -0.2) -- (0, 1.2) node[above] {$\eta$};

    \foreach \x in {0,0.5,1,1.5,2}
        \draw (\x,0.52) -- (\x,0.48) node[below] {$\x$};

    \foreach \y in {0.4,0.6}
        \draw (-0.02,5*\y-2) -- (0.02,5*\y-2) node[left] {$\y$};

    \draw[domain=0:2, samples=100, smooth, thick, blue] plot (\x, {0.5 + 0.25*sin(360*\x)}) node[above right] {$\eta^Q$};

    \foreach \y in {0,1,2,3}
     \draw[domain=0.5*\y:0.5*\y+0.2, samples=100, smooth, thick, red] plot (\x, {0.5 - 0.5*(-1)^(\y)*sin(900*(\x-0.5*\y))});

     \foreach \y in {0,1,2}
     \draw[domain=0.5*\y+0.2:0.5*\y+0.5, samples=100, smooth, thick, red] plot (\x, {0.5 + 0.5*(-1)^(\y)*sin(600*(\x-0.5*\y-0.2))});

     \draw[domain=1.7:2, samples=100, smooth, thick, red] plot (\x, {0.5 -0.5*sin(600*(\x-1.7))}) node[below right] {$\eta^P$};

\end{tikzpicture}\end{center}
\caption{Illustration of $\eta^P$ in the second simulation setup with $\gamma=1$ and $r=0.4$.}\label{fig:visual}
\end{figure}
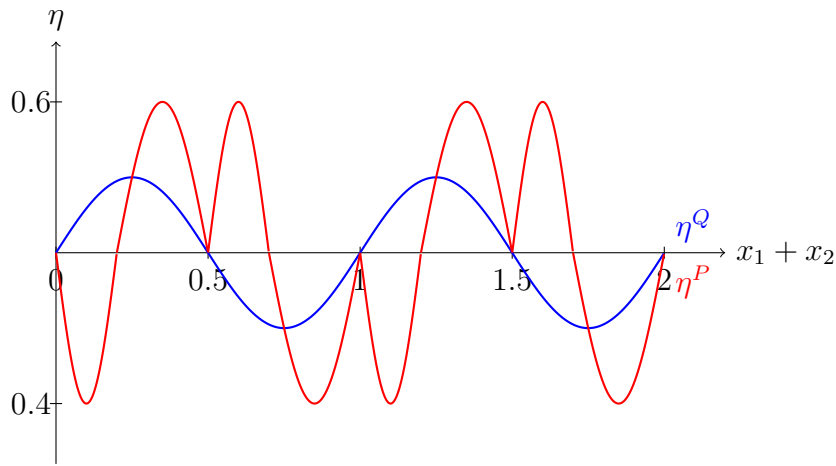
\end{enumerate}
In each scenario, we set $n_Q=200$ and $n_P=1000$, as a large $n_P$ is necessary to observe the benefits of transfer learning. Previous studies \cite{ttcai2020classification,reeve2021adpative} have shown that accuracy improves with increasing $n_P$. We also generate $50000$ independent test pairs from $Q$. For the TAB $K$-NN classifier, we choose $k_Q=\lfloor n_Q^{\frac{2\beta}{2\beta+d}}\rfloor=31,\ k_P=\lfloor n_P^{\frac{2\gamma\beta}{2\gamma\beta+d}}\rfloor$, and $\tau=0.05$. As mentioned before, $\beta = 1$.

We choose the simple $K$-NN classifiers on $Q-$data and $P$-data as two benchmarks for comparison. Moreover, we add the $K$-NN classifier on the pooled data, combining both the target and source with the nearest-neighbor parameter $k$ chosen by $5$-fold cross-validation. We also consider the weighted $K$-NN classifiers proposed in \cite{ttcai2020classification} with their indicated optimal weighting scheme $w_Q+w_P=1,\ w_P/w_Q=(n_Q+n_P^{\frac{2\beta+d}{2\gamma\beta+d}})^{\frac{(\gamma-1)\beta}{2\beta+d}}$.

Figure \ref{fig:band} shows that our TAB classifier is accurate when $\Delta$ is small, as does $K$-NN on $P$-data. Furthermore, our TAB classifier outperforms $K$-NN on $P$-data, $K$-NN on pooled data, and the weighted $K$-NN by a significant margin for large $\Delta$, demonstrating its ability to avoid negative transfer when the source data is unreliable. The pooling and weighting algorithm benchmarks improve the classification when $\Delta$ is small; however, they tend to break down when $\Delta$ is large.

\begin{figure}
    \centering
\begin{subfigure}{.45\textwidth}
  
\begin{tikzpicture}
\begin{axis}[
width=5cm,
   height=5cm,
   scale only axis,
   xmin=0, xmax=0.5,
   xtick={0.1,0.2,0.3,0.4,0.5},
   xticklabels={0.1,0.2,0.3,0.4,0.5},
   xmajorgrids,
   ymin=0.4, ymax=1,
   xlabel={$\Delta$},
   ylabel={Accuracy(\%)},
   ymajorgrids,
   title={$\gamma=0.5$},
   axis lines*=left,
   legend style ={ at={(1.03,1)}, 
        anchor=north west, draw=black, 
        fill=white,align=left},
    cycle list name=black white,
    smooth
]

    \addplot[dashed,color=red,mark=triangle*] coordinates{
    (0, 0.64718)
    (0.1, 0.64718)
    (0.2, 0.64718)
    (0.3, 0.64718)
    (0.4, 0.64718)
    (0.5, 0.64718)
   };

   \addplot [dashed,color=green!60!black,mark=square*]coordinates{
    (0, 0.81304)
    (0.1, 0.79472)
    (0.2, 0.74392)
    (0.3, 0.65546)
    (0.4, 0.55926)
    (0.5, 0.56430)
   };

   \addplot [dashed,color=blue,mark=*]coordinates{
    (0, 0.78374)
    (0.1, 0.77166)
    (0.2, 0.74782)
    (0.3, 0.69208)
    (0.4, 0.63808)
    (0.5, 0.66174)
   };

   \addplot [dashed,color=brown,mark=diamond*]coordinates{
    (0, 0.81308)
    (0.1, 0.79288)
    (0.2, 0.74052)
    (0.3, 0.66636)
    (0.4, 0.57894)
    (0.5, 0.58120)
   };

   \addplot [dashed,color=purple,mark=pentagon*]coordinates{
    (0, 0.81244)
    (0.1, 0.79329)
    (0.2, 0.74098)
    (0.3, 0.66122)
    (0.4, 0.56274)
    (0.5, 0.56962)
   };

   \end{axis}
\end{tikzpicture}  
\end{subfigure}
\begin{subfigure}{.45\textwidth}
 \begin{tikzpicture}
\begin{axis}[
width=5cm,
   height=5cm,
   scale only axis,
   xmin=0, xmax=0.5,
   xtick={0.1,0.2,0.3,0.4,0.5},
   xticklabels={0.1,0.2,0.3,0.4,0.5},
   xmajorgrids,
   ymin=0.4, ymax=1,
   xlabel={$\Delta$},
   ylabel={Accuracy(\%)},
   ymajorgrids,
   title={$\gamma=1$},
   axis lines*=left,
   legend style ={ at={(1.03,1)}, 
        anchor=north west, draw=black, 
        fill=white,align=left, font=\footnotesize},
    cycle list name=black white,
    smooth
]

    \addplot[dashed,color=red,mark=triangle*] coordinates{
    (0, 0.64718)
    (0.1, 0.64718)
    (0.2, 0.64718)
    (0.3, 0.64718)
    (0.4, 0.64718)
    (0.5, 0.64718)
   };
   \addlegendentry{$Q$-KNN};

   \addplot [dashed,color=green!60!black,mark=square*]coordinates{
    (0, 0.79320)
    (0.1, 0.68216)
    (0.2, 0.58854)
    (0.3, 0.52104)
    (0.4, 0.49588)
    (0.5, 0.49920)
   };
   \addlegendentry{$P$-KNN};

   \addplot [dashed,color=blue,mark=*]coordinates{
    (0, 0.77558)
    (0.1, 0.70362)
    (0.2, 0.68716)
    (0.3, 0.64970)
    (0.4, 0.63746)
    (0.5, 0.64076)
   };
   \addlegendentry{TAB-KNN};

   \addplot [dashed,color=brown,mark=diamond*]coordinates{
    (0, 0.79410)
    (0.1, 0.69130)
    (0.2, 0.61706)
    (0.3, 0.53824)
    (0.4, 0.50048)
    (0.5, 0.49848)
   };
   \addlegendentry{Pooled-KNN};

   \addplot [dashed,color=purple,mark=pentagon*]coordinates{
    (0, 0.75414)
    (0.1, 0.62790)
    (0.2, 0.61036)
    (0.3, 0.55458)
    (0.4, 0.52130)
    (0.5, 0.50242)
   };
   \addlegendentry{Weighted-KNN};

   \end{axis}
\end{tikzpicture}   
\end{subfigure}
    \caption{Accuracy of the TAB $K$-NN classifiers under the band-like ambiguity scenario. We experiment with different values of $\Delta$ for a given $\gamma=0.5$ and $1$. Blue: TAB $K$-NN classifier; Red: $K$-NN classifier on only $Q$-data; Green: $K$-NN classifier on only $P$-data; Brown: $K$-NN classifer on pooled data.}
    \label{fig:band}
\end{figure}
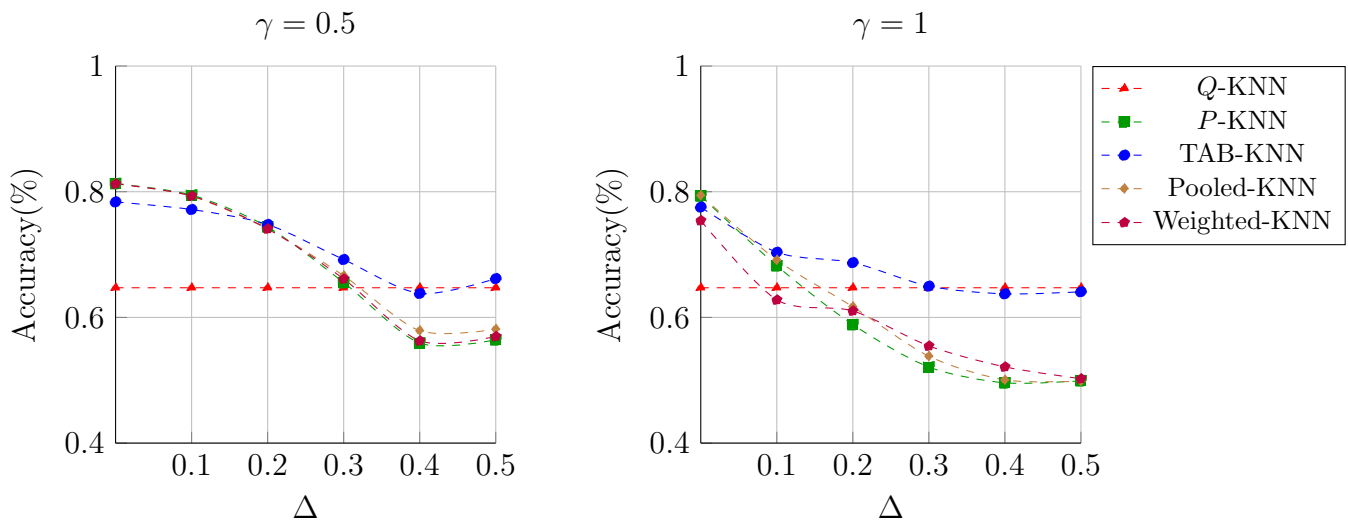

Additionally, Figure \ref{fig:flipped} shows that our TAB classifier improves the accuracy when the ambiguity level, depicted by $r$, is small. In addition, our TAB classifier significantly outperforms the other three benchmarks when the ambiguity level is too large to benefit from transfer learning, conserving the classification ability using only $Q$-data.

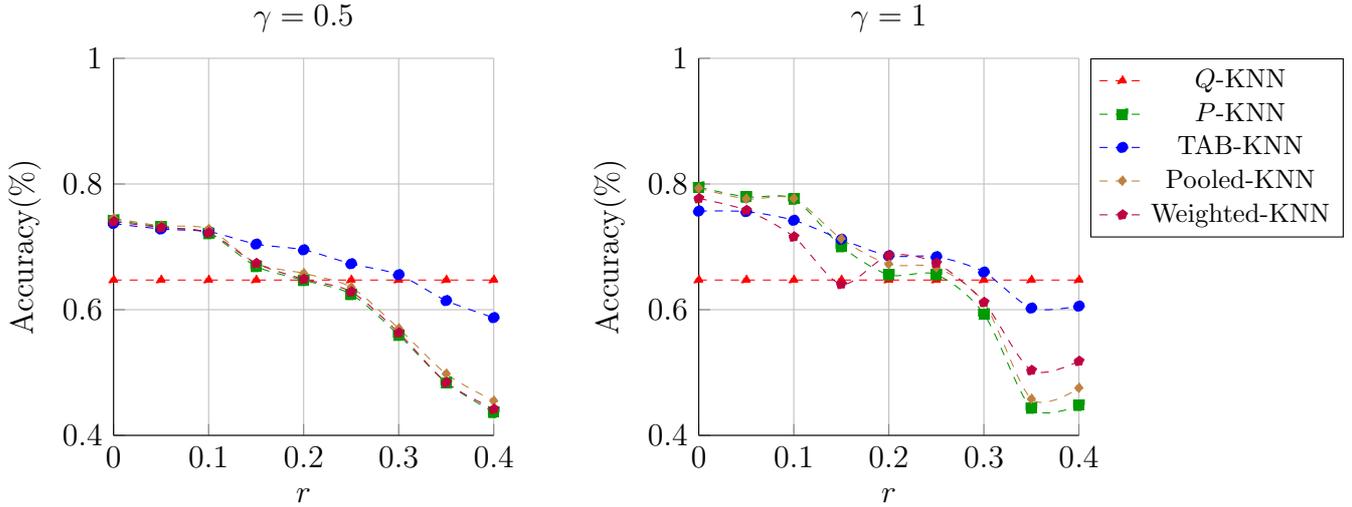
\begin{figure}
    \centering
\begin{subfigure}{.45\textwidth}
  
\begin{tikzpicture}
\begin{axis}[
width=5cm,
   height=5cm,
   scale only axis,
   xmin=0, xmax=0.4,
   xtick={0,0.1,0.2,0.3,0.4},
   xticklabels={0,0.1,0.2,0.3,0.4},
   xmajorgrids,
   ymin=0.4, ymax=1,
   xlabel={$r$},
   ylabel={Accuracy(\%)},
   ymajorgrids,
   title={$\gamma=0.5$},
   axis lines*=left,
   legend style ={ at={(1.03,1)}, 
        anchor=north west, draw=black, 
        fill=white,align=left},
    cycle list name=black white,
    smooth
]

    \addplot[dashed,color=red,mark=triangle*] coordinates{
    (0, 0.64718)
    (0.05, 0.64718)
    (0.1, 0.64718)
    (0.15, 0.64718)
    (0.2, 0.64718)
    (0.25, 0.64718)
    (0.3, 0.64718)
    (0.35, 0.64718)
    (0.4, 0.64718)
   };

   \addplot [dashed,color=green!60!black,mark=square*]coordinates{
(0, 0.74268)
(0.05, 0.73176)
(0.1, 0.72138)
(0.15, 0.66906)
(0.2, 0.64688)
(0.25, 0.62452)
(0.3, 0.5596)
(0.35, 0.48404)
(0.4, 0.43674)
   };

   \addplot [dashed,color=blue,mark=*]coordinates{
(0, 0.73718)
(0.05, 0.7283)
(0.1, 0.72416)
(0.15, 0.70436)
(0.2, 0.69522)
(0.25, 0.6731)
(0.3, 0.65572)
(0.35, 0.6144)
(0.4, 0.58734)
   };

   \addplot [dashed,color=brown,mark=diamond*]coordinates{
(0, 0.74542)
(0.05, 0.73346)
(0.1, 0.72822)
(0.15, 0.67368)
(0.2, 0.65778)
(0.25, 0.63562)
(0.3, 0.57008)
(0.35, 0.49794)
(0.4, 0.45482)
   };

      \addplot [dashed,color=purple,mark=pentagon*]coordinates{
(0, 0.74024)
(0.05, 0.73074)
(0.1, 0.72238)
(0.15, 0.67374)
(0.2, 0.64862)
(0.25, 0.62846)
(0.3, 0.56324)
(0.35, 0.48414)
(0.4, 0.44178)
   };

   \end{axis}
\end{tikzpicture}  
\end{subfigure}
\begin{subfigure}{.45\textwidth}
 \begin{tikzpicture}
\begin{axis}[
width=5cm,
   height=5cm,
   scale only axis,
   xmin=0, xmax=0.4,
   xtick={0,0.1,0.2,0.3,0.4},
   xticklabels={0,0.1,0.2,0.3,0.4},
   xmajorgrids,
   ymin=0.4, ymax=1,
   xlabel={$r$},
   ylabel={Accuracy(\%)},
   ymajorgrids,
   title={$\gamma=1$},
   axis lines*=left,
   legend style ={ at={(1.03,1)}, 
        anchor=north west, draw=black, 
        fill=white,align=left, font=\footnotesize},
    cycle list name=black white,
    smooth
]

    \addplot[dashed,color=red,mark=triangle*] coordinates{
    (0, 0.64718)
    (0.05, 0.64718)
    (0.1, 0.64718)
    (0.15, 0.64718)
    (0.2, 0.64718)
    (0.25, 0.64718)
    (0.3, 0.64718)
    (0.35, 0.64718)
    (0.4, 0.64718)
   };
   \addlegendentry{$Q$-KNN};

   \addplot [dashed,color=green!60!black,mark=square*]coordinates{
(0, 0.79488)
(0.05, 0.77944)
(0.1, 0.77628)
(0.15, 0.70062)
(0.2, 0.65558)
(0.25, 0.6558)
(0.3, 0.59326)
(0.35, 0.44362)
(0.4, 0.44808)
   };
   \addlegendentry{$P$-KNN};

   \addplot [dashed,color=blue,mark=*]coordinates{
(0, 0.7569)
(0.05, 0.75602)
(0.1, 0.74222)
(0.15, 0.71196)
(0.2, 0.68592)
(0.25, 0.68432)
(0.3, 0.65984)
(0.35, 0.60254)
(0.4, 0.60564)
   };
   \addlegendentry{TAB-KNN};

   \addplot [dashed,color=brown,mark=diamond*]coordinates{
(0, 0.79334)
(0.05, 0.77628)
(0.1, 0.77668)
(0.15, 0.71274)
(0.2, 0.67226)
(0.25, 0.66718)
(0.3, 0.61122)
(0.35, 0.45782)
(0.4, 0.4753)
   };
   \addlegendentry{Pooled-KNN};

      \addplot [dashed,color=purple,mark=pentagon*]coordinates{
(0, 0.77700)
(0.05, 0.75814)
(0.1, 0.71598)
(0.15, 0.64050)
(0.2, 0.68628)
(0.25, 0.67404)
(0.3, 0.61174)
(0.35, 0.50350)
(0.4, 0.51808)
   };
   \addlegendentry{Weighted-KNN};

   \end{axis}
\end{tikzpicture}   
\end{subfigure}
    \caption{Accuracy of the TAB $K$-NN classifiers under the scenario with partially flipped sine functions. We experiment with different values of the ratio parameter $r$ for a given $\gamma=0.5$ and $1$. Blue: TAB $K$-NN classifier; Red: $K$-NN classifier on only $Q$-data; Green: $K$-NN classifier on only $P$-data; Brown: $K$-NN classifer on pooled data.  }
    \label{fig:flipped}
\end{figure}

\subsection{Logistic Regression Setting}
We next consider the scenario where both $\eta^Q$ and $\eta^P$ follow the logistic models \eqref{eq:logisticModel}, and the parameters are estimated by \eqref{eq:minimization}. We set $n_Q=200,\ n_P=500,\ Q_X=P_X=N(0,I_d)$, and simulate $50000$ test data points. The linear coefficient is given by 
$$
    \beta_Q=(0.5\cdot\mathbf{1}_s,\mathbf{0}_{d-s}),\quad \beta_P=(1.5\cdot\mathbf{1}_s,\frac{\|\beta_Q\|}{\sqrt{d-s}}\tan\Delta \cdot\mathbf{1}_{d-s}),
$$ 
where $\mathbf{1}_s$  and $\mathbf{0}_s$, respectively, denotes a vector of all $1$s and a vector of all $0$s with size $s$. Here, we set $s=10$, which is small compared to $d$. For the source distribution, $\beta_P$ could be treated as a rotated version of $3\beta_Q$, with an angle of exactly $\Delta$ between them. The range of $\Delta$ is chosen in the set $\{0, 0.2, 0.4, 0.6, 0.8, 1, 1.2, 1.4, 1.6\}$, which gradually approaches $\pi/2$. We simulate $50000$ test data points for each value of $\Delta$. For the TAB threshold, we also choose $\tau=0.05$. The lasso regularization parameter is selected by $5$-fold cross-validation and chosen to be the largest $\lambda$ at which the MSE is within one standard error of the minimum MSE, namely,  $\lambda_{1se}$.

In addition to our proposed TAB classifier, we compare three benchmarks for performance: logistic regression with lasso penalty on $Q$-data, $P$-data, and pooled data. Figure \ref{fig:logistic} demonstrates that our TAB classifier achieves high accuracy when the angle between $\beta_Q$ and $\beta_P$ is small. Interestingly, as evidence of its robustness, our classifier retains some classification ability with the target data even when the angle is large, while the benchmark classifiers based on $P$-data and pooled data suffer from negative transfer.

\begin{figure}
    \centering
 \begin{tikzpicture}
\begin{axis}[
width=8cm,
   height=6cm,
   scale only axis,
   xmin=0, xmax=2,
   xtick={0,0.5,1,1.5,2},
   xticklabels={0,0.5,1,1.5,2},
   xmajorgrids,
   ymin=0.4, ymax=0.8,
   xlabel={$\Delta$},
   ylabel={Accuracy(\%)},
   ymajorgrids,
   axis lines*=left,
   legend style ={ at={(1.03,1)}, 
        anchor=north west, draw=black, 
        fill=white,align=left},
    cycle list name=black white,
    smooth
]

    \addplot[dashed,color=red,mark=triangle*] coordinates{
    (0, 0.66594)
(0.25, 0.66594)
(0.5, 0.66594)
(0.75, 0.66594)
(1, 0.66594)
(1.25, 0.66594)
(1.5, 0.66594)
(1.75, 0.66594)
   };
   \addlegendentry{$Q$-Lasso Logistic};

   \addplot [dashed,color=green!60!black,mark=square*]coordinates{
    (0, 0.73674)
(0.25, 0.73472)
(0.5, 0.71608)
(0.75, 0.70018)
(1, 0.5758)
(1.25, 0.54726)
(1.5, 0.50326)
(1.75, 0.4693)
   };
   \addlegendentry{$P$-Lasso Logistic};

   \addplot [dashed,color=blue,mark=*]coordinates{
(0, 0.72776)
(0.25, 0.72266)
(0.5, 0.71138)
(0.75, 0.706)
(1, 0.67704)
(1.25, 0.64846)
(1.5, 0.63306)
(1.75, 0.63782)
   };
   \addlegendentry{TAB-Lasso Logistic};

   \addplot [dashed,color=brown,mark=diamond*]coordinates{
(0, 0.73794)
(0.25, 0.73012)
(0.5, 0.73378)
(0.75, 0.72)
(1, 0.68256)
(1.25, 0.64046)
(1.5, 0.55328)
(1.75, 0.52656)
   };
   \addlegendentry{Pooled-Lasso Logistic};

   \end{axis}
\end{tikzpicture}   
    \caption{Accuracy of the TAB logistic classifier with lasso penalty. We conduct experiments with difference choices of the angle $\Delta\in[0,\pi/2]$. Blue: TAB logistic classifiers with lasso penalty; Red: Logistic classifier with lasso penalty on only $Q$-data; Green: Logistic classifier with lasso penalty on only $P$-data; Brown: Logistic classifier with lasso penalty on pooled data.}
    \label{fig:logistic}
\end{figure}
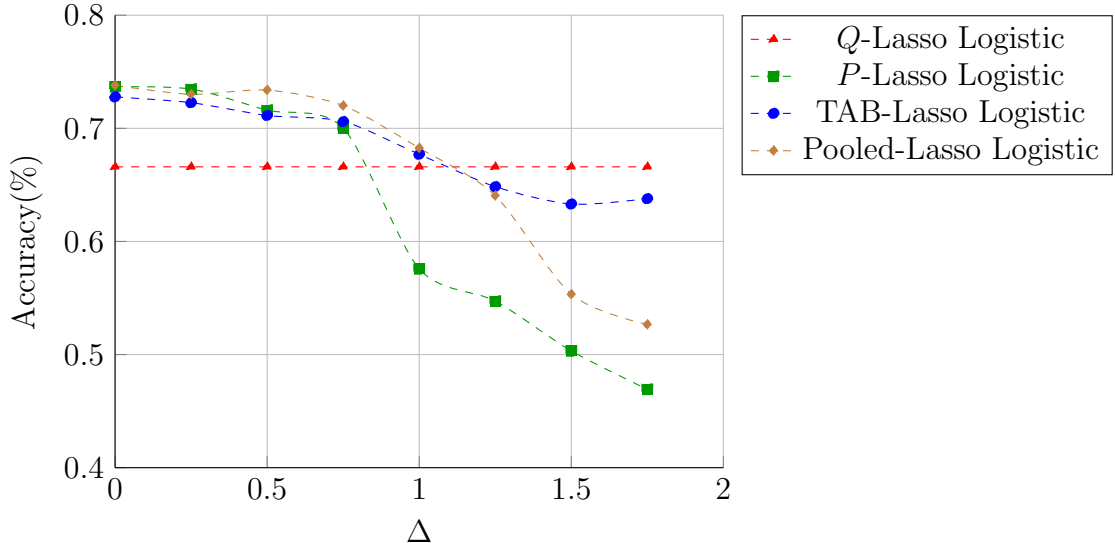
\section{Conclusion}
In this paper, we have proposed a new approach to transfer learning that is robust against an unreliable source distribution with arbitrary ambiguity in the source data. Our work uses a different way of transferring the source data and, in particular, encompasses the non-parametric setting in \cite{ttcai2020classification} and \cite{reeve2021adpative} and parametric setting in \cite{ttcai2020highDimensional}. By introducing the ambiguity level, our approach enables us to understand the circumstances under which we can improve classification performance, given source data with potential ambiguity. Our proposed TAB classifier, with a threshold $\tau$ that balances the performance of both the target and source data, is shown to be both \emph{efficient} and \emph{robust}, as the excess risk improves for a reliable source distribution and avoids negative transfer with an unreliable source distribution. Furthermore, we provide simple approaches to bound the signal transfer risk, a key component of the excess risk in our general convergence result, in terms of the conventional excess risk extensively studied in the literature of statistical learning.

We then demonstrate the power of our approach on specific classification tasks, with a focus on non-parametric classification and logistic regression settings. The upper bounds are shown to be optimal up to some logarithmic factors and are more general than previous work on transfer learning. Simulation studies provide numerical evidence for these two classification tasks.

There are several promising avenues for future research that may build on the contributions of this paper. One potential direction is to consider an extension of the signal strength and ambiguity level that incorporates a translation parameter, i.e., 
$$
  s_\kappa(x):=
\begin{cases}
 |\eta^P(x)-\frac{1}{2}-\kappa|,\quad &\mathrm{sgn}\left(\eta^Q(x)-\frac{1}{2}-\kappa\right)\times(\eta^P(x)-\frac{1}{2})\geq 0,\\
0,\quad &\text{otherwise.}   
\end{cases}
$$
$$\E_{(X,Y)\sim Q}\left[\Big|\eta^Q(X)-\frac{1}{2}\Big|\mathbf{1}\left\{s_\kappa(X)\leq C_\gamma\Big|\eta^Q(X)-\frac{1}{2}\Big|^\gamma\leq C_\gamma z^\gamma\right\}\right]\leq \varepsilon_\kappa(z;\gamma,C_\gamma).$$
This extension is natural since the decision boundary $\{x\in\Omega:\eta^Q(x)=1/2\}$ may be similar to $\{x\in\Omega:\eta^P(x)=1/2+\kappa\}$ (they are identical under the settings of \cite{reeves2019labelNoise} and \cite{abramovich2019logistic}). We conjecture that an additional estimation error of $(\frac{\log n_Q}{n_Q})^{\frac{1+\alpha}{2+\alpha}}$ may be incurred due to the presence of an unknown $\kappa$, and an empirical risk minimization procedure after obtaining $\hat{\eta}^P$ such as  
$$\hat{\kappa}=\argmin_{\kappa\in[-\frac{1}{2},\frac{1}{2}]}\frac{1}{n}\sum_{i=1}^n\mathbf{1}\left\{\mathbf{1}\Big\{\hat{\eta}^P(X_i)\geq \frac{1}{2}+\kappa\Big\}\neq Y_i\right\}$$
may be required to obtain a corrected version of $\hat{f}^P$ over the target data.

Another direction is to develop an adaptive and generic procedure for selecting the threshold $\tau$ while preserving optimality. There are three possible directions. Firstly, it is an open question whether simple Empirical Risk Minimization (ERM) methods for choosing $\tau$ still keeps the upper bound optimal. Secondly, Lepski's method \citep{lepski1993} may offer a solution for maintaining optimal rates with an adaptive choice of $\tau$. Thirdly, we conjecture that choosing $\tau\asymp \log(n_Q\lor n_P)k_Q^{-1/2}$, where $\hat{\eta}^Q$ is an $M$-estimator consisting of a simple average of $k_Q$ terms, is helpful for obtaining an optimal rate.

As a final future direction, the conditions presented in the non-parametric and logistic model settings could be relaxed to some extent, for instance, by considering non-compact feature spaces with sub-Gaussian conditions or other marginal distribution assumptions (e.g., Assumption A4 in \cite{gadat2016KNN}). Our proposed TAB classifier is expected to perform well in these settings without a significant modification of the framework, and may offer advantages over other case-specific estimators in the transfer learning literature.

\begin{acks}[Acknowledgments]
Fan is supported by ONR grant N00014-22-1-2340, and NSF grants DMS-2052926, DMS-2053832 and DMS-2210833.
Klusowski is supported by NSF grants CAREER DMS-2239448, DMS-2054808, and HDR TRIPODS CCF-1934924.
\end{acks}

\begin{supplement}
\stitle{Supplement to ``Robust Transfer Learning with Unreliable Source Data"}
\sdescription{Please refer to the Supplementary Material for a more detailed explanation of certain topics and proofs of the results.}
\end{supplement}

\bibliography{File/reference}

\newpage
\title{Supplement to ``Robust Transfer Learning with Unreliable Source Data"}
\begin{aug}
\author{\fnms{Jianqing}~\snm{Fan}\ead[label=e1]{jqfan@princeton.edu}},
\author{\fnms{Cheng}~\snm{Gao}\ead[label=e2]{chenggao@princeton.edu}}
\and
\author{\fnms{Jason}~\snm{M. Klusowski}\ead[label=e3]{jason.klusowski@princeton.edu}}

\address{Department of Operations Research and Financial Engineering, Princeton University\printead[presep={,\ }]{e1,e2,e3}}
\end{aug}
\maketitle

   Appendix \ref{sec:appA} gives further illustration on certain topics in the main text of \cite{submitted}. Appendix \ref{sec:appB} - \ref{sec:appAux} collect detailed proofs of all results.

\begin{appendix}
\section{Supplementary Results}
\label{sec:appA}
\subsection{Data-driven Approaches to Determine the Threshold Parameter}
\label{sec:adaptiveTau}
In this section, we analyze the following three recommendations for adaptively choosing the threshold parameter $\tau$ in practice:

\textbf{I. Simple rule of thumb.} Given the target sample size \(n_Q\), we recommend any value between \(0.5n_Q^{-1/2}\) and \(n_Q^{-1/2}\). In Section \ref{sec:simulation}, the choice of \(\tau = 0.05\) also falls into this interval and works well.

\textbf{II. Adjust \(\tau\) from the (adaptive) selection of base model parameters.} According to the general convergence results, a tight choice is \(\tau \asymp \log(n_Q \lor n_P) \delta_Q\), where \(\delta_Q\) is the concentration level determined by the classification model. Therefore, any insight into the characteristics of \(\delta_Q\), including its order and representing parameters, and data-driven methods to estimate these parameters, will help construct a good estimate of \(\tau\) that depends on a good estimate of \(\delta_Q\).

For the cases in our application studies in Sections \ref{sec:nonPara} and \ref{sec:para}, further investigation on the concentration properties of $K$-NN and logistic models shows:
$$
\begin{aligned}
&(K\text{-NN}) \quad \tau \asymp \log(n_Q \lor n_P)k_Q^{-1/2}, \\
&(\text{Logistic}) \quad \tau \asymp \log(n_Q \lor n_P)\lambda_Q.
\end{aligned}
$$
Therefore, any good estimate of \(k_Q\) in the $K$-NN case and \(\lambda_Q\) in the logistic case directly leads to a good estimate of \(\tau\). After simple testing in simulations, we find that a good choice of the constant term is 0.2, i.e.,
$$
\begin{aligned}
&(K\text{-NN}) \quad \tau = 0.2 \log(n_Q \lor n_P)k_Q^{-1/2}, \\
&(\text{Logistic}) \quad \tau = 0.2 \log(n_Q \lor n_P)\lambda_Q.
\end{aligned}
$$
To find good choices for \(k_Q\) and \(\lambda_Q\), besides the theoretical recommendations proposed in our work, one can also use data-driven approaches such as empirical risk minimization and cross-validation. In Section \ref{sec:simulation}, we use the theoretical recommendation for determining \(k_Q\) and 5-fold cross-validation techniques to determine \(\lambda_Q\). For detailed analysis on the theoretical behavior of cross-validation, see \cite{celisse2015TheoreticalAO} for the $K$-NN model and \cite{chetverikov2021lasso} for the lasso-penalized models.

\textbf{III. Empirical Risk Minimization.} A more reliable approach than the rule of thumb might be choosing the best candidate for the parameter $\tau$ from a purely data-driven method. Denote the TAB classifiers with the same classifier components \(\mathbf{1}\{\hat{\eta}^Q(x) \geq 1/2\}\) and \(\hat{f}^P(x)\), but different choices of \(\tau\), as:
$$
\hat{f}_{TAB}(x;\tau_i)=
\begin{cases}
\mathbf{1}\{\hat{\eta}^Q(x)\geq \frac{1}{2}\},\quad &\text{if }|\hat{\eta}^Q(x) - \frac{1}{2}|\geq \tau_i,\\
\hat{f}^P(x),\quad &\text{otherwise},
\end{cases}\quad i=1,2,\dots,M
$$
given a candidate list \(0\leq\tau_1<\tau_2<\dots<\tau_M\leq 1/2\) for some \(M > 0\).
Following a similar proof approach as in \cite{tsybakov2004optimal} and \cite{reeve2021adpative}, the following lemma controls the additional error incurred by selecting the best \(\tau\) using $0-1$ empirical risk minimization:
\begin{lemma}[Empirical Risk Minimization]
\label{lemma:erm}
Suppose that Assumption \ref{assum:margin} holds with $\alpha\geq 0,\ C_\alpha>0$. Let $$j^*\in\mathrm{argmin}_{1\leq j\leq M}\frac{1}{n_Q}\sum_{i=1}^{n_Q}\mathbf{1}\{\hat{f}_{TAB}(X_i;\tau_j)\neq Y_i\}$$ with $\hat{f}_{TAB}^{ERM}(\cdot)=\hat{f}_{TAB}(\cdot;\tau_{j^*})$. Also define $\mathcal{E}_Q^*:=\min_{1\leq j\leq M}\mathcal{E}_Q(\hat{f}_{TAB}(\cdot;\tau_j))$. Then, for any $\delta\in(0,1]$, we have
$$
\sup_{(Q,P)\in\Pi}\pr_{\mathcal{D}_Q}\Big\{\mathcal{E}_Q(\hat{f}^{ERM}_{TAB})\geq 2\mathcal{E}_Q^*+64C_\alpha^{\frac{1}{2+\alpha}}\Big(\frac{\log(2M/\delta)}{n}\Big)^{\frac{1+\alpha}{2+\alpha}}\Big\}\leq \delta.
$$
Furthermore, we have
$$\sup_{(Q,P)\in\Pi}\E\mathcal{E}_Q(\hat{f}^{ERM}_{TAB})\lesssim \mathcal{E}_Q^* + \Big(\frac{\log(2M)}{n}\Big)^{\frac{1+\alpha}{2+\alpha}}.$$
\end{lemma}
Based on Lemma \ref{lemma:erm}, it is natural to ask if we can design a reasonably sized candidate list that contains a good choice of $\tau$ to effectively upper bound the excess risk as shown in Theorem \ref{thm:general}. The following theorem confirms this by uniformly selecting the parameters with equal step size:
\begin{theorem}[Empirical Risk Minimization]
\label{thm:erm}
Keep the assumptions and notations made in Theorem \ref{thm:general} and Lemma \ref{lemma:erm}. For any constant $-c>0$ such that $n_Q^{c}\ll \delta_Q$, if we choose the candidate list as
$$M=\lfloor n_Q^{c}/2 \rfloor,\quad \tau_1=n_Q^{-c},\quad \tau_{j+1}=\tau_j+n_Q^{-c},\quad j=1,2,\dots,M-1,$$ then when $n_Q$ is sufficiently large, we have
$$
\sup_{(Q,P)\in\Pi}\pr_{\mathcal{D}_Q}\Big\{\mathcal{E}_Q(\hat{f}^{ERM}_{TAB})\geq 2\mathcal{E}_Q^{TAB}+64C_\alpha^{\frac{1}{2+\alpha}}\Big(\frac{\log(n_Q^c/\delta)}{n}\Big)^{\frac{1+\alpha}{2+\alpha}}\Big\}\leq \delta,
$$
and
$$\sup_{(Q,P)\in\Pi}\E\mathcal{E}_Q(\hat{f}^{ERM}_{TAB})\lesssim \mathcal{E}_Q^{TAB} + \Big(\frac{\log(n_Q^c)}{n}\Big)^{\frac{1+\alpha}{2+\alpha}},$$
where $\mathcal{E}_Q^{TAB}\lesssim \left(\sup_{(Q,P)\in\Pi}\E_{\mathcal{D}_P}\xi(\hat{f}^P)+ \varepsilon(2\log(n_Q\lor n_P)\delta_Q)\right)\land \tau^{1+\alpha} +\delta_f$.
\end{theorem}
By comparison with the risk upper bound in Theorem \ref{thm:general}, the bound in Theorem \ref{thm:erm} incurs an additional term of the order \( n_Q^{-\frac{1+\alpha}{2+\alpha}} \) up to logarithmic factors, making the upper bound potentially sub-optimal. However, we conjecture that this additional risk term due to ERM will not significantly hurt empirical performance, and \(\hat{f}^{ERM}_{TAB}\) can achieve similar accuracy to the original \(\hat{f}_{TAB}\) with a theoretically reasonable choice in practice. Additionally, the term \( n_Q^{-\frac{1+\alpha}{2+\alpha}} \) is dominated by \( n_Q^{-\frac{\beta(1+\alpha)}{2\beta+d}} \) and \(\sqrt{\frac{s \log d}{n_Q}}\), the conventional rate without source data in the $K$-NN and logistic modeling cases. Therefore, under the regime of a large \(n_P\) and a small ambiguity level, we could still improve the excess risk convergence rate and gain from transfer learning from a theoretical perspective.

Besides the theoretical guarantee of the excess risk upper bound, it is also meaningful to discuss obtaining the \(\hat{f}^{ERM}_{TAB}\) from an algorithmic perspective. Although minimizing the empirical risk with respect to the \(0-1\) loss, as considered by \cite{audibert2011fast, tsybakov2004optimal, tsybakov2004smoothDiscriminationAnalysis}, is generally computationally infeasible, we can still obtain the ERM TAB classifier through a linear search due to the simple and special structure of the focal parameter \(\tau\) in our model. Note that for any $j=1,2,\dots,M-1$, we have
$$\hat{f}_{TAB}(x;\tau_{j+1})=\hat{f}_{TAB}(x;\tau_{j})=\hat{f}^P(x)\quad \text{if}\quad |\hat{\eta}^Q(x)-1/2|\in[0,\tau_j).$$
$$\hat{f}_{TAB}(x;\tau_{j+1})=\hat{f}_{TAB}(x;\tau_{j})=\mathbf{1}\{\hat{\eta}^Q(x)\geq 1/2\}\quad \text{if}\quad |\hat{\eta}^Q(x)-1/2|\in[\tau_{j+1},1/2].$$
Thus, we have
$$\hat{f}_{TAB}(x;\tau_{j+1})=\hat{f}_{TAB}(x;\tau_{j})\quad \text{if}\quad |\hat{\eta}^Q(x)-1/2|\in[0,\tau_j)\cup[\tau_{j+1},1/2],$$
which deduces
$$
\begin{aligned}
&\mathbf{1}\{\hat{f}_{TAB}(x;\tau_{j+1})\neq Y_i\}-\mathbf{1}\{\hat{f}_{TAB}(x;\tau_{j})\neq Y_i\}\\
=&\begin{cases}
0,\quad &\text{if }|\hat{\eta}^Q(x)-1/2|\in[0,\tau_j)\cup[\tau_{j+1},1/2],\\
\mathbf\{\mathbf{1}\{\hat{\eta}^Q(x)\geq 1/2\}\neq f^P(x)\}, &\text{if }|\hat{\eta}^Q(x)-1/2|\in[\tau_j,\tau_{j+1}).  
\end{cases}
\end{aligned}
$$
Therefore, we have
$$
\begin{aligned}
&\sum_{i=1}^{n_Q}\mathbf{1}\{\hat{f}_{TAB}(X_i;\tau_{j+1})\neq Y_i\}-\mathbf{1}\{\hat{f}_{TAB}(X_i;\tau_{j})\neq Y_i\}\\
=&\sum_{i:|\eta^Q(x)-1/2|\in[\tau_j,\tau_{j+1})}\mathbf{1}\{\mathbf{1}\{\hat{\eta}^Q(x)\geq 1/2\}\neq f^P(x)\}=:\mathrm{diff}_j.
\end{aligned}
$$
Therefore, we could create \(M-1\) variables \(\mathrm{diff}_1, \dots, \mathrm{diff}_{M-1}\), initialized as zero, and compute them with a simple traversal over the target data. Then, we could immediately obtain the ERM classifier by calculating all \(\frac{1}{n_Q}\sum_{i=1}^{n_Q}\mathbf{1}\{\hat{f}_{TAB}(X_i; \tau_j) \neq Y_i\}\) using the following procedure:
\begin{itemize}
    \item \textbf{Step I.} Calculate $\{\mathrm{diff}_j\}_{1\leq j\leq M-1}$ using a simple search over all target data $\{(X_i,Y_i)\}_{1\leq i\leq n_Q}$.\\
    \emph{Time complexity:} $O(n_Q)$, \emph{Space complexity:} $O(M)$.
    \item \textbf{Step II.} Calculate $\frac{1}{n_Q}\sum_{i=1}^{n_Q}\mathbf{1}\{\hat{f}_{TAB}(X_i;\tau_{1})\neq Y_i\}$.\\
    \emph{Time complexity:} $O(n_Q)$, \emph{Space complexity:} $O(1)$.
    \item \textbf{Step III.} Calculate $\frac{1}{n_Q}\sum_{i=1}^{n_Q}\mathbf{1}\{\hat{f}_{TAB}(X_i;\tau_{j})\neq Y_i\}$ for all $1\leq j \leq M$ using the following formula:
    $$\frac{1}{n_Q}\sum_{i=1}^{n_Q}\mathbf{1}\{\hat{f}_{TAB}(X_i;\tau_{j+1})\neq Y_i\}=\frac{1}{n_Q}\sum_{i=1}^{n_Q}\mathbf{1}\{\hat{f}_{TAB}(X_i;\tau_{j})\neq Y_i\}+\mathrm{diff}_j/n_Q,\quad j=1,2,\dots,M-1.$$
    \emph{Time complexity:} $O(M)$, \emph{Space complexity:} $O(M)$.
    \item \textbf{Step IV.} Obtain $j^*\in\mathrm{argmin}_{1\leq j\leq M}\frac{1}{n_Q}\sum_{i=1}^{n_Q}\mathbf{1}\{\hat{f}_{TAB}(X_i;\tau_j)\neq Y_i\}$ and $\hat{f}_{TAB}^{ERM}(\cdot)=\hat{f}_{TAB}(\cdot;\tau_{j^*})$.\\
    \emph{Time complexity:} $O(M)$, \emph{Space complexity:} $O(1)$.
\end{itemize}

\subsection{General Signal Transfer Risk Bound for Plug-in Rules}
\label{sec:pluginSTR}
Suppose that $\hat{f}^P$ can be expressed in the form of a plug-in rule, i.e., $\hat{f}^P(\cdot)=\mathbf{1}\{\hat{\eta}^P(\cdot)\geq \frac{1}{2}\}$, where $\hat{\eta}^P(\cdot)$ is an estimate of the regression function $\eta^P(\cdot)$. For notational simplicity, define the area with strong signal strength as
$$\Omega^+(\gamma,C_\gamma):=\{x\in\Omega:s(x)\geq C_\gamma |\eta^Q(x)-\frac{1}{2}|^\gamma\}.$$
Since $\Omega^+(\gamma,C_\gamma)$ indicates the informative area where $\eta^P$ provides strong signal relative to $\eta^Q$, $\hat{\eta}^P$ could also give indications on $\eta^Q$ if it approximates $\eta^P$ well.

In order to obtain a complete convergence rate for the signal transfer risk, it is necessary to impose some conditions on the behavior of $\hat{\eta}^P$. The following result requires that $\hat{\eta}^P$ recovers $\eta^P$ in the sense that the misclassification rate is upper bounded.

\begin{theorem}
Let $\hat{\eta}^P$ be an estimator of the regression function $\eta^P$ depending on $\mathcal{D}_P$. Suppose that two $n_P$-sequences $\delta_P, \delta_{P,f}$ satisfy that, with probability at least $1-\delta_{P,f}$,  for some constant $C_2>0$,
\begin{equation}
    \sup_{(Q,P)\in \Pi}\pr_{\mathcal{D}_P}\left((\hat{\eta}^P(x)-\frac{1}{2})(\eta^Q(x)-\frac{1}{2})<0|X^P_{1:n_P}\right)\leq C_2\exp\left(- (\frac{C_\gamma|\eta^Q(x)-\frac{1}{2}|^{\gamma}}{\delta_P})^2\right)
    \label{eq:misclassP}
\end{equation}
holds except at $\Omega^b\subset \Omega^+(\gamma,C_\gamma)$. Then we have
$$\sup_{(Q,P)\in \Pi}\E_{\mathcal{D}_P}\xi(\hat{f}^P)\lesssim \delta_P^{\frac{1+\alpha}{\gamma}}+\delta_{P,f}+\delta_P^b,$$
where $\delta_P^b=\sup_{(Q,P)\in\Pi}\int_{\Omega^b}|\eta^Q(x)-\frac{1}{2}|dQ_X$.
\label{thm:STRPlugIn}
\end{theorem}

Theorem \ref{thm:STRPlugIn} does not explicitly require the concentration property for $\hat{\eta}^P$ around the true regression function $\hat{\eta}^P$. Instead, it bounds the misclassification rate for a given query point by $C_2\exp(-(\frac{|\eta^P(x)-1/2|}{\delta_P})^2)$ which is consistent with the signal provided by $\eta^P$ relative to $\eta^Q$ within the signal transfer set.

Similar to $\delta_f$, the quantity $\delta_{P,f}$ denotes the (typically small) probability that the misclassification rate bound \eqref{eq:misclassP} fails. To account for the case where the query point is near the boundary of the signal transfer set $\Omega^+(\gamma,C_\gamma)$, and the estimate $\hat{\eta}^P(x)$ may be influenced by its neighborhood outside the set, we introduce a new term $\delta_P^b$ in addition to $\delta_{P,f}$. This allows for the possibility of failure of the misclassification rate bound \eqref{eq:misclassP} near the boundary of the signal transfer set. This situation may occur in some non-parametric methods, such as $K$-NN models, and needs to be considered in theoretical analysis.

The following simple proposition shows that \eqref{eq:misclassP} is weaker than the classical exponential concentration inequality.

\begin{proposition}
Let $\hat{\eta}^P$ be an estimator of the regression function $\eta^P$ depending on $\mathcal{D}_P$. For any $x\in\Omega^+(\gamma,C_\gamma)$, the event that for any $t>0$,
$$\sup_{(Q,P)\in\Pi}\pr_{\mathcal{D}_P}\left(|\hat{\eta}^P(x)-\eta^P(x)|\geq t|X^P_{1:n_P}\right)\leq C_2 \exp\left(-(\frac{t}{\delta_P})^2\right)$$
belongs to the event that \eqref{eq:misclassP} holds.
\label{prop:concentrationPtoMisclassP}
\end{proposition}
\begin{proof}
The result is straightforward by observing that $$\pr_{\mathcal{D}_P}((\hat{\eta}^P(x)-\frac{1}{2})(\eta^Q(x)-\frac{1}{2})<0|X^P_{1:n_P})\leq \pr_{\mathcal{D}_P}(|\hat{\eta}^P(x)-\eta^P(x)|\geq |\eta^P(x)||X^P_{1:n_P}),$$ of which the latter term is less than $C_2\exp (-(\frac{|\eta^P(x)|}{\delta_P})^2)$. Since we require that $x\in\Omega^+(\gamma,C_\gamma)$, it is also less than $C_2\exp(- (\frac{C_\gamma|\eta^Q(x)-1/2|^{\gamma}}{\delta_P})^2)$.
\end{proof}

\subsection{Non-parametric Classification with General Form of Ambiguity}
\label{sec:nonParaGeneral}
In this section, we focus on the general parametric space $\Pi^{NP}$ under the non-parametric classification scenario. Under this scenario, it is necessary to discuss the relationship between the parameters $\alpha,\gamma,\beta$ and $\beta_P$. We claim that, if $\gamma\beta\geq \beta_P$, we can rule out a ``super-smooth" source regression function $\eta^P$. The motivation behind this assumption is that if $\gamma\beta< \beta_P$, the smoothness of $\eta^P$ will potentially restrict the family of the target distribution $Q$.
To see this,  we define the family of target distributions as follows: $$\Pi^{NP}_Q:=\{Q:Q\in\mathcal{M}(\alpha,C_\alpha), \eta^Q\in\mathcal{H}(\beta,C_\beta),Q_X\in\mathcal{S}(\mu)\},$$ where $\mathcal{M}(\alpha,C_\alpha)$ denotes the set of distributions that satisfies the margin assumption (Assumption \ref{assum:margin}) with parameters $\alpha\geq 0$ and $C_\alpha>0$. The following proposition shows that when $\gamma\beta< \beta_P$, $\Pi^{NP}_Q$ is not a subset of the slice of $\Pi^{NP}$ on $Q$ when setting $\varepsilon(z)\equiv 0$.

\begin{proposition} \label{prop:sliceQ}
Suppose $\Pi^{NP}$ satisfies $\varepsilon(z)\equiv 0$. Then,
\begin{enumerate}
    \item If $\gamma\beta\geq \beta_P$, then $\Pi^{NP}_Q= \{Q:(Q,P)\in\Pi^{NP}\}$ for $C_{\beta_P}\geq C_\gamma 2^{-\gamma}\lor 2C_\gamma C_\beta^{\frac{\beta_P}{\beta}}2^{\gamma-\beta_P/\beta}$.
    \item If $\gamma\beta< \beta_P$, then $\Pi^{NP}_Q\not\subset \{Q:(Q,P)\in\Pi^{NP}\}$.
\end{enumerate}
\end{proposition}
Proposition \ref{prop:sliceQ} implies that when $\beta_P$ is sufficiently large, a too-small $\gamma$ imposes an additional smoothness condition on $Q$. To avoid this issue, we assume that a large enough $\gamma\geq \beta_P/\beta$. The upper bound results is then as follows.

\begin{theorem}[Non-parametric Classification Upper Bound, General Case] 
Suppose that \\$n_Q^{\frac{d}{2\beta+d}}\exp (-c_Qn_Q^{\frac{2\beta}{2\beta+d}})\lesssim n_P^{-\frac{\beta(1+\alpha)}{2\gamma\beta+d}}$.  Then the TAB $K$-NN classifier $\hat{f}^{NN}_{TAB}(x)$ satisfies

\begin{itemize}
    \item $\gamma\beta\geq \beta_P:$ For any $(Q,P)\in\Pi^{NP}$,
    
    \begin{equation}
\begin{aligned}
     \E_{(\mathcal{D}_Q,\mathcal{D}_P)} \mathcal{E}_Q(\hat{f}^{NN}_{TAB})&\lesssim \left(n_P^{-\frac{\beta(1+\alpha)}{2\gamma\beta+d}}+\varepsilon_b+\varepsilon(2\tau)\right)\land \left(\log^{1+\alpha} (n_Q\lor n_P) n_Q^{-\frac{\beta(1+\alpha)}{2\beta+d}}\right)\\
      &\lesssim \left( n_P^{-\frac{\beta_P(1+\alpha)/\gamma}{2\gamma\beta+d}}+\varepsilon(2\tau)\right)\land \left(\log^{1+\alpha} (n_Q\lor n_P) n_Q^{-\frac{\beta(1+\alpha)}{2\beta+d}}\right),
\end{aligned}
\label{eq:nonParaUpper}
\end{equation}
where $\varepsilon_b= \int_{\Omega^b}|\eta^Q(X)-\frac{1}{2}|dQ_X$ and the boundary of the signal transfer set is defined as
    $$\Omega^b:=\{x\in\Omega: |\eta^Q(x)-\frac{1}{2}|\leq c_b n_P^{-\frac{\beta_P/\gamma}{2\gamma\beta+d}},
 B(x,c_bn_P^{-\frac{1}{2\gamma\beta+d}})\cap \Omega\cap\Omega_P \not\subset \Omega^+(\gamma,C_\gamma)\}
    $$
    for some constant $c_b>0$.
    \item $\gamma\beta< \beta_P:$ If we choose $k_P=\lfloor c_P n_P^{\frac{2\beta_P}{2\beta_P+d}}\rfloor$ for any $c_P>0$ instead, then
    \begin{equation}
        \sup_{(Q,P)\in\Pi^{NP}}\E_{(\mathcal{D}_Q,\mathcal{D}_P)} \mathcal{E}_Q(\hat{f}^{NN}_{TAB})\lesssim \left( n_P^{-\frac{\beta_P(1+\alpha)/\gamma}{2\beta_P+d}}+\varepsilon(2\tau)\right)\land \left(\log^{1+\alpha} (n_Q\lor n_P) n_Q^{-\frac{\beta(1+\alpha)}{2\beta+d}}\right).
        \label{eq:nonParaUpper2}
    \end{equation}
\end{itemize}

\label{thm:nonParaGeneral}
\end{theorem}

When $\gamma\beta\geq \beta_P$, the optimal ``smoothness" parameter for the source data should be $\gamma\beta$, provided by $\Omega^+(\gamma, C_\gamma)$, as our upper bound (except the boundary risk term $\varepsilon_b$) and the choice of $k_P$ do not involve $\beta_P$. In other words, our upper bound automatically smooths $\eta^P$ implicitly through the connection between the area with strong signal strength and $\eta^Q$. However, if $\beta_P>\gamma\beta$, the upper converges faster as the smoothness of $\eta^P$ also has implicit information about $\eta^Q$.

We next provide lower bound results for the parameter space $\Pi^{NP}$. While the scenario of $\gamma\beta<\beta_P$ causes a ``phase transition" transition with a smaller upper bound, our obtained lower bound is even smaller under this scenario and does not match the lower bound. How to develop a optimal and robust upper bound methodology under this case is still open.

\begin{theorem}[Non-parametric Classification Lower Bound, General Case] (a) Fix the parameters in the definition of $\Pi^{NP}$ with $\alpha\beta\leq d,\gamma\beta\geq\beta_P$. We have that for some constant $c>0$,
$$
\inf_{\hat{f}}\sup_{\substack{(Q,P)\in\Pi^{NP}\\\Omega=\Omega_P}}\E\mathcal{E}_Q(\hat{f})
\gtrsim  \left( n_P^{-\frac{\beta(1+\alpha)}{2\gamma\beta+d}}+\varepsilon(cn_Q^{-\frac{\beta}{2\beta+d}})\right)\land n_Q^{-\frac{\beta(1+\alpha)}{2\beta+d}}.
$$
\noindent (b) Fix the parameters in the definition of $\Pi^{NP}$ with $\max\{\alpha\beta,\alpha\beta_P/\gamma\}\leq d,\gamma\beta<\beta_P$. We have that for some constant $c>0$,
$$
\inf_{\hat{f}}\sup_{\substack{(Q,P)\in\Pi^{NP}\\\Omega=\Omega_P}}\E\mathcal{E}_Q(\hat{f})
\gtrsim  \left( n_P^{-\frac{\beta_P(1+\alpha)/\gamma}{2\beta_P+d}}+\varepsilon(cn_Q^{-\frac{\beta}{2\beta+d}})\right)\land  n_Q^{-\frac{\beta_P(1+\alpha)/\gamma}{2\beta_P/\gamma+d}}.
$$
\label{thm:nonParaGeneralLowerBound}
\end{theorem}
Theorem \ref{thm:nonParaGeneralLowerBound} indicates that when $\gamma\beta\geq \beta_P$, as long as the risk term $\varepsilon_b$ is dominated in that
 \begin{equation}
     \varepsilon_b\lesssim n_P^{-\frac{\beta(1+\alpha)}{2\gamma\beta+d}}+\varepsilon(2\tau),
     \label{eq:deltabDominated}
 \end{equation}
our obtained excess risk upper bound is optimal up to some logarithmic factors. This additional risk term $\varepsilon_b$ can be attributed to the set of points whose neighboring source data tend to fall outside $\Omega^+(\gamma,C_\gamma)$. We claim that if a point $x$ belonging to $\Omega^+(\gamma,C_\gamma)$ is in close proximity to the boundary of $\Omega^+(\gamma,C_\gamma)$, that is, $$B(x,r_P)\cap \Omega_P\not\subset \Omega^+(\gamma,C_\gamma)$$
for some small $r_p\asymp n_P^{-\frac{1}{2\gamma\beta+d}}$, it may be difficult to accurately classify this point using $\hat{\eta}^P$. This difficulty arises due to the tendency of the neighboring source data of $x$ to fall outside the signal transfer set. As a result, the small neighborhood of $x$ might not provide enough information to correctly classify $\eta^P(x)$, leading to inaccurate classification. Giving up the correct classification of these points generate such risk term $\varepsilon_b$. We also take the smoothness parameter $\beta_P$ into consideration since it influences the smoothness of $\eta^P$, and thus controls the distortedness of the boundary of the signal transfer set.

A trivial asymptotic upper bound of $\varepsilon_b$ is $n_P^{-\frac{\beta_P(1+\alpha)/\gamma}{2\gamma\beta+d}}$. Unfortunately, this trivial bound strictly dominates the risk term $n_P^{-\frac{\beta(1+\alpha)}{2\gamma\beta+d}}$ when $\beta_P<\gamma\beta$. Hence, we need additional conditions on the source distribution to ensure a smaller $\varepsilon_b$. 

From the proof of Theorem \ref{thm:nonParaSpecial}, we show that those three listed cases satisfy \eqref{eq:deltabDominated}. The following corollary lists two more special cases to help readers to better understand the upper bound result we obtained. In all of these cases, we can explicitly express $\varepsilon_b$ and verify that it satisfies \eqref{eq:deltabDominated}, which shows that our upper bound achieves the optimal rate up to logarithmic factors.

\begin{corollary}
\label{corol:deltaPbDominated}   
Keep the notations and parameter choices in Theorem \ref{thm:nonParaGeneral}. Define the topology boundary point of any set $A$ in $\mathbb{R}^d$ as $\partial A$. Suppose that $\gamma\beta\geq \beta_P$. Consider the following additional conditions added to $(Q,P)\in\Pi^{NP}$, respectively:
\begin{enumerate}
\item \textbf{Signal Transfer Boundary as part of the Decision Boundary:} If $\eta^Q,\eta^P$ are continuous, and $$\partial\Omega^+(\gamma,C_\gamma)\subset \{x\in\Omega:\eta^Q(x)=\frac{1}{2}\}.$$ Then $\varepsilon_b\lesssim n_P^{-\frac{\beta(1+\alpha)}{2\gamma\beta+d}}$, and
$$\E_{(\mathcal{D}_Q,\mathcal{D}_P)} \mathcal{E}_Q(\hat{f}^{NN}_{TAB})\lesssim n_P^{-\frac{\beta(1+\alpha)}{2\gamma\beta+d}}\land \left(\log^{1+\alpha} (n_Q\lor n_P) n_Q^{-\frac{\beta(1+\alpha)}{2\beta+d}}\right)+\varepsilon(2\tau;\gamma,C_\gamma).$$

\item \textbf{Signal Transfer Boundary Margin:} If for any $r>0$, there exists some constant $\alpha_r\geq 0, r_m>0, C_r>0$ such that for any $0<r<r_m$,
$$Q_X(B(X,r)\cap\Omega\cap\Omega_P\not\subset\Omega^+(\gamma,C_\gamma))\leq C_r r^{\alpha_r}.$$ Specially, we set $\alpha_r=\infty$ when $\Omega^+(\gamma,C_\gamma)$ can be separated with its compliment set in the sense that
$$\inf\{\|x-y\|:x\in\Omega^+(\gamma,C_\gamma), y\in\Omega/\Omega^+(\gamma,C_\gamma)\}\geq r_m$$ for some $r_m>0$. Then if $\beta_P/\gamma+\alpha_r\geq \beta(1+\alpha)$, we have $\varepsilon_b\lesssim n_P^{-\frac{\beta(1+\alpha)}{2\gamma\beta+d}}$, and
$$\E_{(\mathcal{D}_Q,\mathcal{D}_P)} \mathcal{E}_Q(\hat{f}^{NN}_{TAB})\lesssim n_P^{-\frac{\beta(1+\alpha)}{2\gamma\beta+d}}\land \left(\log^{1+\alpha} (n_Q\lor n_P) n_Q^{-\frac{\beta(1+\alpha)}{2\beta+d}}\right)+\varepsilon(2\tau;\gamma,C_\gamma).$$

\end{enumerate}
\end{corollary}
\section{Proofs of General Convergence Theorems}
\label{sec:appB}
In this section, we will prove Theorem \ref{thm:general} and Theorem \ref{thm:moreGeneral}. Suppose $\hat{f}^P$ is a classifier based on $\mathcal{D}_P$.

We list some additional notations used in the rest of the appendix part. For any set $A\subset\Omega$, define $A^c=\Omega/A$ as the compliment. For any set $B$, define $|B|$ as its cardinality. Define
$$\Omega^+(\gamma,C_\gamma):=\{x\in\Omega:s(x)\geq C_\gamma|\eta^Q(x)-\frac{1}{2}|^\gamma\},$$
and
$$\Omega^-(\gamma,C_\gamma):=\{x\in\Omega:s(x)< C_\gamma|\eta^Q(x)-\frac{1}{2}|^\gamma\}=(\Omega^+(\gamma,C_\gamma))^c.$$
We abbreviate $(\mathcal{D}_Q,\mathcal{D}_P)$, the combination of the target and source data, as $\mathcal{D}$. Denote the expectation w.r.t. a new target observation, i.e., $\E_{(X,Y)\sim Q}$, by $\E$.

\subsection{Proofs of Theorem \ref{thm:moreGeneral}}
\begin{proof}
Define the following events: $$E_0:=\{|\eta^Q(X)-\frac{1}{2}|\geq 2\tau\},\quad E^*:=\{X\in\Omega^*\},\quad E:=\{|\hat{\eta}^Q(X)-\eta^Q(X)|\leq \tau\}.$$ For any $(Q,P)\in\Pi$,
denote the event on which $\eta^Q(X)$ is far from 1/2 by $2\tau$, and the event on which the concentration property holds. Note that $E_0$ only depends on $Q$, but $E^*$ and $E$ depends on both $Q$ and $\E_\mathcal{D}$. The expected excess risk could be reformulated by the law of total expectation considering $E^c, E_0\cap E$ and $E_0^c\cap E:$

\begin{align}
\mathcal{E}_Q(\hat{f}_{TAB})&=2Q(E^c)\E[|\eta_Q(X)-\frac{1}{2}|\mathbf{1}\{\hat{f}_{TAB}(X)\neq f_Q^*(X)\}|E^c]\label{eq:A1bound1}\\
&+2Q(E_0\cap E)\E[|\eta_Q(X)-\frac{1}{2}|\mathbf{1}\{\hat{f}_{TAB}(X)\neq f_Q^*(X)\}|E_0\cap E]\label{eq:A1bound2}\\
&+2Q(E_0^c\cap E)\E[|\eta_Q(X)-\frac{1}{2}|\mathbf{1}\{\hat{f}_{TAB}(X)\neq f_Q^*(X)\}|E_0^c\cap E].
\label{eq:A1bound3}
\end{align}
Next, we will bound the expectation of the three components \eqref{eq:A1bound1}, \eqref{eq:A1bound2}, \eqref{eq:A1bound3} with respect to $\E_\mathcal{D}$, respectively.

Firstly, a trivial upper bound is $$2\E[|\eta_Q(X)-\frac{1}{2}|\mathbf{1}\{\hat{f}_{TAB}(X)\neq f_Q^*(X)\}|E^c]\leq 1,$$ which deduces that
\begin{equation}
\begin{aligned}
 &\E_\mathcal{D}[2Q(E^c)\E[|\eta_Q(X)-\frac{1}{2}|\mathbf{1}\{\hat{f}_{TAB}(X)\neq f_Q^*(X)\}|E^c]]\\
 \leq& \E_\mathcal{D}[Q(E^c)]
 = Q(\E_\mathcal{D}[E^c)])\leq 2\delta_Q(n_Q,\tau).   
\end{aligned}
\label{eq:A1result1}
\end{equation}
To see the last inequality of \eqref{eq:A1result1}, the condition of Theorem \ref{thm:moreGeneral} implies that $$Q((E^*)^c)= Q(\Omega/\Omega^*)\leq\delta(n_Q,\tau)\Rightarrow \E_{\mathcal{D}}[Q((E^*)^c)]\leq \delta(n_Q,\tau).$$ In addition, the concentration property on $\Omega^*$ implies that $$\E_\mathcal{D}[E^c|E^*]\leq \sup_{X\in\Omega^*}\E_\mathcal{D}[E^c|X=x]\leq\delta_Q(n_Q,\tau),$$ so by the law of total expectation, we have
$$
\begin{aligned}
Q(\E_\mathcal{D}[Q(E^c)])&\leq Q(\E_\mathcal{D}[E^c|E^*]\E_\mathcal{D}[E^*]+\E_\mathcal{D}[E^c|(E^*)^c]\E_\mathcal{D}[(E^*)^c])\\
&= Q(\E_\mathcal{D}[E^c|E^*]\E_\mathcal{D}[E^*])+Q(\E_\mathcal{D}[E^c|(E^*)^c]\E_\mathcal{D}[(E^*)^c])\\
&\leq \delta_Q(n_Q,\tau)Q(\E_\mathcal{D}[E^*])+Q(\E_\mathcal{D}[(E^*)^c])\\
&\leq \delta_Q(n_Q,\tau)+\E_\mathcal{D}[Q((E^*)^c)]\\
&\leq \delta_Q(n_Q,\tau)+\delta(n_Q,\tau).
\end{aligned}
$$

Next, on the event $E_0\cap E$, we observe that
$$(\hat{\eta}^Q(X)-\frac{1}{2})(\eta^Q(X)-\frac{1}{2})\geq 0, |\hat{\eta}^Q(X)-\frac{1}{2}|\geq \tau.$$
This fact tells us $\mathbf{1}\{\hat{f}_{TAB}(X)\neq f_Q^*(X)\}=0$. Hence, we have that
\begin{equation}
    \E_\mathcal{D}[2Q(E_0\cap E)\E[|\eta_Q(X)-\frac{1}{2}|\mathbf{1}\{\hat{f}_{TAB}(X)\neq f_Q^*(X)\}|E_0\cap E]]=0.
    \label{eq:A1result2}
\end{equation}

Lastly, let's upper bound the expectation of \eqref{eq:A1bound3} with respect to $\E_\mathcal{D}$. Without loss of generality, suppose that $\eta^Q(X)\neq \frac{1}{2}$, or otherwise the upper bound can be set as just $0$. Then, we observe that on the event $E_0^c\cap E$,
$$(\hat{\eta}^Q-\frac{1}{2})(\eta^Q-\frac{1}{2})<0\Rightarrow|\hat{\eta}^Q(X)-\frac{1}{2}|<\tau.$$ This is because, if the plug-in rule $\mathbf{1}\{\hat{\eta}^Q(X)\geq \frac{1}{2}\}$ is different from the Bayes classifier on $X$, the concentration bound of $\tau$ guarantees that $$|\hat{\eta}^Q(X)-\frac{1}{2}|\leq|\hat{\eta}^Q(X)-\eta^Q(X)|<\tau.$$ Therefore, on the event $E_0^c\cap E$, the indicator of wrong classification could be bounded by
$$\mathbf{1}\{\hat{f}_{TAB}(X)\neq f_Q^*(X)\}\leq \mathbf{1}\{\hat{f}^P(X)\neq f_Q^*(X)\},$$ and it holds that

$$
\begin{aligned}
    &\E_\mathcal{D}[2Q(E_0^c\cap E)\E[|\eta_Q(X)-\frac{1}{2}|\mathbf{1}\{\hat{f}_{TAB}(X)\neq f_Q^*(X)\}|E_0^c\cap E]]\\
    = &\E_\mathcal{D}[2\E[|\eta_Q(X)-\frac{1}{2}|\mathbf{1}\{\hat{f}_{TAB}(X)\neq f_Q^*(X)\}\mathbf{1}\{X\in E_0^c\cap E\}]\\
    \leq &\E_\mathcal{D}[2\E[|\eta_Q(X)-\frac{1}{2}|\mathbf{1}\{\hat{f}^P(X)\neq f_Q^*(X)\}\mathbf{1}\{X\in E_0^c\}]\\
    = &\E_{\mathcal{D}_P}[2\E[|\eta_Q(X)-\frac{1}{2}|\mathbf{1}\{\hat{f}^P(X)\neq f_Q^*(X)\}\mathbf{1}\{X\in E_0^c\}],
\end{aligned}
$$
where the last equality is due to the fact that the random variable $$|\eta_Q(X)-\frac{1}{2}|\mathbf{1}\{\hat{f}^P(X)\neq f_Q^*(X)\}\mathbf{1}\{X\in E_0^c\}$$ does not depend on $\mathcal{D}_Q$.

Define $E^+:=E_0^c\cap \{X\in\Omega^+(\gamma,C_\gamma)\}$ and $E^-:=E_0^c\cap\{X\in\Omega^-(\gamma,C_\gamma)\}$. Applying the law of total expectation again on $E^+$ and $E^-$ as a partition of $E_0^c$, we have
$$
\begin{aligned}
  &\E_{\mathcal{D}_P}[2\E[|\eta_Q(X)-\frac{1}{2}|\mathbf{1}\{\hat{f}^P(X)\neq f_Q^*(X)\}\mathbf{1}\{X\in E_0^c\}]\\
  \leq & \E_{\mathcal{D}_P}[2\E[|\eta_Q(X)-\frac{1}{2}|\mathbf{1}\{\hat{f}^P(X)\neq f_Q^*(X)\}\mathbf{1}\{X\in E^+\}]\\
  + &\E_{\mathcal{D}_P}[2\E[|\eta_Q(X)-\frac{1}{2}|\mathbf{1}\{\hat{f}^P(X)\neq f_Q^*(X)\}\mathbf{1}\{X\in E^-\}]\\
  = &2\mathbb{E}_{\mathcal{D}_P} \int_{\Omega^+(\gamma,C_\gamma)}|\eta^Q(X)-\frac{1}{2}|\mathbf{1}\{\hat{f}^P(X)\neq f_Q^*(X),|\eta^Q(X)-\frac{1}{2}|\leq 2\tau\}dQ_X \\
  + & \E_{\mathcal{D}_P}[2\E[|\eta_Q(X)-\frac{1}{2}|\mathbf{1}\{X\in E^-\}]\\
  \leq &\left(2\mathbb{E}_{\mathcal{D}_P} \int_{\Omega^+(\gamma,C_\gamma)}|\eta^Q(X)-\frac{1}{2}|\mathbf{1}\{\hat{f}^P(X)\neq f_Q^*(X)\}dQ_X\right)\land\left( 2 Q_X(|\eta^Q(X)-\frac{1}{2}|\leq 2\tau)\right) \\
  + & \E_{\mathcal{D}_P}[2\E[|\eta_Q(X)-\frac{1}{2}|\mathbf{1}\{X\in E^-\}].
\end{aligned}
$$
The term $\left(2\mathbb{E}_{\mathcal{D}_P} \int_{\Omega^+(\gamma,C_\gamma)}|\eta^Q(X)-\frac{1}{2}|\mathbf{1}\{\hat{f}^P(X)\neq f_Q^*(X)\}dQ_X\right)\land \left(2 Q_X(|\eta^Q(X)-\frac{1}{2}|\leq 2\tau)\right)$ is upper bounded by the restricted signal transfer risk $2\sup_{(Q,P)\in\Pi}\mathbb{E}_{\mathcal{D}_P}\xi(\hat{f}^P;\gamma,C_\gamma)\land 2C_\alpha(2\tau)^{1+\alpha}$ by Definition \ref{def:STR} and Assumption \ref{assum:margin}. Moreover, since $E^-$ does not depend on $\mathcal{D}_P$, it holds that 
$$
\begin{aligned}
    &\E_{\mathcal{D}_P}[2\E[|\eta_Q(X)-\frac{1}{2}|\mathbf{1}\{X\in E^-\}]\\=&2\int_{\Omega^-(\gamma,C_\gamma)}|\eta^Q(X)-\frac{1}{2}|\mathbf{1}\{0<|\eta^Q(X)-\frac{1}{2}|<2\tau\}dQ_X
    \leq 2\varepsilon(2\tau;\gamma,C_\gamma),
\end{aligned}
$$
by the definition of the ambiguity level. Hence, the following bound holds:
\begin{equation}
\begin{aligned}
    &\E_\mathcal{D}[2Q(E_0^c\cap E)\E[|\eta_Q(X)-\frac{1}{2}|\mathbf{1}\{\hat{f}_{TAB}(X)\neq f_Q^*(X)\}|E_0^c\cap E]]\\
    \leq &2\sup_{(Q,P)\in\Pi}\mathbb{E}_{\mathcal{D}_P}\xi(\hat{f}^P;\gamma,C_\gamma)\land 2C_\alpha(2\tau)^{1+\alpha}+2\varepsilon(2\tau;\gamma,C_\gamma).
\end{aligned}
\label{eq:A1result3}
\end{equation}
It is easy to see that Theorem \ref{thm:moreGeneral} holds by bounding the three components \eqref{eq:A1bound1}, \eqref{eq:A1bound2}, \eqref{eq:A1bound3} by \eqref{eq:A1result1}, \eqref{eq:A1result2}, \eqref{eq:A1result3}, respectively.
\end{proof}

\subsection{Proofs of Theorem \ref{thm:general}}
Suppose that $(Q,P)\in\Pi$. By plugging in $t=\tau$ in \eqref{eq:localConcentrationQ}, we have that with probability at least $1-\delta_f$ w.r.t. the distribution of $X_{1:n_Q}$, we have for any $x\in\Omega$,
\begin{equation}
    \pr_{\mathcal{D}_Q}(|\hat{\eta}^Q(x)-\eta^Q(x)| \geq \tau|X_{1:n_Q})\lesssim\exp(- (\frac{\tau}{\delta_Q})^2).
    \label{eq:A2result1}
\end{equation}

In other words, by taking the probability of which the concentration \eqref{eq:A2result1} fails into account, for any $x\in\Omega$, the equation
$$|\hat{\eta}^Q(x)-\eta^Q(x)| \leq \tau$$ holds with probability that is asymptotically less than $\delta_f+\exp(- (\frac{\tau}{\delta_Q})^2)$ with respect to $\pr_{\mathcal{D}_Q}$. We fix $\Omega^*=\Omega$ in the terminology of Theorem \ref{thm:moreGeneral}, and thus set $\delta(\cdot,\cdot)\equiv 0.$

\begin{proof}[Proof of Theorem \ref{thm:general}]
Suppose that for some $c_0>0$, the inequality $\tau\geq c_0 \log(n_Q\lor n_P)\delta_Q$ always holds. We have that
$$
\begin{aligned}
 \exp(- (\frac{\tau}{\delta_Q})^2)&\leq \exp(-c_0^2\log^2(n_Q\lor n_P))\\
 &= n_Q^{-c_0^2\log n_Q}\land n_P^{-c_0^2\log n_P}\\
 &\lesssim n_Q^{-c}\land n_P^{-c}\\
 &\lesssim \sup_{(Q,P)\in\Pi}\E_{\mathcal{D}_P}\xi(\hat{f}^P,2\delta_Q;\gamma,C_\gamma).  
\end{aligned}
$$
Suppose that for some $c_0>0$, the inequality $\tau\geq c_0 \log(n_Q\lor n_P)\delta_Q$ always holds. We have that
$$
\begin{aligned}
 \exp(- (\frac{\tau}{\delta_Q})^2)&\leq \exp(-c_0^2\log^2(n_Q\lor n_P))\\
 &= n_Q^{-c_0^2\log n_Q}\land n_P^{-c_0^2\log n_P}\\
 &\lesssim n_Q^{-c}\land n_P^{-c}\\
 &\lesssim \sup_{(Q,P)\in\Pi}\E_{\mathcal{D}_P}\xi(\hat{f}^P;\gamma,C_\gamma)\land \delta_Q^{1+\alpha}.
\end{aligned}
$$Therefore, the conditions of Theorem \ref{thm:moreGeneral} holds by setting $$\delta_Q(n_Q,\tau)=C(\delta_f+\sup_{(Q,P)\in\Pi}\E_{\mathcal{D}_P}\xi(\hat{f}^P;\gamma,C_\gamma)+\delta_Q^{1+\alpha})$$ for some $C>0$ that is large enough. Theorem \ref{thm:general} thus holds.
\end{proof}

\subsection{Proof of Theorem \ref{thm:STRgeneral}}
\begin{proof}
Based on the definition of \(\Omega^+(\gamma, C_\gamma)\), we see that $$\Omega^+(\gamma, C_\gamma)/\{x\in\Omega:\eta^Q(x)=1/2\} \subset \Omega \cap \Omega_P,$$
as the source strong signal condition along with $\eta^Q(x)\neq 1/2$ implies $x\in\mathrm{supp}(P_X)=\Omega_P$.
Suppose that $(Q,P)\in\Pi\cap\mathcal{A}(M)$. For any $t\geq 0$, define $$\Omega(t):=\{x\in\Omega:0<|\eta^Q(x)-\frac{1}{2}|\leq t\}.$$ For any $z\in(0,\frac{1}{2})$, define
$$
\begin{aligned}
    &A(z):=\mathbb{E}_{\mathcal{D}_P} \int_{\Omega^+(\gamma,C_\gamma)\cap \Omega(z)}|\eta^Q(X)-\frac{1}{2}|\mathbf{1}\{\hat{f}^P(X)\neq f_Q^*(X)\}dQ_X,\\
    &A:=A(\frac{1}{2}),\ B:=\mathbb{E}_{\mathcal{D}_P} \E_{(X,Y)\sim P}[|\eta^P(X)-\frac{1}{2}|\mathbf{1}\{\hat{f}^P(X)\neq f_P^*(X)\}].
\end{aligned}
$$
\noindent\textbf{I. Case of $\gamma\geq 1:$} By definition, $\varepsilon_P$ is an upper bound of $B$, so it suffices to prove $$A\leq 2M^{\frac{1+\alpha}{\gamma+\alpha}}C_\alpha^{\frac{\gamma-1}{\gamma+\alpha}}C_\gamma^{-\frac{1+\alpha}{\gamma+\alpha}}B^{\frac{1+\alpha}{\gamma+\alpha}}.$$
By decomposing $\Omega^+(\gamma,C_\gamma)$ into $\Omega^+(\gamma,C_\gamma)\cap \Omega(t)$ and $\Omega^+(\gamma,C_\gamma)\cap \Omega(t)^c$, we have
$$
\begin{aligned}
A/M\leq &\frac{1}{M}\mathbb{E}_{\mathcal{D}_P} \int_{\Omega^+(\gamma,C_\gamma)\cap \Omega(t)}|\eta^Q(X)-\frac{1}{2}|\mathbf{1}\{\hat{f}^P(X)\neq f_Q^*(X)\}dQ_X\\
+ &\mathbb{E}_{\mathcal{D}_P} \int_{\Omega^+(\gamma,C_\gamma)\cap \Omega(t)^c}|\eta^Q(X)-\frac{1}{2}|\mathbf{1}\{\hat{f}^P(X)\neq f_Q^*(X)\}dP_X&
\end{aligned}
$$
since $dQ_X\leq M dP_X$ over $\Omega^+(\gamma,C_\gamma)\cap \Omega(t)$. First, by upper bounding $\mathbf{1}\{\hat{f}^P(X)\neq f_Q^*(X)\}$ by $1$, we have
$$
\begin{aligned}
    &\frac{1}{M}\mathbb{E}_{\mathcal{D}_P} \int_{\Omega^+(\gamma,C_\gamma)\cap \Omega(t)}|\eta^Q(X)-\frac{1}{2}|\mathbf{1}\{\hat{f}^P(X)\neq f_Q^*(X)\}dQ_X\\
    \leq &\frac{1}{M} \int_{\Omega^+(\gamma,C_\gamma)\cap \Omega(t)}|\eta^Q(X)-\frac{1}{2}|dQ_X\\
    \leq &\frac{C_\alpha}{M}tQ_X(\Omega^+(\gamma,C_\gamma)\cap \Omega(t))\leq \frac{C_\alpha}{M}t^{1+\alpha}
\end{aligned}
$$ by the margin assumption. Secondly, when $x\in\Omega^+(\gamma,C_\gamma)\cap \Omega(t)^c$, we have
$$
\begin{aligned}
   &t^{\gamma-1}|\eta^Q(x)-\frac{1}{2}|\leq |\eta^Q(x)-\frac{1}{2}|^\gamma\leq |\eta^P(x)-\frac{1}{2}|/C_\gamma\\
\Rightarrow &|\eta^Q(x)-\frac{1}{2}|\leq t^{1-\gamma}|\eta^P(x)-\frac{1}{2}|/C_\gamma.
\end{aligned}
$$
Applying this inequality to the risk term with respect to $\Omega^+(\gamma,C_\gamma)\cap\Omega(t)^c$, we have 
$$
\begin{aligned}
   & \mathbb{E}_{\mathcal{D}_P} \int_{\Omega^+(\gamma,C_\gamma)\cap \Omega(t)}|\eta^Q(X)-\frac{1}{2}|\mathbf{1}\{\hat{f}^P(X)\neq f_Q^*(X)\}dP_X \\
   \leq & \frac{t^{1-\gamma}}{C_\gamma}\mathbb{E}_{\mathcal{D}_P} \int_{\Omega^+(\gamma,C_\gamma)\cap \Omega(t)}|\eta^P(X)-\frac{1}{2}|\mathbf{1}\{\hat{f}^P(X)\neq f_Q^*(X)\}dP_X\\
   \leq & \frac{t^{1-\gamma}}{C_\gamma}B.
\end{aligned}
$$
By plugging in the two inequalities above, we have that for any $t\geq 0$,
$$A/M\leq \frac{C_\alpha}{M}t^{1+\alpha}+\frac{t^{1-\gamma}}{C_\gamma}B.$$ Since the choice of $t$ is arbitrary, we choose $t=(\frac{MB}{C_\alpha C_\gamma})^{\gamma+\alpha}$, the bound becomes
$$A\leq 2M^{\frac{1+\alpha}{\gamma+\alpha}}C_\alpha^{\frac{\gamma-1}{\gamma+\alpha}}C_\gamma^{-\frac{1+\alpha}{\gamma+\alpha}}B^{\frac{1+\alpha}{\gamma+\alpha}}.$$

\noindent\textbf{II. Case of $\gamma< 1:$} By definition, $\varepsilon_P$ is an upper bound of $B$, so it suffices to prove
$$A(z)\leq z^{1-\gamma}MC_\gamma^{-1}B$$ and then let $z=\frac{1}{2}$.
Note that when $x\in\Omega^+(\gamma,C_\gamma)\cap\Omega(z)$, we have
$$
\begin{aligned}
 & z^{\gamma-1}|\eta^Q(x)-\frac{1}{2}|\leq |\eta^Q(x)-\frac{1}{2}|^\gamma\leq |\eta^P(x)-\frac{1}{2}|/C_\gamma\\
 \Rightarrow & |\eta^Q(x)-\frac{1}{2}|\leq z^{1-\gamma}|\eta^P(x)-\frac{1}{2}|/C_\gamma.
\end{aligned}
$$
Therefore,
$$
\begin{aligned}
  A(z)/M&\leq \frac{1}{M}\mathbb{E}_{\mathcal{D}_P} \int_{\Omega^+(\gamma,C_\gamma)\cap \Omega(z)}|\eta^Q(X)-\frac{1}{2}|\mathbf{1}\{\hat{f}^P(X)\neq f_Q^*(X)\}dQ_X\\
&\leq C_\gamma^{-1}z^{1-\gamma}\mathbb{E}_{\mathcal{D}_P} \int_{\Omega^+(\gamma,C_\gamma)\cap \Omega(z)}|\eta^P(x)-\frac{1}{2}|\mathbf{1}\{\hat{f}^P(X)\neq f_Q^*(X)\}dP_X\\
&= C_\gamma^{-1}z^{1-\gamma}B,
\end{aligned}
$$
which finishes the proof.
\end{proof}

\subsection{Proofs of Lemma \ref{lemma:erm} and Theorem \ref{thm:erm}}
\begin{proof}[Proof of Lemma \ref{lemma:erm}]
Suppose that $(Q,P)\in\Pi$. By Assumption \ref{assum:margin}, for any $t>0$ and $j=1,\dots,M$, we have
$$
\begin{aligned}
\mathcal{E}_Q(f)&=\int_{x\in\Omega:\hat{f}_{TAB}(x;\tau_j)\neq f^*_Q(x)}|2\eta^Q(x)-1|dQ_X\\
&\geq 2t Q_X(\hat{f}_{TAB}(X;\tau_j)\neq f^*_Q(X),|\eta^Q(x)-1/2|\geq t)\\
&\geq 2t \Big(Q_X(\hat{f}_{TAB}(X;\tau_j)\neq f^*_Q(X))-Q_X(|\eta^Q(x)-1/2|< t)\Big)\\
&\geq 2t \Big(Q_X(\hat{f}_{TAB}(X;\tau_j)\neq f^*_Q(X))-C_\alpha t^\alpha\Big).
\end{aligned}
$$
Particularly, by choosing $\epsilon=\Big(Q_X(\hat{f}_{TAB}(X;\tau_j)\neq f^*_Q(X))/(2C_\alpha)\Big)^{1/\alpha}$, we have
\begin{equation}
\label{eq:lemma1eq1}
Q_X(\hat{f}_{TAB}(X;\tau_j)\neq f^*_Q(X))\leq (2C_\alpha)^{\frac{1}{1+\alpha}}\mathcal{E}_Q(\hat{f}_{TAB}(X;\tau_j))^{\frac{\alpha}{1+\alpha}}.
\end{equation}
Define the empirical process as
$$Z_i^j:=\mathbf{1}\{\hat{f}_{TAB}(X_i;\tau_j)\neq Y_i\}-\mathbf{1}\{f_Q^*(X_i;\tau_j)\neq Y_i\},\quad \forall i=1,\dots,n_Q,\ j=1,\dots,M.$$
It is easy to observe that $|Z_i^j|\leq 1$, $\E_Q[Z_i^j]=\mathcal{E}_Q(\hat{f}_{TAB}(X;\tau_j))$, and $\E[(Z_i^j)^2]\leq (2C_\alpha)^{\frac{1}{1+\alpha}}\mathcal{E}_Q(\hat{f}_{TAB}(X;\tau_j))^{\frac{\alpha}{1+\alpha}}$. Hence, we could derive the following inequality by Bernstein's inequality combined with a union bound:
\begin{equation}
\label{eq:lemma1eq2}
\begin{aligned}
\pr_{\mathcal{D}_Q}\Big(&\sup_{1\leq j\leq M}|\frac{1}{n_Q}\sum_{i=1}^{n_Q}Z_i^j-\mathcal{E}_Q(\hat{f}_{TAB}(X;\tau_j))|\geq\\
& 2\sqrt{\frac{(2C_\alpha)^{\frac{1}{1+\alpha}}\mathcal{E}_Q(\hat{f}_{TAB}(X;\tau_j))^{\frac{\alpha}{1+\alpha}}\log(2M/\delta)}{n}}+\frac{2\log(2M/\delta)}{3n}\Big)\leq \delta.  
\end{aligned} 
\end{equation}
By letting the event where the inequality in \eqref{eq:lemma1eq2} holds as $E$, we have $Q(E^c)\geq 1-\delta$. By the definition of $\mathcal{E}_Q^*$, there exists some $j_0$ such that $\mathcal{E}_Q(\hat{f}_{TAB}(\cdot;\tau_{j_0}))=\mathcal{E}_Q^*$. On the event $E^c$, we have
\begin{equation}
\label{eq:lemma1eq3}
\begin{aligned}
\mathcal{E}_Q(\hat{f}_{TAB}(X;\tau_{j^*}))=&\mathcal{E}_Q^*+\frac{1}{n_Q}\sum_{i=1}^{n_Q} Z_i^{j^*}-\frac{1}{n_Q}\sum_{i=1}^{n_Q} Z_i^{j_0}\\
+&\Big|\frac{1}{n_Q}\sum_{i=1}^{n_Q}Z_i^{j^*}-\mathcal{E}_Q(\hat{f}_{TAB}(X;\tau_{j^*}))\Big|+\Big|\frac{1}{n_Q}\sum_{i=1}^{n_Q}Z_i^{j^0}-\mathcal{E}_Q(\hat{f}_{TAB}(X;\tau_{j^0}))\Big|\\
\leq& 4\sqrt{\frac{(2C_\alpha)^{\frac{1}{1+\alpha}}\mathcal{E}_Q(\hat{f}_{TAB}(X;\tau_j))^{\frac{\alpha}{1+\alpha}}\log(2M/\delta)}{n}}+\frac{4\log(2M/\delta)}{3n}.
\end{aligned}   
\end{equation}
Solving \eqref{eq:lemma1eq3} with respect to $\mathcal{E}_Q(\hat{f}_{TAB}(X;\tau_{j^*}))$ indicates that
$$\mathcal{E}_Q(\hat{f}^{ERM}_{TAB})=\mathcal{E}_Q(\hat{f}_{TAB}(X;\tau_{j^*}))\geq 2\mathcal{E}_Q^*+64C_\alpha^{\frac{1}{2+\alpha}}\Big(\frac{\log(2M/\delta)}{n}\Big)^{\frac{1+\alpha}{2+\alpha}}$$
on the event $E^c$, which completes the proof of the first result. The second result follows directly as a corollary of the first result.
\end{proof}

\begin{proof}[Proof of Theorem \ref{thm:erm}]
Since $n_Q^{-c}\ll \delta_Q$, when $n_Q$ is sufficiently large, there must be some $j_n\in\{1,2,\dots,M\}$ such that
$$\frac{1}{2}\log(n_Q\lor n_P)\delta_Q\leq \tau_{j_n}\leq \log(n_Q\lor n_P)\delta_Q.$$
Let $\mathcal{E}_Q^{TAB}$ be $\mathcal{E}_Q(\hat{f}_{TAB}(X;\tau_{j_n}))$. It is obvious that $\mathcal{E}_Q^{*}\leq\mathcal{E}_Q^{TAB}$.
Then, the result of asymptotic upper bound of \(\mathcal{E}^{TAB}_Q\) is simply the result of Theorem \ref{thm:general}. By Lemma \ref{lemma:erm}, we have
$$
\begin{aligned}
&\pr_{\mathcal{D}_Q}\Big\{\mathcal{E}_Q(\hat{f}^{ERM}_{TAB})\geq 2\mathcal{E}_Q^{TAB}+64C_\alpha^{\frac{1}{2+\alpha}}\Big(\frac{\log(2M/\delta)}{n}\Big)^{\frac{1+\alpha}{2+\alpha}}\Big\}\\
\leq&
\pr_{\mathcal{D}_Q}\Big\{\mathcal{E}_Q(\hat{f}^{ERM}_{TAB})\geq 2\mathcal{E}_Q^*+64C_\alpha^{\frac{1}{2+\alpha}}\Big(\frac{\log(2M/\delta)}{n}\Big)^{\frac{1+\alpha}{2+\alpha}}\Big\}\leq \delta,
\end{aligned}
$$
which completes the proof of the first result. The second result follows directly as a corollary of the first result. 
\end{proof}
\section{Proofs in Non-parametric Classification}
\label{sec:appC}
In this section, we only consider the case of $(Q,P)\in\Pi^{NP}$. Since $\Pi^{NP}$ assumes the compactness of the support sets $\Omega$ and $\Omega_P$, we assume for simplicity that $\Omega,\Omega_P$ is a subset of the $d$-dimension unit square $[0,1]^d$.

Our proof of Lemma \ref{lemma:nonParaGeneralP} partially relies on verifying the conditions in Theorem \ref{thm:STRPlugIn}, presented in Appendix \ref{sec:pluginSTR}. Therefore, we will provide the proof of Theorem \ref{thm:STRPlugIn} first in this section.

\subsection{Proof of Theorem \ref{thm:STRPlugIn}}
\begin{proof}
Denote the distribution of $X_{1:n_P}^P$ by $\pr_{X_{1:n_P}^P}$, and $\Omega^*(\gamma,C_\gamma):=\Omega^+(\gamma,C_\gamma)/\Omega^b$.
Suppose $(Q,P)\in\Pi$.
It is easy to see that there exists an event $E^f_P$ related to the distribution of $X_{1:n_P}^P$ such that
\begin{itemize}
    \item $\pr_{X_{1:n_P}^P}(E^f_P)\leq \delta_{P,f}$.
    \item On the event $(E^f_P)^c$, we have for any $x\in\Omega^*(\gamma,C_\gamma)$,
    $$\pr_{\mathcal{D}_P}((\hat{\eta}^P(x)-\frac{1}{2})(\eta^Q(x)-\frac{1}{2})<0|X^P_{1:n_P})\leq C_2\exp(- (\frac{C_\gamma|\eta^Q(x)-\frac{1}{2}|^{\gamma}}{\delta_P})^2).$$
\end{itemize}
Define $E^*=\{X\in \Omega^*(\gamma,C_\gamma)\}$. We apply the law of total expectation on $$\E_{\mathcal{D}_P}[\E[|\eta^Q(X)-\frac{1}{2}|\mathbf{1}\{\hat{f}^P(X)\neq f_Q^*(X)\}\mathbf{1}\{X\in\Omega^+(\gamma,C_\gamma)\}]]$$ by decomposing $\{X\in\Omega^+(\gamma,C_\gamma)\}$ into $E^f_P,(E^f_P)^c\cap E^*$ and $(E^f_P)^c\cap (E^*)^c$. The decomposition reads
\begin{align}
&\E_{\mathcal{D}_P}[\E[|\eta^Q(X)-\frac{1}{2}|\mathbf{1}\{\hat{f}^P(X)\neq f_Q^*(X)\}\mathbf{1}\{X\in\Omega^+(\gamma,C_\gamma)\}]]\\
=&\pr_{X_{1:n_P}^P}(E^f_P)\E_{\mathcal{D}_P}[\E[|\eta^Q(X)-\frac{1}{2}|\mathbf{1}\{\hat{f}^P(X)\neq f_Q^*(X)\}]|E^f_P]
\label{eq:B1bound1}\\
+& \pr_{X_{1:n_P}^P}((E^f_P)^c)\E_{\mathcal{D}_P}[\E[|\eta^Q(X)-\frac{1}{2}|\mathbf{1}\{\hat{f}^P(X)\neq f_Q^*(X)\}\mathbf{1}\{X\in\Omega^*(\gamma,C_\gamma)\}]|(E^f_P)^c]\label{eq:B1bound2}\\
+&\pr_{X_{1:n_P}^P}((E^f_P)^c)\E_{\mathcal{D}_P}[\E[|\eta^Q(X)-\frac{1}{2}|\mathbf{1}\{\hat{f}^P(X)\neq f_Q^*(X)\}\mathbf{1}\{X\in\Omega^b\}]|(E^f_P)^c].\label{eq:B1bound3}
\end{align}
We will then bound the three components \eqref{eq:B1bound1}, \eqref{eq:B1bound2}, \eqref{eq:B1bound3} respectively. Firstly, we observe that \eqref{eq:B1bound1} could be bounded by
\begin{equation}
    \pr_{X_{1:n_P}^P}(E^f_P)\E_{\mathcal{D}_P}[\E[\eta^Q(X)-\frac{1}{2}|\mathbf{1}\{\hat{f}^P(X)\neq f_Q^*(X)\}]|E^f_P]\leq \frac{1}{2}\pr_{X_{1:n_P}^P}(E^f_P)\leq \frac{1}{2}\delta_{P,f}.
    \label{eq:B1result1}
\end{equation}
Secondly, \eqref{eq:B1bound2} could be bounded by
\begin{equation}
\begin{aligned}
    &\pr_{X_{1:n_P}^P}((E^f_P)^c)\E_{\mathcal{D}_P}[\E[|\eta^Q(X)-\frac{1}{2}|\mathbf{1}\{\hat{f}^P(X)\neq f_Q^*(X)\}\mathbf{1}\{X\in\Omega^b\}]|(E^f_P)^c]\\
    \leq & \E_{\mathcal{D}_P}[\E[|\eta^Q(X)-\frac{1}{2}|\mathbf{1}\{X\in\Omega^b\}\mathbf{1}\{(E^f_P)^c \text{ holds}\}]]\\
    \leq & \E_{\mathcal{D}_P}[\E[|\eta^Q(X)-\frac{1}{2}|\mathbf{1}\{X\in\Omega^b\}]]\\
    = & \E[|\eta^Q(X)-\frac{1}{2}|\mathbf{1}\{X\in\Omega^b\}]\leq \delta_P^b,
\end{aligned}
\label{eq:B1result3}
\end{equation}
where the last equality holds as $\E[\eta^Q(X)-\frac{1}{2}|\mathbf{1}\{X\in\Omega^b\}]$ does not depend on $\mathcal{D}_P$. By plugging in \eqref{eq:B1result1} and \eqref{eq:B1result3} back into \eqref{eq:B1bound1} and \eqref{eq:B1bound3}, we see that to finish the proof, it suffices show that
\begin{equation}
    \E_{\mathcal{D}_P}[\E[|\eta^Q(X)-\frac{1}{2}|\mathbf{1}\{\hat{f}^P(X)\neq f_Q^*(X)\}\mathbf{1}\{X\in\Omega^*(\gamma,C_\gamma)\}]|(E^f_P)^c]\lesssim \delta_P^{\frac{1+\alpha}{\gamma}}.
    \label{eq:B1bound4}
\end{equation}
Consider a partition of $\Omega^*(\gamma,C_\gamma)$, which are $A_j\subset \Omega^*(\gamma,C_\gamma), j=0,1,2,\cdots$ defined as
$$A_0:=\{x\in\Omega^*(\gamma,C_\gamma):0<|\eta^Q(x)-\frac{1}{2}|\leq \delta_P^{1/\gamma}\},$$
$$A_j:=\{x\in\Omega^*(\gamma,C_\gamma):2^{j-1}\delta_P^{1/\gamma}<|\eta^Q(x)-\frac{1}{2}|\leq 2^j\delta_P^{1/\gamma}\},\ j\geq 1.$$
We have that
$$
\begin{aligned}
&\E_{\mathcal{D}_P}[\E[|\eta^Q(X)-\frac{1}{2}|\mathbf{1}\{\hat{f}^P(X)\neq f_Q^*(X)\}\mathbf{1}\{X\in\Omega^*(\gamma,C_\gamma)\}]|(E^f_P)^c]\\
=&\sum_{j\geq 0}\E_{\mathcal{D}_P}[\E[|\eta^Q(X)-\frac{1}{2}|\mathbf{1}\{\hat{f}^P(X)\neq f_Q^*(X)\}\mathbf{1}\{X\in A_j\}]|(E^f_P)^c].
\end{aligned}
$$
For $j=0$, by the margin assumption, we have
\begin{equation}
\begin{aligned}
    &\E_{\mathcal{D}_P}[\E[|\eta^Q(X)-\frac{1}{2}|\mathbf{1}\{\hat{f}^P(X)\neq f_Q^*(X)\}\mathbf{1}\{X\in A_0\}]|(E^f_P)^c]\\
    \leq &\E[|\eta^Q(X)-\frac{1}{2}|\mathbf{1}\{X\in A_0\}]\leq \delta_P^{1/\gamma} Q(X\in A_0)\leq C_\alpha \delta_P^{\frac{1+\alpha}{\gamma}}.
\end{aligned}
\label{eq:B1boundA0}
\end{equation}
For $j\geq 1$, under the event $X\in A_j$, it holds by \eqref{eq:misclassP} that
$$2^{j-1}\delta_P^{1/\gamma}<|\eta^Q(X)-\frac{1}{2}|\leq 2^j\delta_P^{1/\gamma},$$ $$\pr_{\mathcal{D}_P}(\hat{f}^P(X)\neq f_Q^*(X)|(E^f_P)^c)=\pr_{\mathcal{D}_P}((\hat{\eta}^P(x)-\frac{1}{2})(\eta^Q-\frac{1}{2})\geq 0|(E^f_P)^c)\leq C\exp(-C_\gamma^2 2^{2\gamma(j-1)})$$ for some constant $C>0$.
Therefore,
\begin{equation}
\begin{aligned}
  &\E_{\mathcal{D}_P}[\E[|\eta^Q(X)-\frac{1}{2}|\mathbf{1}\{\hat{f}^P(X)\neq f_Q^*(X)\}\mathbf{1}\{X\in A_j\}]|(E^f_P)^c]\\
\leq & \E_{\mathcal{D}_P}[\E[2^j\delta_P^{1/\gamma}\mathbf{1}\{\hat{f}^P(X)\neq f_Q^*(X)\}\mathbf{1}\{X\in A_j\}]|(E^f_P)^c] \\
= & \E[2^j\delta_P^{1/\gamma}\E_{\mathcal{D}_P}[\mathbf{1}\{\hat{f}^P(X)\neq f_Q^*(X)\}|(E^f_P)^c]\mathbf{1}\{X\in A_j\}]\\
\leq & C\E[2^j\delta_P^{1/\gamma}\exp(-C_\gamma^2 2^{2\gamma(j-1)})\mathbf{1}\{X\in A_j\}]\\
\leq & C2^j\delta_P^{1/\gamma}\exp(-C_\gamma^2 2^{2\gamma(j-1)})Q(X\in A_j)\\
\leq & C2^j\delta_P^{1/\gamma}\exp(-C_\gamma^2 2^{2\gamma(j-1)})2^{\alpha j}\delta^{\frac{\alpha}{\gamma}}= C2^{(1+\alpha)j}\exp(-C_\gamma^2 2^{2\gamma(j-1)})\delta_P^{\frac{1+\alpha}{\gamma}},
\end{aligned}
\label{eq:B1boundAj}
\end{equation}
where the last inequality is derived by the margin assumption.

Summing all terms above in \eqref{eq:B1boundA0} and \eqref{eq:B1boundAj}, we have that 
$$
\begin{aligned}
&\E_{\mathcal{D}_P}[\E[|\eta^Q(X)-\frac{1}{2}|\mathbf{1}\{\hat{f}^P(X)\neq f_Q^*(X)\}\mathbf{1}\{X\in\Omega^*(\gamma,C_\gamma)\}]|(E^f_P)^c]\\
\leq & \delta_P^{\frac{1+\alpha}{\gamma}}(C_\alpha+\sum_{j\geq 1}C2^{(1+\alpha)j}\exp(-C_\gamma^2 2^{2\gamma(j-1)}))\lesssim \delta_P^{\frac{1+\alpha}{\gamma}},
\end{aligned}
$$
since $\sum_{j\geq 1}C2^{(1+\alpha)j}\exp(-C_\gamma^2 2^{2\gamma(j-1)})<\infty$. Therefore, \eqref{eq:B1bound4} holds and we finish the proof.
\end{proof}
\subsection{Proof of Theorem \ref{thm:nonParaGeneral}}
Next, we will prove the excess risk upper bound of the TAB $K$-NN classifier. Let $\beta_P^*=\max\{\gamma\beta,\beta_P\}$. Define $\Omega^*(\gamma,C_\gamma):=\Omega^+(\gamma,C_\gamma)/\Omega^b$, and $\delta_P:=n_P^{-\frac{\beta_P^*}{2\beta_P^*+d}}$. To simplify the analysis, we assume that $n_Q$ is large enough so that
\begin{equation}
\frac{1}{2}c_Qn_Q^{\frac{2\beta}{2\beta+d}}\leq k_Q\leq c_Qn_Q^{\frac{2\beta}{2\beta+d}},\quad \frac{1}{2}c_Pn_P^{\frac{2\beta_P^*}{2\beta_P^*+d}}\leq k_P\leq c_Pn_P^{\frac{2\beta_P^*}{2\beta_P^*+d}},
\label{eq:knnAsymp}
\end{equation}
 where $k_Q=\lfloor c_Qn_Q^{\frac{2\beta}{2\beta+d}}\rfloor$ and $k_P=\lfloor c_Pn_P^{\frac{2\beta_P^*}{2\beta_P^*+d}}\rfloor$ are the chosen number of nearest neighbors in $\mathcal{D}_Q$ and $\mathcal{D}_P$, respectively, and $c_Q$ and $c_P$ are constants. This assumption is valid since our results are shown in the asymptotic regime, and we assume that $n_P\overset{n_Q\rightarrow\infty}{\longrightarrow}\infty$.

 Given the condition of \eqref{eq:knnAsymp}, the $K$-NN distance bound (Lemma \ref{lemma:knnBound}) shows that there exists a constant $c_D>0$ such that with probability at least $1-\frac{c_D}{c_Q}n_Q^{\frac{d}{2\beta+d}}\exp(-c_Qn_Q^{\frac{2\beta}{2\beta+d}})$ w.r.t. the distribution of $X_{1:n_Q}$, we have
\begin{equation}
    \|X_{(k_Q)}(x)-x\|\leq c_D(\frac{k_Q}{n_Q})^{\frac{1}{d}}\leq c_Dc_Q^{\frac{1}{d}}n_Q^{-\frac{1}{2\beta+d}},\quad \forall x\in\Omega.
    \label{eq:prop3knnQ}
\end{equation}
 In addition, with probability at least $1-\frac{c_D}{c_P}n_P^{\frac{d}{2\beta_P^*+d}}\exp(-c_Pn_P^{\frac{2\beta_P^*}{2\beta_P^*+d}})$ w.r.t. the distribution of $X^P_{1:n_P}$, we have
 \begin{equation}
   \|X^P_{(k_P)}(x)-x\|\leq c_D(\frac{k_P}{n_P})^{\frac{1}{d}}\leq c_Dc_P^{\frac{1}{d}}n_P^{-\frac{1}{2\beta_P^*+d}},\quad \forall x\in\Omega_P.
   \label{eq:prop3knnP}
 \end{equation}
Denote $E_Q$ and $E_P$ as the event that \eqref{eq:prop3knnQ} and \eqref{eq:prop3knnP} hold, respectively. Note that $E_Q$ and $E_P$ depend on the distribution of $X_{1:n_Q}$ and $X^P_{1:n_P}$ respectively. From the probability bound, we have
$$\pr_{\mathcal{D}_Q}(E_Q^c)\leq \frac{c_D}{c_Q}n_Q^{\frac{d}{2\beta+d}}\exp(-c_Qn_Q^{\frac{2\beta}{2\beta+d}}),\quad \pr_{\mathcal{D}_P}(E_P^c)\leq \frac{c_D}{c_P}n_P^{\frac{d}{2\beta^*_P+d}}\exp(-c_Qn_P^{\frac{2\beta^*_P}{2\beta^*_P+d}}).$$

The proof of Theorem \ref{thm:nonParaGeneral} is obtained simply by verifying the conditions in Theorem \ref{thm:general}. The next two lemmas aim to prove such conditions:

\begin{lemma}
\label{lemma:nonParaGeneralQ}
By choosing $k_Q=\lfloor c_Qn_Q^{\frac{2\beta}{2\beta+d}}\rfloor$ for any constant $c_Q>0$, the regression function estimates $\hat{\eta}^{Q}_{k_Q}(x)$ has the property that there exists constants $c_1,c_2,c_3>0$ such that with probability at least $1-c_1 n_Q^{\frac{d}{2\beta+d}}\exp (-c_Qn_Q^{\frac{2\beta}{2\beta+d}})$ w.r.t. the distribution of $X_{1:n_Q}:=(X_1,\cdots,X_{n_Q})$, for any $x\in\Omega$ we have $\forall t>0$,
    \begin{equation}
        \sup_{(Q,P)\in\Pi^{NP}}\pr_{\mathcal{D}_Q}(|\hat{\eta}^Q(x)-\eta^Q(x)| \geq t|X_{1:n_Q})\leq c_2\exp(- (\frac{t}{c_3n_Q^{-\frac{\beta}{2\beta+d}}})^2).
        \label{eq:knnConcentrationQ}
    \end{equation}
\end{lemma}

\begin{lemma}
\label{lemma:nonParaGeneralP}
Let $\beta_P^*=\max\{\gamma\beta,\beta_P\}$. By choosing $k_P=\lfloor c_Pn_P^{\frac{2\beta_P^*}{2\beta_P^*+d}}\rfloor$ for any constant $c_P>0$, the regression function estimates $\hat{\eta}^{P}_{k_P}(x)$ has the property that there exists constants $c_b>0$ such that for any $\Pi^{NP,s}\subset \Pi^{NP}$, we have$$\xi(\hat{f}^P;\gamma,C_\gamma,\Pi^{NP,s})\lesssim n_P^{-\frac{\beta_P^*(1+\alpha)/\gamma}{2\beta_P^*+d}}+\varepsilon_b,$$$$\varepsilon_b= \sup_{(Q,P)\in\Pi^{NP,s}}\int_{\Omega^b}|\eta^Q(X)-\frac{1}{2}|dQ_X,$$ where the boundary of the signal transfer set is defined as
    $$\Omega^b:=\{x\in\Omega: |\eta^Q(x)-\frac{1}{2}|\leq c_b n_P^{-\frac{\beta_P/\gamma}{2\beta_P^*+d}},
 B(x,c_bn_P^{-\frac{1}{2\beta_P^*+d}})\cap \Omega_P\not\subset \Omega^+(\gamma,C_\gamma)\}
    $$
    for some constant $c_b>0$.
\end{lemma}

\begin{proof}[Proof of Lemma \ref{lemma:nonParaGeneralQ}]
Define the simple average of the regression functions of the $k_Q$ neighbors of $x$ as $$\bar{\eta}^Q_{k_Q}(x):=\frac{1}{k_Q}\sum_{i=1}^{k_Q}\eta^Q(Y_{(i)}(x)).$$ It is easy to see that $\bar{\eta}^Q_{k_Q}(x)$ depends on the distribution of $X_{1:n_Q}$, and the conditional expectation
$$\E_{\mathcal{D}_Q}[\hat{\eta}^Q_{k_Q}(x)|X_{1:n_Q}]=\bar{\eta}^Q_{k_Q}(x).$$
Recall that Condition \ref{con:smooth} states that for any $x,y\in\mathbb{R}^d$, we have
$$|\eta^Q(x)-\eta^Q(y)|\leq C_\beta\|x-y\|^\beta.$$
Therefore, on the event $E_Q$, we have
$$
\begin{aligned}
|\bar{\eta}^Q_{k_Q}(x)-\eta^Q(x)|&\leq \frac{1}{k_Q}\sum_{i=1}^{k_Q}|\eta^Q(Y_{(i)}(x))-\eta^Q(x)|\\
&\leq C_\beta\frac{1}{k_Q}\sum_{i=1}^{k_Q}\|X_{(i)}(x)-x\|^\beta\\
&\leq C_\beta c_D^\beta c_Q^{\frac{\beta}{d}}n_Q^{-\frac{\beta}{2\beta+d}}.
\end{aligned}
$$
On the other hand, conditioning on $X_{1:n_Q}$, the quantities $\{Y_{(i)}(x)-\eta^Q(Y_{(i)}(x))\}_{i=1,\cdots,k_Q}$ are independent with mean $0$. Hence, the Hoeffding's inequality tells that for any $t\geq 0$
\begin{equation}
\pr_{\mathcal{D}_Q}(|\hat{\eta}^Q_{k_Q}(x)-\bar{\eta}^Q_{k_Q}(x)|\geq t|X_{1:n_Q})\leq 2\exp(-\frac{2t^2}{k_Q})\leq 2\exp(-\frac{2}{c_Q}(\frac{t}{n_Q^{-\frac{\beta}{2\beta+d}}})^2).
\label{eq:prop3knnHoeffdingQ}
\end{equation}
On the event $E_Q,\ \bar{\eta}^Q_{k_Q}(x)$ falls into the interval $[\eta^Q(x)- C_\beta c_D^\beta c_Q^{\frac{\beta}{d}}n_Q^{-\frac{\beta}{2\beta+d}},\eta^Q(x)+ C_\beta c_D^\beta c_Q^{\frac{\beta}{d}}n_Q^{-\frac{\beta}{2\beta+d}}]$, so with probability at least $1-\frac{c_D}{c_Q}n_Q^{\frac{d}{2\beta+d}}\exp(-c_Qn_Q^{\frac{2\beta}{2\beta+d}})$ w.r.t. the distribution of $X_{1:n_Q}$, we have for any $t\geq 2C_\beta c_D^\beta c_Q^{\frac{\beta}{d}}n_Q^{-\frac{\beta}{2\beta+d}}$,
\begin{equation}
\begin{aligned}
\pr_{\mathcal{D}_Q}(|\hat{\eta}^Q_{k_Q}(x)-\eta^Q(x)|\geq t|X_{1:n_Q})&\leq \pr_{\mathcal{D}_Q}(|\hat{\eta}^Q_{k_Q}(x)-\bar{\eta}^Q_{k_Q}(x)|\geq t/2|X_{1:n_Q})\\
&\leq 2\exp(-\frac{1}{2c_Q}(\frac{t}{n_Q^{-\frac{\beta}{2\beta+d}}})^2).
\label{eq:prop3first1}
\end{aligned}
\end{equation}
In addition, for any $t\in[0,2C_\beta c_D^\beta c_Q^{\frac{\beta}{d}}n_Q^{-\frac{\beta}{2\beta+d}}]$, since $\exp(-\frac{1}{2c_Q}(\frac{t}{n_Q^{-\frac{\beta}{2\beta+d}}})^2)$ are bounded below from $0$, we have
\begin{equation}
    \pr_{\mathcal{D}_Q}(|\hat{\eta}^Q_{k_Q}(x)-\eta^Q(x)|\geq t|X_{1:n_Q})\leq C\exp(-\frac{1}{2c_Q}(\frac{t}{n_Q^{-\frac{\beta}{2\beta+d}}})^2).
    \label{eq:prop3first2}
\end{equation}
for some $C>0$ large enough. Combining \eqref{eq:prop3first1} and \eqref{eq:prop3first2}, the first statement holds with $c_1=\frac{c_D}{c_Q}$ and $c_2=\max\{C,2\}$. 
\end{proof}

\begin{proof}[Proof of Lemma \ref{lemma:nonParaGeneralP}]
Consider any $(Q,P)\in\Pi^{NP,s}$. Suppose that $x\in\Omega^+(\gamma,C_\gamma)$. Similarly, define the simple average of the regression functions of the $k_P$ neighbors of $x$ as $$\bar{\eta}^P_{k_P}(x):=\frac{1}{k_P}\sum_{i=1}^{k_P}\eta^Q(Y^P_{(i)}(x)),$$ which depends on the distribution of $X^P_{1:n_P}$, and the conditional expectation
$$\E_{\mathcal{P}_Q}[\hat{\eta}^P_{k_P}(x)|X_{1:n_P}]=\bar{\eta}^P_{k_P}(x).$$
Again, Condition \ref{con:smooth} gives that for any $x,y\in\mathbb{R}^d$, we have
$$|\eta^P(x)-\eta^P(y)|\leq C_{\beta_P}\|x-y\|^{\beta_P}.$$
Therefore, on the event $E_P$ we have
\begin{equation}
\begin{aligned}
|\bar{\eta}^P_{k_P}(x)-\eta^P(x)|&\leq \frac{1}{k_P}\sum_{i=1}^{k_P}|\eta^P(Y^P_{(i)}(x))-\eta^P(x)|\\
&\leq C_{\beta_P}\frac{1}{k_P}\sum_{i=1}^{k_P}\|X^P_{(i)}(x)-x\|^{\beta_P}\\
&\leq C_{\beta_P} c_D^{\beta_P} c_P^{\frac{{\beta_P}}{d}}n_P^{-\frac{{\beta_P}}{2{\beta^*_P}+d}}.
\end{aligned}
\label{eq:prop3knnBoundP}
\end{equation}

Setting $c_b\geq(2C_{\beta_P} c_D^{\beta_P} c_Q^{\frac{{\beta_P}}{d}}/C_\gamma)^{\frac{1}{\gamma}}$, \eqref{eq:prop3knnBoundP} implies that on the event $E_P$, if $|\eta^Q(x)-\frac{1}{2}|\geq c_b n_P^{-\frac{\beta_P/\gamma}{2\beta_P^*+d}}$, we have
$$|\eta^P(x)-\frac{1}{2}|\geq 2C_{\beta_P} c_D^{\beta_P} c_Q^{\frac{{\beta_P}}{d}}n_P^{-\frac{{\beta_P}}{2{\beta^*_P}+d}}\Rightarrow (\bar{\eta}^P_{k_P}(x)-\frac{1}{2})(\eta^P(x)-\frac{1}{2})\geq 0, |\bar{\eta}^P_{k_P}(x)-\frac{1}{2}|\geq \frac{1}{2}|\eta^P(x)-\frac{1}{2}|.$$
The Hoeffding's inequality then tells that for any $t\geq 0$
\begin{equation}
\pr_{\mathcal{D}_P}(|\hat{\eta}^P_{k_P}(x)-\bar{\eta}^P_{k_P}(x)|\geq t|X^P_{1:n_P})\leq 2\exp(-\frac{2t^2}{k_P})\leq 2\exp(-\frac{2}{c_P}(\frac{t}{n_P^{-\frac{\beta^*_P}{2\beta^*_P+d}}})^2).
\label{eq:prop3knnHoeffdingP}
\end{equation}
Define $$\Omega^+_1(\gamma,C_\gamma):=\{x\in\Omega^+(\gamma,C_\gamma):|\eta^Q(x)-\frac{1}{2}|\geq c_b n_P^{-\frac{\beta_P/\gamma}{2\beta_P^*+d}}\}.$$ If $x\in\Omega^+_1(\gamma,C_\gamma)$, then on the event $E_P$, we have
\begin{equation}
\begin{aligned}
\pr_{\mathcal{D}_P}((\hat{\eta}^P(x)-\frac{1}{2})(\eta^Q(x)-\frac{1}{2})<0|X^P_{1:n_P})&\leq \pr_{\mathcal{D}_P}(|\hat{\eta}^P_{k_P}(x)-\bar{\eta}^P_{k_P}(x)|\geq \frac{1}{2}|\eta^P(x)-\frac{1}{2}||X^P_{1:n_P})\\
&\leq 2\exp\left(-(\frac{C_\gamma|\eta^Q(x)-\frac{1}{2}|^\gamma}{\sqrt{\frac{c_P}{2}}n_P^{-\frac{\beta^*_P}{2\beta^*_P+d}}})^2\right).
\end{aligned}
\label{eq:prop3knnHoeffdingP1}
\end{equation}
On the other hand, setting $c_b\geq c_Dc_P^{\frac{1}{d}}$, \eqref{eq:prop3knnP} implies that if $B(x,c_bn_P^{-\frac{1}{2\beta^*_P+d}})\cap \Omega_P\subset \Omega^+(\gamma,C_\gamma)$, on the event $E_P$ we have
$$X^P_{(i)}(x)\in\Omega^+(\gamma,C_\gamma)\Rightarrow |\eta^P(X^P_{(i)}(x))-\frac{1}{2}|\geq C_\gamma|\eta^Q(X^P_{(i)}(x))-\frac{1}{2}|^\gamma.$$
Define $$\Omega^+_2(\gamma,C_\gamma):=\{x\in\Omega^+(\gamma,C_\gamma):B(x,c_bn_P^{-\frac{1}{2\beta^*_P+d}})\cap\Omega\cap \Omega_P\subset \Omega^+(\gamma,C_\gamma)\}.$$ If $x\in\Omega^+_2(\gamma,C_\gamma)$ satisfies $\eta^Q(x)-1/2\geq 2C_\beta c_D^{\beta}c_P^{\frac{\beta}{d}}n_P^{-\frac{\beta}{2\beta_P^*+d}}$, then on the event $E_P$, we have
$$
\begin{aligned}
\bar{\eta}^P_{k_P}(x)-\frac{1}{2}&=\frac{1}{k_P}\sum_{i=1}^{k_P}(\eta^P(X^P_{(i)}(x))-\frac{1}{2})\\
&\geq C_\gamma\frac{1}{k_P}\sum_{i=1}^{k_P}(\eta^Q(X^P_{(i)}(x))-\frac{1}{2})^\gamma\\
&\geq C_\gamma\frac{1}{k_P}\sum_{i=1}^{k_P}(\eta^Q(x)-\frac{1}{2}-C_\beta\|X^P_{(i)}(x)-x\|^\beta)^\gamma\\
&\geq C_\gamma\frac{1}{k_P}\sum_{i=1}^{k_P}(\eta^Q(x)-\frac{1}{2}-C_\beta(c_Dc_P^{\frac{1}{d}}n_P^{-\frac{1}{2\beta^*_P+d}})^\beta)^\gamma\\
&\geq C_\gamma \frac{1}{2^\gamma}(\eta^Q(x)-\frac{1}{2})^\gamma.
\end{aligned}
$$
The Hoeffding's inequality then tells that
\begin{equation}
\begin{aligned}
\pr_{\mathcal{D}_P}((\hat{\eta}^P(x)-\frac{1}{2})(\eta^Q(x)-\frac{1}{2})<0|X^P_{1:n_P})&\leq \pr_{\mathcal{D}_P}(|\hat{\eta}^P_{k_P}(x)-\bar{\eta}^P_{k_P}(x)|\geq C_\gamma \frac{1}{2^\gamma}|\eta^Q(x)-\frac{1}{2}|^\gamma|X^P_{1:n_P})\\
&\leq 2\exp\left(-(\frac{C_\gamma|\eta^Q(x)-\frac{1}{2}|^\gamma}{2^\gamma\sqrt{\frac{c_P}{2}}n_P^{-\frac{\beta^*_P}{2\beta^*_P+d}}})^2\right).
\end{aligned}
\label{eq:prop3knnHoeffdingP2}
\end{equation}
\eqref{eq:prop3knnHoeffdingP2} holds for $\eta^Q(x)-\frac{1}{2}\leq -2C_\beta c_D^{\beta}c_P^{\frac{\beta}{d}}n_P^{-\frac{\beta}{2\beta_P^*+d}}$ with the similar argument. To summarize, if $x\in\Omega^+_2(\gamma,C_\gamma)$ satisfies $|\eta^Q(x)-\frac{1}{2}|\geq 2C_\beta c_D^{\beta}c_P^{\frac{\beta}{d}}n_P^{-\frac{\beta}{2\beta_P^*+d}}$, then on the event $E_P$, we have
$$\pr_{\mathcal{D}_P}((\hat{\eta}^P(x)-\frac{1}{2})(\eta^Q(x)-\frac{1}{2})<0|X^P_{1:n_P})\leq 2\exp(-(\frac{C_\gamma|\eta^Q(x)-\frac{1}{2}|^\gamma}{2^\gamma\sqrt{\frac{c_P}{2}}n_P^{-\frac{\beta^*_P}{2\beta^*_P+d}}})^2).$$

We finish the proof of this lemma divided into two cases:

\noindent\textbf{I. Case of $\gamma\beta \geq \beta_P:$} If $x\in\Omega^+_2(\gamma,C_\gamma)$ satisfies $$|\eta^Q(x)-\frac{1}{2}|\geq 2C_\beta c_D^{\beta}c_P^{\frac{\beta}{d}}n_P^{-\frac{\beta}{2\beta_P^*+d}},$$ then $\exp(-(\frac{C_\gamma|\eta^Q(x)-\frac{1}{2}|^\gamma}{2^\gamma\sqrt{\frac{c_P}{2}}n_P^{-\frac{\beta^*_P}{2\beta^*_P+d}}})^2)$ is bounded below from $0$, which means on the event $E_P$,
$$\pr_{\mathcal{D}_P}((\hat{\eta}^P(x)-\frac{1}{2})(\eta^Q(x)-\frac{1}{2})<0|X^P_{1:n_P})\leq C\exp(-(\frac{C_\gamma|\eta^Q(x)-\frac{1}{2}|^\gamma}{2^\gamma\sqrt{\frac{c_P}{2}}n_P^{-\frac{\beta^*_P}{2\beta^*_P+d}}})^2)$$ for some constant $C>0$ large enough. In other words, the exponential concentration holds within the whole $\Omega^+_2(\gamma,C_\gamma)$. Combining this concentration bound with \eqref{eq:prop3knnHoeffdingP1} and \eqref{eq:prop3knnHoeffdingP2}, we see that by setting that $c_b=\max\{(2C_{\beta_P} c_D^{\beta_P} c_Q^{\frac{{\beta_P}}{d}}/C_\gamma)^{\frac{1}{\gamma}},c_Dc_P^{\frac{1}{2}}\},$
\begin{equation}
   \pr_{\mathcal{D}_P}((\hat{\eta}^P(x)-\frac{1}{2})(\eta^Q(x)-\frac{1}{2})<0|X^P_{1:n_P})\leq \max\{C,2\} \exp(-(\frac{C_\gamma|\eta^Q(x)-\frac{1}{2}|^\gamma}{2^\gamma\sqrt{\frac{c_P}{2}}n_P^{-\frac{\beta^*_P}{2\beta^*_P+d}}})^2)
   \label{eq:prop3Hoeffding}
\end{equation}
holds with probability at least $1-\frac{c_D}{c_P}n_P^{\frac{d}{2\beta^*_P+d}}\exp(-c_Qn_P^{\frac{2\beta^*_P}{2\beta^*_P+d}})$ w.r.t. the distribution of $X^P_{1:n_P}$ for any $x\in\Omega^+_1(\gamma,C_\gamma)\cup \Omega^+_2(\gamma,C_\gamma)=\Omega^+(\gamma,C_\gamma)/\Omega^b$. The statement then holds by directly applying Theorem \ref{thm:STRPlugIn} with the setting of \begin{equation}
    \delta_P=2^\gamma\sqrt{\frac{c_P}{2}}n_P^{-\frac{\beta^*_P}{2\beta^*_P+d}},\ \delta_{P,f}=\pr_{\mathcal{D}_P}[E_P^c],\ \delta_P^b=\varepsilon_b.
    \label{eq:prop3setting}
\end{equation}

\noindent\textbf{II. Case of $\gamma\beta < \beta_P:$} Without loss of generality, we suppose that $n_Q$ is large enough such that
$$2C_\beta c_D^{\beta}c_P^{\frac{\beta}{d}}n_P^{-\frac{\beta}{2\beta_P^*+d}}\geq (2C_{\beta_P} c_D^{\beta_P} c_Q^{\frac{{\beta_P}}{d}}/C_\gamma)^{\frac{1}{\gamma}}n_P^{-\frac{\beta_P/\gamma}{2\beta_P^*+d}}.$$ Hence, if $x\in\Omega^+_2(\gamma,C_\gamma)$ satisfies $$|\eta^Q(x)-\frac{1}{2}|\geq 2C_\beta c_D^{\beta}c_P^{\frac{\beta}{d}}n_P^{-\frac{\beta}{2\beta_P^*+d}},$$ then by setting $c_b=\max\{(2C_{\beta_P} c_D^{\beta_P} c_Q^{\frac{{\beta_P}}{d}}/C_\gamma)^{\frac{1}{\gamma}},c_Dc_P^{\frac{1}{2}}\}$, it holds that $x\in\Omega^+_1(\gamma,C_\gamma)$. Therefore, \eqref{eq:prop3Hoeffding} holds with probability at least $1-\frac{c_D}{c_P}n_P^{\frac{d}{2\beta^*_P+d}}\exp(-c_Qn_P^{\frac{2\beta^*_P}{2\beta^*_P+d}})$ w.r.t. the distribution of $X^P_{1:n_P}$ for any $x\in\Omega^+_1(\gamma,C_\gamma)\cup \Omega^+_2(\gamma,C_\gamma)=\Omega^+(\gamma,C_\gamma)/\Omega^b$. The statement then holds by directly applying Theorem \ref{thm:STRPlugIn} with the same setting as the one in \eqref{eq:prop3setting}.
\end{proof}
\begin{proof}[Proof of Theorem \ref{thm:nonParaGeneral}]
Consider any $(Q,P)\in\Pi^{NP}$. If $\gamma\beta\geq \beta_P$, the proof is straightforward by applying the results derived in Lemma \ref{lemma:nonParaGeneralQ} and \ref{lemma:nonParaGeneralP} to Theorem \ref{thm:general} by setting $\Pi^{NP,s}=\{(Q,P)\}$ and $\beta_P^*=\gamma\beta$.

If $\gamma\beta<\beta_P$, the proof is straightforward by applying the results derived in Lemma \ref{lemma:nonParaGeneralQ} and \ref{lemma:nonParaGeneralP} to Theorem \ref{thm:general} by setting $\Pi^{NP,s}=\{(Q,P)\}$ and $\beta_P^*=\beta_P$. To illustrate the reason why the risk term $\varepsilon_b$ in Lemma \ref{lemma:nonParaGeneralP} is asymptotically dominated by $n_P^{-\frac{\beta_P(1+\alpha)/\gamma}{2\beta_P+d}}$, by Assumption \ref{assum:margin} we have
$\varepsilon_b\leq c_bn_P^{-\frac{\beta_P/\gamma}{2\beta_P+d}}Q_X(\Omega_b)\leq C_\alpha c_b^{1+\alpha} n_P^{-\frac{\beta_P(1+\alpha)/\gamma}{2\beta_P+d}}$, which finishes the proof.
\end{proof}

\subsection{Proofs of Theorem \ref{thm:nonParaSpecial}}
For convenience, we repeat the definitions that $$\varepsilon_b= \int_{\Omega^b}|\eta^Q(X)-\frac{1}{2}|dQ_X$$ and the boundary of the signal transfer set is defined as
    $$\Omega^b:=\{x\in\Omega: |\eta^Q(x)-\frac{1}{2}|\leq c_b n_P^{-\frac{\beta_P/\gamma}{2\beta_P^*+d}},
 B(x,c_bn_P^{-\frac{1}{2\beta_P^*+d}})\cap \Omega\cap\Omega_P \not\subset \Omega^+(\gamma,C_\gamma)\}
    $$
    for some constant $c_b>0$. From Theorem \ref{thm:nonParaGeneral}, which is the upper bound result for the general parametric space $\Pi^{NP}$, it suffices to show that $\varepsilon_b\lesssim n_P^{-\frac{\beta}{2\gamma\beta+d}}$ in each case listed in Theorem \ref{thm:nonParaSpecial} to prove the results.
\begin{proof}

\noindent\textbf{Band-like Ambiguity:} By Lemma \ref{lemma:BAtoAPB}, Assumption \ref{assum:ambiguity} holds with $$\varepsilon(z;\gamma,C_\gamma/2)=\left(C_\alpha z^{1+\alpha}\right)\land \left(2^{\frac{1+\alpha}{\gamma}}C_\alpha C_\gamma^{-\frac{1+\alpha}{\gamma}}\Delta^{\frac{1+\alpha}{\gamma}}\right).$$ From Theorem \ref{thm:nonParaGeneral}, it suffices to show that $\varepsilon_b\lesssim \Delta^{\frac{1+\alpha}{\gamma}}+n_P^{-\frac{\beta(1+\alpha)}{2\gamma\beta+d}}$ to prove the excess risk upper bound.

Following the same proof in Lemma \ref{lemma:BAtoAPB}, we have that
$$\{x\in\Omega:|\eta^Q(x)-\frac{1}{2}|\geq (\frac{2\Delta}{C_\gamma})^{\frac{1}{\gamma}}\}\subset\Omega^+(\gamma,C_\gamma/2).$$ Suppose that $x\in\Omega^+(\gamma,C_\gamma/2)$ satisfies $|\eta^Q(x)-\frac{1}{2}|\geq (\frac{2\Delta}{C_\gamma})^{\frac{1}{\gamma}} + C_\beta c_b^\beta n_P^{-\frac{\beta}{2\gamma\beta+d}}$. For any $x'\in B(x,c_bn_P^{-\frac{1}{2\gamma\beta+d}})$, we have
$$
|\eta^Q(x')-\eta^Q(x)|\leq C_\beta\|x'-x\|^\beta\leq C_\beta c_b^\beta n_P^{-\frac{\beta}{2\gamma\beta+d}}\Rightarrow |\eta^Q(x')-\frac{1}{2}|\geq (\frac{2\Delta}{C_\gamma})^{\frac{1}{\gamma}},
$$
which further implies that $x\in\Omega^+(\gamma,C_\gamma/2)$. In other words,
$$\Omega_b\subset \{x\in\Omega^+(\gamma,C_\gamma/2):|\eta^Q(x)-\frac{1}{2}|< (\frac{2\Delta}{C_\gamma})^{\frac{1}{\gamma}} + C_\beta c_b^\beta n_P^{-\frac{\beta}{2\gamma\beta+d}}\},$$ so by Assumption \ref{assum:margin}, it holds that
$$
\begin{aligned}
  \varepsilon_b&\leq \left((\frac{2\Delta}{C_\gamma})^{\frac{1}{\gamma}} + C_\beta c_b^\beta n_P^{-\frac{\beta}{2\gamma\beta+d}}\right) Q_X(\Omega_b)\\
  &\leq C_\alpha \left((\frac{2\Delta}{C_\gamma})^{\frac{1}{\gamma}} + C_\beta c_b^\beta n_P^{-\frac{\beta}{2\gamma\beta+d}}\right)^{1+\alpha}\\
  &\lesssim \Delta^{\frac{1+\alpha}{\gamma}}+n_P^{-\frac{\beta(1+\alpha)}{2\gamma\beta+d}}.   
\end{aligned}
$$
The proof is then finished given the inequality above.\\

\noindent\textbf{Smooth Source:} Suppose $(Q,P)\in\Pi^{NP}_S$. By Assumption \ref{assum:margin}, since $\beta_P=\gamma\beta$, we have
$\varepsilon_b\leq c_bn_P^{-\frac{\beta}{2\gamma\beta+d}}Q_X(\Omega_b)\leq C_\alpha c_b^{1+\alpha} n_P^{-\frac{\beta(1+\alpha)}{2\gamma\beta+d}}$, which finishes the proof.\\

\noindent\textbf{Strong Signal with Imperfect Transfer:} Suppose $(Q,P)\in\Pi^{NP}_F$. For any $x\in\Omega,r>0$ such that $$\eta^Q(x')\neq \frac{1}{2},\quad \forall x'\in B(x,r),$$ i.e., $B(x,r)$ does not intersect with the decision boundary $\{x\in\Omega:\eta^Q(x)=\frac{1}{2}\}$, by continuity of $\eta^Q$ we see that either $\eta^Q(x')> \frac{1}{2}$ or $\eta^Q(x')< \frac{1}{2}$ for any $x'\in B(x,r)$.

We claim that either $$B(x,r)\cap\Omega\subset \Omega^+(\gamma,C_\gamma)$$ or $$B(x,r)\cap\Omega\subset \Omega/\Omega^+(\gamma,C_\gamma).$$ Otherwise, suppose for any two points $x_1,x_2\in B(x,r)\cap\Omega$, we have $x_1\in\Omega^+(\gamma,C_\gamma)$ but $x_2\in\Omega/\Omega^+(\gamma,C_\gamma)$. Since $|\eta^P(x)-\frac{1}{2}|\geq C_\gamma |\eta^Q(x)-\frac{1}{2}|^\gamma$ holds for any $x\in\Omega$, we have $$
\begin{aligned}
   &(\eta^P(x_1)-\frac{1}{2})(\eta^Q(x_1)-\frac{1}{2})>0,\ (\eta^P(x_2)-\frac{1}{2})(\eta^Q(x_2)-\frac{1}{2})<0,\\ \Rightarrow &(\eta^P(x_1)-\frac{1}{2})(\eta^P(x_2)-\frac{1}{2})<0. 
\end{aligned}$$
Therefore, there exists some $\lambda\in(0,1)$ such that 
$$\eta^P(\lambda x_1+(1-\lambda)x_2)=\frac{1}{2},$$
which contradicts with the facts that $\eta^Q(\lambda x_1+(1-\lambda)x_2)\neq \frac{1}{2}$ and $|\eta^P(x)-\frac{1}{2}|\geq C_\gamma |\eta^Q(x)-\frac{1}{2}|^\gamma$ for any $x\in\Omega$.

Since the choice of $x\in\Omega$ and $r>0$ are arbitrary, we conclude that any sphere in $\Omega$ that does not intersect with the decision boundary $\{x\in\Omega:\eta^Q(x)=\frac{1}{2}\}$ is a subset of either $\Omega^+(\gamma,C_\gamma)$ or $\Omega/\Omega^+(\gamma,C_\gamma)$. Therefore, the signal transfer boundary is a part of the decision boundary, i.e., $$\partial\Omega^+(\gamma, C_\gamma)\subset \{x\in\Omega:\eta^Q(x)=\frac{1}{2}\},$$ which is a sufficient condition for $\varepsilon_b\lesssim n_P^{-\frac{\beta(1+\alpha)}{2\gamma\beta+d}}$ as shown in the first case of Corollary \ref{corol:deltaPbDominated}.

\end{proof}

\subsection{Proofs of Corollary \ref{corol:deltaPbDominated}}
\begin{proof}
\noindent\textbf{Signal Transfer Boundary as part of the Decision Boundary:} Suppose that $x\in\Omega^+(\gamma,C_\gamma)$ satisfies $|\eta^Q(x)-\frac{1}{2}|\geq 2 C_\beta c_b^\beta n_P^{-\frac{\beta}{2\gamma\beta+d}}$. For any $x'\in B(x,c_bn_P^{-\frac{1}{2\gamma\beta+d}})$, we have
$$
|\eta^Q(x')-\eta^Q(x)|\leq C_\beta\|x'-x\|^\beta\leq C_\beta c_b^\beta n_P^{-\frac{\beta}{2\gamma\beta+d}}\leq \frac{1}{2}|\eta^Q(x)-\frac{1}{2}|,
$$
which further deduces that
\begin{equation}
(\eta^Q(x')-\frac{1}{2})(\eta^Q(x)-\frac{1}{2})>0,\quad \forall x'\in B(x,c_bn_P^{-\frac{1}{2\gamma\beta+d}}).
\label{eq:corol1state6}
\end{equation}
Since $\partial\Omega^+(\gamma,C_\gamma)\subset \{x\in\Omega:\eta^Q(x)=\frac{1}{2}\}$, \eqref{eq:corol1state6} implies that $$B(x,c_bn_P^{-\frac{1}{2\gamma\beta+d}})\cap\Omega\subset \Omega^+(\gamma,C_\gamma)\Rightarrow x\notin \Omega_b.$$ In other words,
$$\Omega_b\subset \{x\in\Omega^+(\gamma,C_\gamma):|\eta^Q(x)-\frac{1}{2}|< 2 C_\beta c_b^\beta n_P^{-\frac{\beta}{2\gamma\beta+d}}\},$$ so by Assumption \ref{assum:margin}, it holds that
$$
  \varepsilon_b\leq 2 C_\beta c_b^\beta n_P^{-\frac{\beta}{2\gamma\beta+d}} Q_X(\Omega_b)\\
  \leq C_\alpha (2 C_\beta c_b^\beta)^{1+\alpha} n_P^{-\frac{\beta(1+\alpha)}{2\gamma\beta+d}}\\
  \lesssim n_P^{-\frac{\beta(1+\alpha)}{2\gamma\beta+d}},
$$
which finishes the proof.

\noindent\textbf{Signal Transfer Boundary Margin:} Suppose that $n_P$ is large enough such that
$$c_bn_P^{-\frac{1}{2\gamma\beta+d}}\leq r_m.$$ It is obvious that
$$\varepsilon_b\leq c_b n_P^{-\frac{\beta_P/\gamma}{2\gamma\beta+d}} Q_X(B(X,c_bn_P^{-\frac{1}{2\gamma\beta+d}})\cap\Omega\cap\Omega_P\not\subset\Omega^+(\gamma,C_\gamma))\leq C_{r}c_b^{1+\alpha_r}n_P^{-\frac{\beta_P/\gamma+\alpha_r}{2\gamma\beta+d}},$$
of which the last term is asymptotically less than $n_P^{-\frac{\beta(1+\alpha)}{2\gamma\beta+d}}$ given that $\beta_P/\gamma+\alpha_r\geq \beta(1+\alpha)$.
\end{proof}

\subsection{Proof of Theorem \ref{thm:nonParaSpecialLowerBound} and \ref{thm:nonParaGeneralLowerBound}}
Let $H(\mu,\nu)$ and $TV(\mu,\nu)$ denote the Hellinger distance and total variation distance between two probability measures $\mu$ and $\nu$. Before proving the minimax lower bound in Theorem \ref{thm:nonParaSpecialLowerBound} and \ref{thm:nonParaGeneralLowerBound} over $$\Pi^{NP}(\alpha,C_\alpha,\gamma,C_\gamma,\varepsilon,\beta,\beta_P,C_\beta,C_{\beta_P},\mu^+,\mu^-,c_\mu,r_\mu)$$ and its subsets
$\Pi_{BA}^{NP},\Pi_{S}^{NP},\Pi_{I}^{NP}$,
 we first prove the excess risk minimax lower bound over a special subset of $\Pi^{NP}$ and $\Pi_{BA}^{NP}$ that satisfies $\varepsilon(z;\gamma,C_\gamma)\equiv 0$ or $\Delta=0$. The parameter space is rigorously defined as
$$\Pi^{NP}_0:=\Pi^{NP}(\alpha,C_\alpha,\gamma,C_\gamma,0,\beta,\beta_P,C_\beta,C_{\beta_P},\mu^+,\mu^-,c_\mu,r_\mu)\cap\{(Q,P):\Omega^+(\gamma,C_\gamma)=\Omega=\Omega_P\}.$$ 
It is trivial to see that $\Pi^{NP}_0\subset \Pi_I^{NP}\subset\Pi^{NP}$ and $\Pi^{NP}_0\subset \Pi_{BA}^{NP}$ since the ``perfect source" setting can be viewed as a special case of those with different types of ambiguity. In addition, when $\gamma\beta=\beta_P$, we have $\Pi^{NP}_0\subset \Pi_S^{NP}$. Therefore, the minimax lower over $\Pi^{NP}_0$ must be the minimax lower bound over the desired parametric spaces,
and our first goal is to see the minimax lower bound over $\Pi^{NP}_0$.   The propositions are as follows:

\begin{proposition}
Fix the parameters in the definition of $\Pi^{NP}$ with $\alpha\beta\leq d,\gamma\beta\geq \beta_P$. We have that
$$
\inf_{\hat{f}}\sup_{(Q,P)\in\Pi^{NP}_0}\E\mathcal{E}_Q(\hat{f})
\gtrsim  n_P^{-\frac{\beta(1+\alpha)}{2\gamma\beta+d}}\land n_Q^{-\frac{\beta(1+\alpha)}{2\beta+d}}.
$$
\label{prop:nonParaGeneralLowerBound01}
\end{proposition}

\begin{proposition}
Fix the parameters in the definition of $\Pi^{NP}$ with $\max\{\alpha\beta,\alpha\beta_P/\gamma\}\leq d,\gamma\beta< \beta_P$. We have that
$$
\inf_{\hat{f}}\sup_{(Q,P)\in\Pi^{NP}_0}\E\mathcal{E}_Q(\hat{f})
\gtrsim  n_P^{-\frac{\beta_P(1+\alpha)/\gamma}{2\beta_P+d}}\land n_Q^{-\frac{\beta_P(1+\alpha)/\gamma}{2\beta_P/\gamma+d}}.
$$
\label{prop:nonParaGeneralLowerBound02}
\end{proposition}

Proposition \ref{prop:nonParaGeneralLowerBound01} and \ref{prop:nonParaGeneralLowerBound02} serve as the minimax lower bound under the regimes of $\gamma\beta\geq \beta_P$ and $\gamma\beta< \beta_P$ given that there is no ambiguity in the signal provided by $\eta^P$ relative to $\eta^Q$. Furthermore, we see that Proposition \ref{prop:nonParaGeneralLowerBound01} is an extension of Theorem 3.2 of \cite{ttcai2020classification} as we impose the condition $\eta^P=0$ as well.

The idea of the proofs are based on the application of Assouad's lemma on the family $\pr^\sigma_{\mathcal{D}_Q}\times \pr^\sigma_{\mathcal{D}_P}$. Here $\pr^\sigma_{\mathcal{D}_Q}$ and $\pr^\sigma_{\mathcal{D}_P}$ are defined as the product probability measure with respect to the target and source data, corresponding to the distribution pairs $(Q_\sigma,P_\sigma),\ \sigma\in\{1,-1\}^m$. See Section \ref{sec:prop5proof} and \ref{sec:prop6proof} for the proof details.
\begin{proof}[Proofs of Theorem \ref{thm:nonParaSpecialLowerBound} (2), (3) and Theorem \ref{thm:nonParaGeneralLowerBound}] We claim that it suffices to prove 
\begin{equation}
\inf_{\hat{f}}\sup_{\substack{(Q,P)\in\Pi_S^{NP}\\\Omega=\Omega_P}}\E\mathcal{E}_Q(\hat{f})\gtrsim  \varepsilon(cn_Q^{-\frac{\beta}{2\beta+d}};\gamma,C_\gamma).
\label{eq:nonParaGeneralLowerBoundVar}
\end{equation}
for some constant $c>0$ given any setting of $\gamma,\beta$ and $\beta_P=\gamma\beta$ to obtain Theorem \ref{thm:nonParaSpecialLowerBound} (2). To clarify, since $\Pi^{NP}_0\subset \Pi^{NP}\cap\{(Q,P):\Omega=\Omega_P\}$, the minimax lower over $\Pi^{NP}_0$ must be the minimax lower bound over $\Pi^{NP}\cap\{(Q,P):\Omega=\Omega_P\}$. Hence, 
$$\inf_{\hat{f}}\sup_{\substack{(Q,P)\in\Pi_S^{NP}\\\Omega=\Omega_P}}\E\mathcal{E}_Q(\hat{f})\gtrsim  n_P^{-\frac{\beta(1+\alpha)}{2\gamma\beta+d}}\land n_Q^{-\frac{\beta(1+\alpha)}{2\beta+d}}.$$
Since the minimax lower bounds we desire is just $\varepsilon(cn_Q^{-\frac{\beta}{2\beta+d}};\gamma,C_\gamma)$ added to the lower bounds above, it suffices to prove \eqref{eq:nonParaGeneralLowerBoundVar} for both cases to finish the proof. Similarly, it suffices to prove
$$
\inf_{\hat{f}}\sup_{\substack{(Q,P)\in\Pi_I^{NP}\\\Omega=\Omega_P}}\E\mathcal{E}_Q(\hat{f})\gtrsim  \varepsilon(cn_Q^{-\frac{\beta}{2\beta+d}};\gamma,C_\gamma)
$$
for some constant $c>0$, given any setting of $\gamma,\beta$ to obtain Theorem \ref{thm:nonParaSpecialLowerBound} (3), and
$$
\inf_{\hat{f}}\sup_{\substack{(Q,P)\in\Pi^{NP}\\\Omega=\Omega_P}}\E\mathcal{E}_Q(\hat{f})\gtrsim  \varepsilon(cn_Q^{-\frac{\beta}{2\beta+d}};\gamma,C_\gamma)
$$
for some constant $c>0$ given any setting of $\gamma,\beta$, and $\beta_P$ to obtain Theorem \ref{thm:nonParaGeneralLowerBound}.

We only prove the case of $\Pi_S^{NP}$. The cases of $\Pi_I^{NP}$ and $\Pi^{NP}$, no matter what $\gamma,\beta$ and $\beta_P$ are, are the same because it is easy to check that our construction of distribution pairs satisfy all conditions, i.e., $\{(Q_\sigma,P_\sigma),\ \sigma\in\{1,-1\}^m\}\subset \Pi_S^{NP}\cap \Pi_I^{NP} \cap \Pi^{NP} \cap \{(Q,P):\Omega=\Omega_P\}$ for any $\gamma,\beta$ and $\beta_P$. 

We would like to illustrate our proof idea as follows. In order to satisfy the ambiguity level condition, we need to select a smaller number $m$ of spheres with positive density than the classical case of no source data, as stated in Theorem 4.1 of \cite{atsybakov2007fastPlugin}. To achieve this, we set $\eta^P\equiv 1$ and select the maximum possible value of $m$ while still satisfying the given ambiguity level $\varepsilon(z;\gamma,C_\gamma)$.

Fix $\eta^P\equiv 1$. Since $\eta^P$ is a constant, it is trivial that $\eta^P\in\mathcal{H}(\beta_P,C_{\beta_P})$ for any $\beta_P,C_{\beta_P}$. Next, define the quantities
$$r=c_r n_Q^{-\frac{1}{2\beta+d}},\quad w=c_wr^d,\quad m=\lceil\frac{c_m \varepsilon(C_\beta c_r^\beta n_Q^{-\frac{\beta}{2\beta+d}};\gamma,C_\gamma)}{C_\beta w r^\beta}\rceil,$$
where the constants $c_r,c_w,c_m$ will be specified later in the proof. Here, $\lceil a\rceil$ is the minimum integer that is greater equal than $a$ for any real value $a$. In addition, suppose $n_Q$ is large enough such that $$m\leq \frac{2c_m \varepsilon(C_\beta c_r^\beta  n_Q^{-\frac{\beta}{2\beta+d}};\gamma,C_\gamma)}{C_\beta w r^\beta}.$$
Suppose that $\varepsilon(C_\beta c_r^\beta n_Q^{-\frac{\beta}{2\beta+d}};\gamma,C_\gamma)>0$, or otherwise the lower bound is $0$ and the proof is trivial. 

We consider a packing $\{x_k\}_{k=1,\cdots,m}$ with radius $2r$ in $[0,1]^d$. For an example of such constuction, divide $[0,1]^d$ into uniform small cubes with side length as $6r$, which forms a grid with at least $\lfloor (6r)^{-1}\rfloor^d$ small cubes. Since
$$\varepsilon(C_\beta c_r^\beta n_Q^{-\frac{\beta}{2\beta+d}};\gamma,C_\gamma)\leq C_\alpha (C_\beta c_r^\beta)^{1+\alpha} n_Q^{-\frac{\beta(1+\alpha)}{2\beta+d}}=C_\alpha c_r^{-\beta(1+\alpha)}(C_\beta c_r^\beta )^{1+\alpha}r^{\beta(1+\alpha)},$$
we have
$$
\begin{aligned}
  m&\leq \frac{2c_m \varepsilon(2n_Q^{-\frac{\beta}{2\beta+d}};\gamma,C_\gamma)}{C_\beta w r^\beta}\\
  &\leq \frac{2c_m C_\alpha c_r^{-\beta(1+\alpha)}(C_\beta c_r^\beta )^{1+\alpha}r^{\beta(1+\alpha)}}{C_\beta w r^\beta}\\
  &=\frac{2c_m C_\alpha c_r^{-\beta(1+\alpha)}(C_\beta c_r^\beta )^{1+\alpha}}{C_\beta c_w}r^{\alpha\beta-d}\\
  &\ll\lfloor (6r)^{-1}\rfloor^d.  
\end{aligned}
$$

Therefore, we could suppose that $c_m$ is small enough such that $m<\lfloor (6r)^{-1}\rfloor^d$. Therefore, we could assign exactly one sphere with radius $2r$ in one cube without intersection, which forms a packing of $\{x_k\}_{k=1,\cdots,m}$ with radius $2r$ in $[0,1]^d$. For simplicity of notations,  we denote $\frac{2c_m C_\alpha c_r^{-\beta(1+\alpha)}(C_\beta c_r^\beta )^{1+\alpha}}{C_\beta c_w}$ by $c'_m$ that converges to $0$ when $c_m\rightarrow 0$. Also, define $B_c$ as the compliment of these $m$ balls, i.e.
$B_c:=[0,1]^d/\left(\bigcup_{k=1}^{m} B(x_k,2r)\right)$.

For any $\sigma\in\{1,-1\}^{m}$, we consider the regression function $\eta^Q_\sigma(x)$ defined as follows:
$$
  \eta^Q_\sigma(x)=
\begin{cases}
\frac{1}{2}+\sigma_k C_\beta r^{\beta}g^{\beta}(\frac{\|x-x_k\|}{r})\quad &\text{if $x\in B(x_k,2r)$ for some $k=1,\cdots,m$}\\
\frac{1}{2}\quad &\text{otherwise,}
\end{cases}
$$
where the function $g(\cdot)$ is defined as $g(x)=\min\{1,2-x\}$ on $x\in[0,2]$. 

The construction of the marginal distributions $Q_{\sigma,X}$ is as follows. Define
$$r_0=\left(\frac{c_m \varepsilon(C_\beta c_r^\beta n_Q^{-\frac{\beta}{2\beta+d}};\gamma,C_\gamma)-(m-1)C_\beta w r^\beta}{C_\beta c_w}\right)^{\frac{1}{\beta+d}},\quad w_0=c_w r_0^d.$$
By the definition of $m$, we have that 
$$r_0\leq (\frac{C_\beta w r^\beta}{C_\beta c_w})^{\frac{1}{\beta+d}}=r.$$
Let $Q_{\sigma,X}$ have the density function $\mu(\cdot)$ defined as
$$\mu(x)=
\begin{cases}
\frac{w}{\lambda[B(x_k,r)]}\quad &\text{if $x\in B(x_k,r)$ for some $k=1,\cdots,m-1$}\\
\frac{w_0}{\lambda[B(x_m,r_0)]}\quad &\text{if $x\in B(x_m,r_0)$}\\
\frac{1-(m-1)w-w_0}{1-m\lambda[B(x_k,2r)]}\quad &\text{if $x\in B_c$}\\
0\quad &\text{otherwise.}
\end{cases}
$$
Since the density function $\mu(\cdot)$ does not depend on the specific choice of $\sigma$, we fix $P_X=Q_{\sigma,X}$ for any $\sigma\in\{1,-1\}^{m}$. Given the the construction of $Q_\sigma$, we next verify that $(Q_\sigma,P)$ belongs to $\Pi_S^{NP}$ for any $\sigma\in\{1,-1\}^{m}$.

\vspace{\baselineskip}
\noindent\textbf{Verify Margin Assumption:} We have that
$$
\begin{aligned}
Q_\sigma(0<|\eta_\sigma^Q-\frac{1}{2}|<t)&=(m-1) Q_\sigma(0<C_\beta r^\beta g^\beta(\frac{\|X-x_1\|}{r})\leq t)\\
&+Q_\sigma(0<C_\beta r^\beta g^\beta(\frac{\|X-x_m\|}{r})\leq t)\\
&=(m-1) Q_\sigma(0<g(\frac{\|X-x_1\|}{r})\leq (\frac{t}{C_\beta r^\beta})^{\frac{1}{\beta}})\\
&+Q_\sigma(0<g(\frac{\|X-x_m\|}{r})\leq (\frac{t}{C_\beta r^\beta})^{\frac{1}{\beta}})\\
&=((m-1)w+w_0)\mathbf{1}\{t\geq C_\beta r^\beta\}\\
&\leq mw\mathbf{1}\{t\geq C_\beta r^\beta\}\\
&\leq c'_mc_w r^{\alpha\beta}\mathbf{1}\{t\geq C_\beta r^\beta\}\\
&\leq C_\alpha t^\alpha.
\end{aligned}
$$
given that $c_m$ is small enough since $c_m'\overset{c_m\rightarrow 0}{\longrightarrow 0}$. Therefore, it holds that $Q_\sigma\in\mathcal{M}(\alpha,C_\alpha)$.

\vspace{\baselineskip}
\noindent\textbf{Verify Smoothness Assumption:}
It is easy to see that for any $a,b\in[0,2]$ we have
$$
|g^{\beta}(a)-g^{\beta}(b)|\leq |a-b|^\beta.
$$
Thus, for any $x,x'\in B(x_k,2r)$, we obtain from the triangular inequality $\|x-x_k\|-\|x'-x_k\|\leq \|x-x'\|$ that
\begin{equation}
|r^{\beta}g^{\beta}(\|x-x_k\|/r))-r^{\beta}g^{\beta}(\|x'-x_k\|/r))|\leq\|x-x'\|^{\beta},
\label{eq:theorem5smooth2}
\end{equation}
Therefore, by the definition of $\eta^Q_\sigma$ and \eqref{eq:theorem5smooth2}, we see that $$|\eta^Q_\sigma(x)-\eta^Q_\sigma(x')|\leq C_{\beta}\|x-x'\|^{\beta},\quad \forall x,x\in[0,1]^d.$$ $\eta^P$ is smooth with any smoothness parameter $\beta_P$.

\vspace{\baselineskip}
\noindent\textbf{Verify Strong Density Condition:}
If $x\in B(x_k,r)$ for some $k=1,\cdots,m-1$, we have
$$\mu(x)=\frac{w}{\lambda[B(x_k,r)]}=c_w/\pi_d.$$
If $x\in B(x_m,r_0)$ we have
$$\mu(x)=\frac{w_0}{\lambda[B(x_m,r_0)]}=c_w/\pi_d.$$
If $x\in B_c$, we have
$$(m-1)w+w_0\leq mw\leq c'_m r^{\alpha\beta}\overset{n_Q\rightarrow \infty}{\longrightarrow} 0,$$
and
$$m\lambda[B(x_k,2r)]\leq c'_m2^d\pi_d r^{\alpha\beta} \overset{n_Q\rightarrow \infty}{\longrightarrow} 0.$$
Hence,
$$\mu(x)=\frac{1-(m-1)w-w_0}{1-m\lambda[B(x_k,2r)]}\overset{n_Q\rightarrow \infty}{\longrightarrow}1.$$
Combining all cases above, we could set $c_w\in[\pi_d\mu^-,\pi_d\mu^+]$ to satisfy the condition $(Q_\sigma,P)\in\mathcal{S}(\mu)$.

\vspace{\baselineskip}
\noindent\textbf{Verify Ambiguity Level:}
Since $\Omega=\Omega_P$, we trivially have that $$\Omega^-(\gamma,C_\gamma)\subset \{x\in\Omega:\eta^Q(x)\neq\frac{1}{2}\}=:\Omega^-.$$ Then,
$$
\begin{aligned}
A(z):=&\int_{\Omega^-(\gamma,C_\gamma)}|\eta^Q(X)-\frac{1}{2}|\mathbf{1}\{0<|\eta^Q(X)-\frac{1}{2}|\leq z\}dQ_X\\
\leq& \int_{\Omega^-}|\eta^Q(X)-\frac{1}{2}|\mathbf{1}\{0<|\eta^Q(X)-\frac{1}{2}|\leq z\}dQ_X\\
=& ((m-1)w+w_0) C_\beta r^\beta \mathbf{1}\{z\geq C_\beta r^\beta\}.
\end{aligned}
$$
If $z<C_\beta r^\beta$, we have $$A(z)=0\leq \varepsilon(z;\gamma,C_\gamma).$$
If $z\geq C_\beta r^\beta$, we have
$$
  A(z)\leq C_\beta mwr^\beta\leq 2c_m \varepsilon(C_\beta c_r^\beta n_Q^{-\frac{\beta}{2\beta+d}};\gamma,C_\gamma)\leq \varepsilon(z;\gamma,C_\gamma).  
$$
Therefore, the given ambiguity level $\varepsilon(z;\gamma,C_\gamma)$ is well-defined provided that $c_m\leq \frac{1}{2}$.

Putting all verification steps above together, we conclude that $(Q_\sigma,P)\in\Pi_S^{NP}$. We finish the proof of this part by applying Assouad's lemma to $(Q_\sigma,P),\ \forall \sigma\in\{1,-1\}^{m}$.

Suppose that $\sigma,\sigma'\in\{1,-1\}^{m}$ differ only at one coordinate, i.e.
$$\sigma_k=-\sigma'_k,\quad \sigma_l=\sigma'_l\ (\forall l\neq k).$$ If $k\neq m$, we have the Hellinger distance bound as
$$
\begin{aligned}
H^2(Q_\sigma,Q_{\sigma'})&=\frac{1}{2}\int \left(\sqrt{\eta^Q_\sigma(X)}-\sqrt{\eta^Q_{\sigma'}(X)}\right)^2+\left(\sqrt{1-\eta^Q_\sigma(X)}-\sqrt{1-\eta^Q_{\sigma'}(X)}\right)^2 dQ_X\\
&=\int_{B(x_k,r)}\frac{w}{\lambda[B(x_k,r)]}\left(\sqrt{\frac{1}{2}+C_\beta r^\beta}-\sqrt{\frac{1}{2}-C_\beta r^\beta}\right)^2 dx\\
&= \frac{1}{2}w(1-\sqrt{1-2C^2_\beta r^{2\beta}})\\
&\leq C^2_\beta wr^{2\beta}.
\end{aligned}
$$
Similarly, if $k=m$, the Hellinger distance bound should be
$$
H^2(Q_\sigma,Q_{\sigma'})\leq C^2_\beta w_0r^{2\beta}\leq C^2_\beta w r^{2\beta}.
$$
Recall that $r=c_r n_Q^{-\frac{1}{2\beta+d}}$. By the property of Hellinger distance, we have
$$
H^2(\pr^\sigma_{\mathcal{D}_Q},\pr^{\sigma'}_{\mathcal{D}_Q})\leq n_Q H^2(Q_\sigma,Q_{\sigma'})
\leq  n_Q C^2_\beta w r^{2\beta}
\leq  C^2_\beta c_r^{2\beta+d}c_w
\leq  \frac{\sqrt{2}}{4}
$$
provided that $c_r$ is small enough. This further indicates that
\begin{equation}
TV(\pr^\sigma_{\mathcal{D}_Q},\pr^{\sigma'}_{\mathcal{D}_Q})\leq \sqrt{2}H^2(\pr^\sigma_{\mathcal{D}_Q},\pr^{\sigma'}_{\mathcal{D}_Q})\leq \frac{1}{2}.
\label{eq:theorem4part1H2toTV}
\end{equation}
For any empirical classifier $\hat{f}$, if $k\neq m$, we have
$$
\begin{aligned}
    \mathcal{E}_{Q_\sigma}(\hat{f})+\mathcal{E}_{Q_{\sigma'}}(\hat{f})&=2\E_{Q_\sigma}[|\eta^{Q_\sigma}(X)-\frac{1}{2}|\mathbf{1}\{\hat{f}(X)\neq f^*_{Q_\sigma}(X)\}]\\
&+2\E_{Q_{\sigma'}}[|\eta^{Q_{\sigma'}}(X)-\frac{1}{2}|\mathbf{1}\{\hat{f}(X)\neq f^*_{Q_{\sigma'}}(X)\}]\\
&\geq 2\int_{B(x_k,r)}\frac{w}{\lambda[B(x_k,r)]}C_\beta r^\beta (\mathbf{1}\{\hat{f}(X)\neq f^*_{Q_\sigma}(X)\}+\mathbf{1}\{\hat{f}(X)\neq f^*_{Q_{\sigma'}}(X)\})dx\\
&= 2C_\beta w r^\beta.
\end{aligned}
$$
If $k=m$, we have
$$
\begin{aligned}
    \mathcal{E}_{Q_\sigma}(\hat{f})+\mathcal{E}_{Q_{\sigma'}}(\hat{f})&=2\E_{Q_\sigma}[|\eta^{Q_\sigma}(X)-\frac{1}{2}|\mathbf{1}\{\hat{f}(X)\neq f^*_{Q_\sigma}(X)\}]\\
&+2\E_{Q_{\sigma'}}[|\eta^{Q_{\sigma'}}(X)-\frac{1}{2}|\mathbf{1}\{\hat{f}(X)\neq f^*_{Q_{\sigma'}}(X)\}]\\
&\geq 2\int_{B(x_m,r_0)}\frac{w_0}{\lambda[B(x_m,r_0)]}C_\beta r^\beta (\mathbf{1}\{\hat{f}(X)\neq f^*_{Q_\sigma}(X)\}+\mathbf{1}\{\hat{f}(X)\neq f^*_{Q_{\sigma'}}(X)\})dx\\
&= 2C_\beta w_0 r^\beta.
\end{aligned}
$$
Combining this lower bound with \eqref{eq:theorem4part1H2toTV}, the Assouad's lemma shows that
$$
  \sup_{(Q,P)\in\Pi_{S}^{NP}}\E\mathcal{E}_Q(\hat{f})\geq \sup_{\substack{(Q_\sigma,P)\\ \sigma\in\{1,-1\}^m}}\E\mathcal{E}_{Q_\sigma}(\hat{f})\geq \frac{1}{2}C_\beta ((m-1)w+w_0)r^\beta.  
$$
On one hand, if $m\geq 2$, we further have
$$
  \sup_{(Q,P)\in\Pi_{S}^{NP}}\E\mathcal{E}_Q(\hat{f})\geq \frac{1}{2}C_\beta ((m-1)w+w_0)r^\beta\geq \frac{1}{4}C_\beta mwr^\beta\gtrsim \varepsilon(C_\beta c_r^\beta n_Q^{-\frac{\beta}{2\beta+d}};\gamma,C_\gamma).  
$$
On the other hand, if $m=1$, we have
$$r_0=\left(\frac{c_m \varepsilon(C_\beta c_r^\beta n_Q^{-\frac{\beta}{2\beta+d}};\gamma,C_\gamma)}{C_\beta c_w}\right)^{\frac{1}{\beta+d}},\quad w_0=c_w r_0^d.$$
Therefore,
$$
  \sup_{(Q,P)\in\Pi_{S}^{NP}}\E\mathcal{E}_Q(\hat{f})\geq \frac{1}{2}C_\beta w_0r^\beta\geq \frac{1}{2}C_\beta c_w r_0^{\beta+d}\gtrsim\varepsilon(C_\beta c_r^\beta n_Q^{-\frac{\beta}{2\beta+d}};\gamma,C_\gamma).  
$$
We conclude that it always holds that $\sup_{(Q,P)\in\Pi_{BA}^{NP}}\E\mathcal{E}_Q(\hat{f})\gtrsim\varepsilon(C_\beta c_r^\beta n_Q^{-\frac{\beta}{2\beta+d}};\gamma,C_\gamma)$, which finishes the proof.
\end{proof}
\begin{proof}[Proofs of Theorem \ref{thm:nonParaSpecialLowerBound} (1)] Following a similar argument with the ones in the proofs of Theorem \ref{thm:nonParaSpecialLowerBound} (2), (3) and Theorem \ref{thm:nonParaGeneralLowerBound}, it suffices to show that 
\begin{equation}
\inf_{\hat{f}}\sup_{(Q,P)\in\Pi_{BA}^{NP}}\E\mathcal{E}_Q(\hat{f})\gtrsim  \Delta^{\frac{1+\alpha}{\gamma}}\land n_Q^{-\frac{\beta(1+\alpha)}{2\beta+d}}.
\label{eq:nonParaGeneralLowerBoundDelta}
\end{equation}
for any setting of $\gamma$ and $\beta$.

Fix $\eta^P\equiv \frac{1}{2}$. Since $\eta^P$ is a constant, it is trivial that $\eta^P\in\mathcal{H}(\beta_P,C_{\beta_P})$ for any $\beta_P,C_{\beta_P}$. Next, define the quantities
$$r=(c_r n_Q^{-\frac{1}{2\beta+d}})\land (C_\beta^{-\frac{1}{\beta}}(\frac{\Delta}{C_\gamma})^{\frac{1}{\gamma\beta}}),\quad w=c_wr^d,\quad m=\lfloor c_m r^{\alpha\beta-d}\rfloor,$$ where the constants $c_r,c_w,c_m$ will be specified later in the proof.

We consider a packing $\{x_k\}_{k=1,\cdots,m}$ with radius $2r$ in $[0,1]^d$. We repeat the example of such construction given in Proposition \ref{prop:nonParaGeneralLowerBound01}. Divide $[0,1]^d$ into uniform small cubes with side length as $6r$, which forms a grid with at least $\lfloor (6r)^{-1}\rfloor^d$ small cubes. Since $m\ll\lfloor (6r)^{-1}\rfloor^d$, we suppose that $c_m$ is small enough such that $m<\lfloor (6r)^{-1}\rfloor^d$. Therefore, we could assign exactly one sphere with radius $2r$ in one cube without intersection, which forms a packing of $\{x_k\}_{k=1,\cdots,m}$ with radius $2r$ in $[0,1]^d$. Also, define $B_c$ as the compliment of these $m$ balls, i.e.
$B_c:=[0,1]^d/\left(\bigcup_{k=1}^m B(x_k,2r)\right)$.

For any $\sigma\in\{1,-1\}^m$, we consider the regression function $\eta^Q_\sigma(x)$ defined as follows:
$$
  \eta^Q_\sigma(x)=
\begin{cases}
\frac{1}{2}+\sigma_k C_\beta r^{\beta}g^{\beta}(\frac{\|x-x_k\|}{r})\quad &\text{if $x\in B(x_k,2r)$ for some $k=1,\cdots,m$}\\
\frac{1}{2}\quad &\text{otherwise,}
\end{cases}
$$
where the function $g(\cdot)$ is defined as $g(x)=\min\{1,2-x\}$ on $x\in[0,2]$. 

The construction of the marginal distributions $Q_{\sigma,X}$ is as follows. Let $Q_{\sigma,X}$ have the density function $\mu(\cdot)$ defined as
$$\mu(x)=
\begin{cases}
\frac{w}{\lambda[B(x_k,r)]}\quad &\text{if $x\in B(x_k,r)$ for some $k=1,\cdots,m$,}\\
\frac{1-mw}{1-m\lambda[B(x_k,2r)]}\quad &\text{if $x\in B_c$,}\\
0\quad &\text{otherwise.}
\end{cases}
$$
Since the density function $\mu(\cdot)$ does not depend on the specific choice of $\sigma$, we fix $P_X=Q_{\sigma,X}$ for any $\sigma\in\{1,-1\}^m$. Given the the construction of $Q_\sigma$, we next verify that $(Q_\sigma,P)$ belongs to $\Pi_{BA}^{NP}$ for any $\sigma\in\{1,-1\}^m$.

\vspace{\baselineskip}
\noindent\textbf{Verify Margin Assumption:} We have that
$$
\begin{aligned}
Q_\sigma(0<|\eta_\sigma^Q-\frac{1}{2}|<t)&=m Q_\sigma(0<C_\beta r^\beta g^\beta(\frac{\|X-x_1\|}{r})\leq t)\\
&=m Q_\sigma(0<g(\frac{\|X-x_1\|}{r})\leq (\frac{t}{C_\beta r^\beta})^{\frac{1}{\beta}})\\
&=mw\mathbf{1}\{t\geq C_\beta r^\beta\}\\
&\leq c_mc_w r^{\alpha\beta}\mathbf{1}\{t\geq C_\beta r^\beta\}\\
&\leq C_\alpha t^\alpha.
\end{aligned}
$$
given that $c_m$ is small enough. Therefore, $Q_\sigma\in\mathcal{M}(\alpha,C_\alpha)$.

\vspace{\baselineskip}
\noindent\textbf{Verify Smoothness Assumption:}
It is easy to see that for any $a,b\in[0,2]$ we have
$$
|g^{\beta}(a)-g^{\beta}(b)|\leq |a-b|^\beta.
$$
Thus, for any $x,x'\in B(x_k,2r)$, we obtain from the triangular inequality $\|x-x_k\|-\|x'-x_k\|\leq \|x-x'\|$ that
\begin{equation}
    |r^{\beta}g^{\beta}(\|x-x_k\|/r))-r^{\beta}g^{\beta}(\|x'-x_k\|/r))|\leq\|x-x'\|^{\beta}.
    \label{eq:theorem5smooth1}
\end{equation}
Therefore, by the definition of $\eta^Q_\sigma$ and \eqref{eq:theorem5smooth1}, we see that $$|\eta^Q_\sigma(x)-\eta^Q_\sigma(x')|\leq C_{\beta}\|x-x'\|^{\beta},\quad \forall x,x\in[0,1]^d.$$

\vspace{\baselineskip}
\noindent\textbf{Verify Strong Density Condition:}
If $x\in B(x_k,r)$ for some $k=1,\cdots,m$, we have
$$\mu(x)=\frac{w}{\lambda[B(x_k,r)]}=c_w/\pi_d.$$
If $x\in B_c$, we have
$$\mu(x)=\frac{1-c_w \lfloor c_m r^{\alpha\beta-d}\rfloor r^d}{1-2^d\pi_d\lfloor c_m r^{\alpha\beta-d}\rfloor r^d}\overset{n_Q\rightarrow \infty}{\longrightarrow}1.$$
Therefore, we could set $c_w\in[\pi_d\mu^-,\pi_d\mu^+]$ to satisfy the condition $(Q_\sigma,P)\in\mathcal{S}(\mu)$.

\vspace{\baselineskip}
\noindent\textbf{Verify Band-like Ambiguity:}
Since $\eta^P\equiv \frac{1}{2}$, by Definition \ref{def:signalStrength} we have $s(x)=0$ for any $x\in\Omega$.
In addition, by definition of $\eta^Q_\sigma$ and $r$, we have that for any $x\in\Omega$, $$|\eta^Q_\sigma-\frac{1}{2}|\leq C_\beta r^\beta\leq (\frac{\Delta}{C_\gamma})^{\frac{1}{\gamma}}\Rightarrow C_\gamma|\eta^Q_\sigma-\frac{1}{2}|^\gamma\leq \Delta.$$ Therefore, we have
$$s(x)=0\geq C_\gamma |\eta^Q_\sigma-\frac{1}{2}|^\gamma-\Delta$$ for any $x\in\Omega$.

Putting all verification steps above together, we conclude that $(Q_\sigma,P)\in\Pi_{BA}^{NP}$. We finish the proof of this part by applying Assouad's lemma to $(Q_\sigma,P),\ \forall \sigma\in\{1,-1\}^m$.

If $\sigma,\sigma'\in\{1,-1\}^m$ differ only at one coordinate, i.e.
$$\sigma_k=-\sigma'_k,\quad \sigma_l=\sigma'_l\ (\forall l\neq k),$$ we have the Hellinger distance bound as
$$
\begin{aligned}
H^2(Q_\sigma,Q_{\sigma'})&=\frac{1}{2}\int \left(\sqrt{\eta^Q_\sigma(X)}-\sqrt{\eta^Q_{\sigma'}(X)}\right)^2+\left(\sqrt{1-\eta^Q_\sigma(X)}-\sqrt{1-\eta^Q_{\sigma'}(X)}\right)^2 dQ_X\\
&=\int_{B(x_k,r)}\frac{w}{\lambda[B(x_k,r)]}\left(\sqrt{\frac{1}{2}+C_\beta r^\beta}-\sqrt{\frac{1}{2}-C_\beta r^\beta}\right)^2 dx\\
&= \frac{1}{2}w(1-\sqrt{1-2C^2_\beta r^{2\beta}})\\
&\leq C^2_\beta wr^{2\beta}.
\end{aligned}
$$
Recall that $r=(c_r n_Q^{-\frac{1}{2\beta+d}})\land (C_\beta^{-\frac{1}{\beta}}(\frac{\Delta}{C_\gamma})^{\frac{1}{\gamma\beta}})$. By the property of Hellinger distance, we have
$$
H^2(\pr^\sigma_{\mathcal{D}_Q},\pr^{\sigma'}_{\mathcal{D}_Q})\leq n_Q H^2(Q_\sigma,Q_{\sigma'})
\leq  C^2_\beta w n_Q r^{2\beta}
\leq  C^2_\beta c_r^{2\beta+d}c_w
\leq  \frac{\sqrt{2}}{4}
$$
provided that $c_r$ is small enough. This further indicates that
\begin{equation}
TV(\pr^\sigma_{\mathcal{D}_Q},\pr^{\sigma'}_{\mathcal{D}_Q})\leq \sqrt{2}H^2(\pr^\sigma_{\mathcal{D}_Q},\pr^{\sigma'}_{\mathcal{D}_Q})\leq \frac{1}{2}.
\label{eq:theorem5part1H2toTV}
\end{equation}
For any empirical classifier $\hat{f}$, we have
$$
\begin{aligned}
    \mathcal{E}_{Q_\sigma}(\hat{f})+\mathcal{E}_{Q_{\sigma'}}(\hat{f})&=2\E_{Q_\sigma}[|\eta^{Q_\sigma}(X)-\frac{1}{2}|\mathbf{1}\{\hat{f}(X)\neq f^*_{Q_\sigma}(X)\}]\\
&+2\E_{Q_{\sigma'}}[|\eta^{Q_{\sigma'}}(X)-\frac{1}{2}|\mathbf{1}\{\hat{f}(X)\neq f^*_{Q_{\sigma'}}(X)\}]\\
&\geq 2\int_{B(x_k,r)}\frac{w}{\lambda[B(x_k,r)]}C_\beta r^\beta (\mathbf{1}\{\hat{f}(X)\neq f^*_{Q_\sigma}(X)\}+\mathbf{1}\{\hat{f}(X)\neq f^*_{Q_{\sigma'}}(X)\})dx\\
&= 2C_\beta w r^\beta.
\end{aligned}
$$
Combining this lower bound with \eqref{eq:theorem5part1H2toTV}, the Assouad's lemma shows that
$$\sup_{(Q,P)\in\Pi_{BA}^{NP}}\E\mathcal{E}_Q(\hat{f})\geq \sup_{\substack{(Q_\sigma,P)\\ \sigma\in\{1,-1\}^m}}\E\mathcal{E}_{Q_\sigma}(\hat{f})\geq \frac{1}{2}C_\beta mwr^\beta\gtrsim r^{\beta(1+\alpha)}\gtrsim \Delta^{\frac{1+\alpha}{\gamma}}\land n_Q^{-\frac{\beta(1+\alpha)}{2\beta+d}}.$$
\end{proof}
\section{Proofs in Parametric Classification}
\label{sec:appD}
We first list some definitions that are helpful for the proofs in this section. Denote the target and source distribution with certain linear coefficients $\beta,w$ as
$$Q_{\beta}:=\{Q:X\sim N(0,I_d),\eta^Q(x)=\sigma(\beta^Tx)\},$$
$$P_{w}:=\{Q:X\sim N(0,I_d),\eta^Q(x)=\sigma(w^Tx)\},$$
for any $\beta,w\in\mathbb{R}^d$.
The family of distribution pairs then becomes
$$\Pi^{LR}=\Pi^{LR}(s,\Pi,M)=\{(Q_{\beta},P_w):(\beta,w)\in\Theta(s,\Delta)\},$$
where $\Theta(s,\Delta)$ is defined in \eqref{eq:Theta}. Moreover, we abbreviate $\E_{\beta}$ as $\E_{(X, Y)\sim Q_\beta}$, the expectation with respect to a new observation drawn from the distribution $Q_{\beta}$. Similarly, define $\mathcal{E}_{Q_{\beta}}(f)$ as the excess risk with respect to the target distribution $Q_{\beta}$.

Before entering into the proof of the main theorems, we first show the following lemma that bounds the ambiguity level for any $(Q,P)\in\Pi^{LR}$:

\begin{lemma} Assume that $L\leq m \|\beta_Q\|_2\leq \|\beta_P\|_2\leq U$ for some constant $L,U>0$ and $0<m\leq 1$. We have that $\Pi^{LR}(s,\Delta)\subset \Pi(\alpha,C_\alpha,\gamma,C_\gamma,\varepsilon)$ where
$$
    \alpha=1,\ C_\alpha=\frac{16m}{\sqrt{2\pi}L}\lor 4,\ \gamma=1,\ C_\gamma=\frac{m}{\pi},$$$$\varepsilon(z;1,\frac{m}{\pi})=\left((\frac{16m}{\sqrt{2\pi}L}\lor 4)z^2\right)\land\frac{\sqrt{2}U}{m}\Delta^2.
$$
\label{lemma:logisticVerification}
\end{lemma}

\subsection{Proof of Lemma \ref{lemma:logisticVerification}}
\begin{proof}
For any $x\in B(0,1)$, we have
$$\sigma'(\beta_Q^Tx)=\sigma(\beta_Q^Tx)(1-\sigma(\beta_Q^Tx))\beta_Q,$$
$$\sigma'(\beta_P^Tx)=\sigma(\beta_Q^Tx)(1-\sigma(\beta_Q^Tx))\beta_P.$$
Define $e_Q$ and $e_P$ as the unit vectors that have the same directions with $\beta_Q$ and $\beta_P$, respectively. 

\vspace{\baselineskip}
\noindent\textbf{Verify Margin Assumption:} From the simple equality that $\log (1+\frac{4t}{1-2t})\leq 8t$ for any $t\in[0,\frac{1}{4}]$ and $\|\beta_Q\|\geq \frac{L}{m}$ we have 
$$
\begin{aligned}
 Q_X(0<|\sigma(\beta_Q^TX)-\frac{1}{2}|<t)&= Q_X(0<|\beta_Q^TX|\leq \log (1+\frac{4t}{1-2t}))\\
 &\leq Q_X(0<|\beta_Q^TX|\leq 8t)\\
 &\leq Q_X(0<|e_Q^TX|\leq \frac{8m}{L}t)\\
&= Q_X(|N(0,1)|\leq \frac{8m}{L}t)\\
 &\leq 16\frac{m}{L}t\phi(0)=\frac{16m}{\sqrt{2\pi}L}t,
\end{aligned}
$$
where $\phi(\cdot)$ is the density function of the univariate standard normal distribution $N(0,1)$.
For $t\in(\frac{1}{4},1]$, we trivially have $Q_X(0<|\sigma(\beta_Q^TX)-\frac{1}{2}|<t)\leq 4t$ since $Q_X$ is a probability measure. Therefore, Assumption holds with $\alpha=1,C_\alpha=\frac{16m}{\sqrt{2\pi}L}\lor 4$.

\vspace{\baselineskip}
\noindent\textbf{Verify Ambiguity Level:}
A trivial well-defined ambiguity level by the margin assumption is $$\varepsilon(z;\gamma,\frac{m}{\pi})=C_\alpha z^{1+\alpha}=(\frac{16m}{\sqrt{2\pi}L}\lor 4)z^2.$$ Hence, it suffices to show that a well-defined ambiguity level is $$\varepsilon(z;1,\frac{m}{\pi})=\frac{\sqrt{2}U}{m}\Delta^2.$$

Without loss of generality and due to the symmetry property of of $N(0,I_d)$, we could rotate $\beta_Q$ and $\beta_P$ at the same time so that
$$\beta_Q=(\|\beta_Q\|,0,0,\cdots, 0),\quad \beta_P=(\|\beta_P\|\cos \angle (\beta_Q,\beta_P),\|\beta_P\|\sin \angle (\beta_Q,\beta_P),0,0,\cdots,0).$$ Therefore, we assume that only the first coordinate of $\beta_Q$ and the first and second coordinates of $\beta_P$ can be non-zero.

For any $x\in\mathbb{R}^d$, we define $d_Q(x)$ and $d_P(x)$ as the angle between $(x_1,x_2,0,\cdots,0)$ and the normal planes of $\beta_Q$ and $\beta_P$, respectively. Note that $d_Q(x),d_P(x)\in[0,\pi/2]$. Define the area
$$D:=\{x\in \mathbb{R}^d:(\beta_Q^Tx)(\beta_P^Tx)<0\}\cup \{x\in \mathbb{R}^d:(\beta_Q^Tx)(\beta_P^Tx)\geq 0,d_P(x)<\frac{1}{2}d_Q(x)\}$$ as the region where $\eta^P$ does not give strong signal relative to $\eta^Q$. We see that $D$ is a cone centered at origin with angle between the two normal planes equal to $2\angle (\beta_Q,\beta_P)\leq 2\Delta$.
In addition, for any $x\in D$, the norm of projection of $x$ onto the normal plane of $\beta_Q$ is less than $$(x_1^2+x_2^2)^\frac{1}{2}\sin 2\angle ( \beta_Q,\beta_P)\leq 2(x_1^2+x_2^2)^\frac{1}{2}\Delta,$$ where the simple inequality $\sin2\angle ( \beta_Q,\beta_P)\leq 2\angle ( \beta_Q,\beta_P)$ is used.
Therefore,
$$|\eta^Q(x)-\frac{1}{2}|\leq 2\max\{\sigma'(z)\}_{z\in\mathbb{R}}\|\beta_Q\|(x_1^2+x_2^2)^\frac{1}{2}\Delta\leq \frac{U}{2m}(x_1^2+x_2^2)^\frac{1}{2}\Delta\quad \forall x\in D,$$ and
$$
\begin{aligned}
  \int_{D}|\eta^Q(X)-\frac{1}{2}|dQ_X&\leq \frac{U}{2m}\Delta \int_{D} (X_1^2+X_2^2)^\frac{1}{2}dQ_X\\
  &= \frac{U}{2m}\Delta (2\Delta)\int_{\mathbb{R}^d} (X_1^2+X_2^2)^\frac{1}{2}dQ_X\\
  &\leq \frac{U}{m}\Delta^2\int_{\mathbb{R}^d} (X_1^2+X_2^2)^\frac{1}{2}dQ_X\\
  &\leq\frac{\sqrt{2}U}{m}\Delta^2,
\end{aligned}
$$
since $\Delta\in [0,\frac{\pi}{2}]$ and $\E_{Q_X}[(X_1^2+X_2^2)^\frac{1}{2}]\leq \E_{Q_X}[X_1^2+X_2^2]^\frac{1}{2}=\sqrt{2}$.

On the other hand, if $x\notin D$, we have
$$
\begin{aligned}
|\eta^P(x)-\frac{1}{2}|&=|\sigma(\beta_P^Tx)-\frac{1}{2}|\\
&=|\sigma(\|\beta_P\|(x_1^2+x_2^2)^\frac{1}{2}\sin(d_P(x))-\frac{1}{2}|\\
&\geq |\sigma(m\|\beta_Q\|(x_1^2+x_2^2)^\frac{1}{2}\sin(\frac{1}{2}d_Q(x))-\frac{1}{2}|\\
&\geq |\sigma(\frac{m}{\pi}\|\beta_Q\|(x_1^2+x_2^2)^\frac{1}{2}\sin(d_Q(x))-\frac{1}{2}|\\
&\geq |\sigma(\frac{m}{\pi}\beta_Q^Tx)-\frac{1}{2}|\\
&\geq \frac{m}{\pi}|\sigma(\beta_Q^Tx)-\frac{1}{2}|,
\end{aligned}
$$
given the fact that $\sin(\frac{1}{2}t)\geq \frac{\sin t}{\pi}$ on $t\in[0,\frac{\pi}{2}]$ and that $|\sigma(t)-\frac{1}{2}|/x$ in monotone decreasing on $t\geq 0$.
Therefore, a well-defined ambiguity level is $\varepsilon(z;1,\frac{m}{\pi})=\frac{\sqrt{2}U}{m}\Delta^2$.
\end{proof}

\subsection{Proof of Theorem \ref{thm:logisticUpper}}

Define the logistic link function $\psi(x):=\log (1+e^x)$. It is obvious from straightforward calculation that $\psi'(x)=\frac{e^x}{e^x+1},\psi''(x)=\frac{e^x}{(e^x+1)^2}\leq \frac{1}{4}$. Prior to the proof, we present two lemmas that are helpful for deriving the key conditions in Theorem \ref{thm:moreGeneral}. They serve as the parametric bound of the estimator obtained by the loss function optimization in \eqref{eq:logisticModel}. See Section \ref{sec:logisticUpperLemmaProof} for the proof of the two lemmas.

\begin{lemma}[Parameter Bound for $\hat{\beta}_Q$]\label{lemma:LRUparaBoundQ}
Under the notations and conditions of Theorem \ref{thm:logisticUpper}, there exists some constant $C_Q, \kappa_Q>0$ such that
\begin{equation}
\pr_{\mathcal{D}}(\|\hat{\beta}_Q-\beta_Q\|\geq \kappa_Q \sqrt{\frac{s\log d}{n_Q}})\leq C_Q \left(\frac{s\log d}{n_Q}\land \frac{s\log d}{n_P}\right).
\label{eq:LRUparaBoundQ}
\end{equation}
\end{lemma}

\begin{lemma}[Parameter Bound for $\hat{\beta}_P$]\label{lemma:LRUparaBoundP}
Under the notations and conditions of Theorem \ref{thm:logisticUpper}, there exists some constant $C_P, \kappa_P>0$ such that
\begin{equation}
\pr_{\mathcal{D}}(\|\hat{\beta}_P-\beta_P\|\geq \kappa_P (\sqrt{\frac{s\log d}{n_P}}+\Delta))\leq C_P \left(\frac{s\log d}{n_Q}\land \frac{s\log d}{n_P}\right)
\label{eq:LRUparaBoundP}
\end{equation}
when $\Delta\leq c_P$ for some constant $c_P>0$.
\end{lemma}
\begin{proof}[Proof of Theorem \ref{thm:logisticUpper}]
By Lemma \ref{lemma:LRUparaBoundQ}, there exist constants $C_Q, \kappa_Q>0$ such that
$$\pr_{\mathcal{D}_Q}(E^c)\leq C_Q (\frac{s\log d}{n_Q}\land\frac{s\log d}{n_P}),$$
where we define the event $E:=\{\|\hat{\beta}_Q-\beta_Q\|\leq \kappa_Q \sqrt{\frac{s\log d}{n_Q}}\}$.
Define $\hat{e}$ as the unit vector with the same direction as $\hat{\beta}_Q-\beta_Q$. We define
$$\Omega^*:=\{x\in\Omega: |\hat{e}^T x|\leq \sqrt{2}\tau/(\frac{\kappa_Q}{2\sqrt{2}} \sqrt{\frac{s\log d}{n_Q}})\}.$$ Since $\hat{e}^T X\sim N(0,1)$ with respect to $Q_X$, the tail bound of normal distributions that
$$Q(\Omega^*)\geq 1- 2 \exp(-\left(\tau/(\frac{\kappa_Q}{2\sqrt{2}} \sqrt{\frac{s\log d}{n_Q}})\right)^2).$$
Since $\sigma(\cdot)$ is Lipschitz-continuous with the Lipschitz constant as $\frac{1}{4}$, on the event $E$, we further have
\begin{equation}
|\sigma(\beta_Q^Tx)-\sigma(\hat{\beta}_Q^Tx)|\leq \frac{1}{4}|(\beta_Q^T-\hat{\beta}_Q^T)x|\leq \|\hat{\beta}_Q-\beta_Q\|\cdot|\hat{e}^T x|\leq \tau
\label{eq:LRUetaQBound}
\end{equation}
for any $x\in\Omega^*$.

Now we could directly apply Theorem \ref{thm:moreGeneral} where the parameters are $$\tau=c_\tau \sqrt{\frac{s\log d}{n_Q}}\log(n_Q\lor n_P),$$
and
$$\delta_Q(n_Q,\tau)=2 \exp(-\left(\tau/(\frac{\kappa_Q}{2\sqrt{2}} \sqrt{\frac{s\log d}{n_Q}})\right)^2)+C_Q (\frac{s\log d}{n_Q}\land\frac{s\log d}{n_P})$$ with $c_\tau\geq\frac{\kappa_Q}{2\sqrt{2}}$. Note that $\delta_Q(n_Q,\tau)\lesssim \frac{s\log d}{n_Q}\land\frac{s\log d}{n_P}$. By Theorem \ref{thm:moreGeneral}, the TAB classifier $\hat{f}_{TAB}^{LR}$ satisfies
$$
\begin{aligned}
\sup_{(Q,P)\in\Pi^{LR}}\E_{(\mathcal{D}_Q,\mathcal{D}_P)}\mathcal{E}_Q(\hat{f}^{LR}_{TAB})&\lesssim \sup_{(Q,P)\in\Pi^{LR}}\E_{\mathcal{D}_P}\xi(\mathbf{1}\{\hat{\beta}_P^Tx\geq 0\};1,\frac{m}{\pi})\land \tau^2 + \varepsilon(2\tau;1,\frac{m}{\pi})+\delta_Q(n_Q,\tau)\\
&\lesssim\sup_{(Q,P)\in\Pi^{LR}}\E_{\mathcal{D}_P}\xi(\mathbf{1}\{\hat{\beta}_P^Tx\geq 0\};1,\frac{m}{\pi})\land (\log^2 (n_Q\lor n_P) \frac{s\log d}{n_Q})\\    
&+(\log^2 (n_Q\lor n_P) \frac{s\log d}{n_Q})\land \Delta^2\\
&+\frac{s\log d}{n_Q}\land \frac{s\log d}{n_P}.
\end{aligned}
$$
Therefore, it suffices to show that $$\E_{\mathcal{D}_P}\xi(\mathbf{1}\{\hat{\beta}_P^Tx\geq 0\};1,\frac{m}{\pi})\lesssim \frac{s\log d}{n_P}+\Delta^2$$ for any $(Q,P)\in\Pi^{LR}$ to prove Theorem \ref{thm:logisticUpper}. If $\Delta> c_p$ where $c_p$ is the constant constraint of $\Delta$ in Lemma \ref{lemma:LRUparaBoundP}, this bound is trivial as $\E_{\mathcal{D}_P}\xi(\mathbf{1}\{\hat{\beta}_P^Tx\geq 0\};1,\frac{m}{\pi})\leq 1$. Hence, we WLOG assume that $\Delta\leq c_p$ to apply Lemma \ref{lemma:LRUparaBoundP}. In addition, we may as well assume that $c_P$ is small enough such that $\kappa_Pc_P\leq \frac{L}{4}$.

By Lemma \ref{lemma:LRUparaBoundP}, there exist constants $C_P, \kappa_P>0$ such that
$$\pr_{\mathcal{D}_P}(E_2^c)\leq C_P (\frac{s\log d}{n_Q}\land\frac{s\log d}{n_P}),$$
where we define the event $E_2:=\{\|\hat{\beta}_P-\beta_P\|\leq \kappa_P (\sqrt{\frac{s\log d}{n_P}}+\Delta)\}$. Denote $\kappa_P (\sqrt{\frac{s\log d}{n_P}}+\Delta)$ by $\Lambda$ by simplicity. By the theorem setting, we could assume that $n_P$ is large enough such that $\kappa_P \sqrt{\frac{s\log d}{n_P}}\leq \frac{L}{4}$, so we have $\Lambda\leq \frac{L}{2}\leq \|\beta_P\|_2$. Therefore, on the event $E_2$ we have $\angle ( \hat{\beta}_P,\beta_P)\leq \frac{\pi}{2}$ and
$$\sin\angle ( \hat{\beta}_P,\beta_P)/\|\hat{\beta}_P-\beta_P\|\leq 1/\|\beta_P\|$$
by the law of sines. Hence,
$$\angle ( \hat{\beta}_P,\beta_P)\leq \frac{\pi}{2}\sin\angle ( \hat{\beta}_P,\beta_P)\leq \|\hat{\beta}_P-\beta_P\|/\|\beta_P\|\leq \frac{\Lambda}{L}.$$ 

Without loss of generality and due to the symmetry property of of $N(0,I_d)$, we could rotate $\beta_P$ and $\hat{\beta}_P$ at the same time so that
$$\beta_P=(\|\beta_P\|,0,0,\cdots, 0),\quad \hat{\beta}_P=(\|\hat{\beta}_P\|\cos \angle ( \hat{\beta}_P,\beta_P),\|\hat{\beta}_P\|\sin\angle ( \hat{\beta}_P,\beta_P),0,0,\cdots,0).$$ Therefore, we assume that only the first coordinate of $\beta_Q$ and the first and second coordinates of $\hat{\beta}_P$ can be non-zero.

Define $D:=\{x\in\mathbb{R}^d:(\beta_P^Tx)(\hat{\beta}_P^Tx)<0\}$, which is a cone centered at the origin with angle size bounded by $\frac{\Lambda}{L}$. For any $x\in D$, the norm of projection of $x$ onto the normal plane of $\beta_Q$ is less than $$(x_1^2+x_2^2)^\frac{1}{2}\sin \angle ( \hat{\beta}_P,\beta_P)\leq (x_1^2+x_2^2)^\frac{1}{2}\frac{\Lambda}{L},$$ where the simple inequality $\sin\angle ( \hat{\beta}_P,\beta_P)\leq \angle ( \hat{\beta}_P,\beta_P)$ is used.
Therefore,
$$|\eta^Q(x)-\frac{1}{2}|\leq \max\{\sigma'(z)\}_{z\in\mathbb{R}}\|\beta_P\|(x_1^2+x_2^2)^\frac{1}{2}\frac{\Lambda}{L}\leq \frac{U}{4}(x_1^2+x_2^2)^\frac{1}{2}\frac{\Lambda}{L}\quad \forall x\in D,$$ and
$$
\begin{aligned}
  \int_{D}|\eta^Q(X)-\frac{1}{2}|dQ_X&\leq \frac{U}{4}\frac{\Lambda}{L} \int_{D} (X_1^2+X_2^2)^\frac{1}{2}dQ_X\\
  &= \frac{U}{4}(\frac{\Lambda}{L})^2\int_{\mathbb{R}^d} (X_1^2+X_2^2)^\frac{1}{2}dQ_X\\
  &\leq\frac{\sqrt{2}U}{4L^2}\Lambda^2
\end{aligned}
$$
since $\frac{\Lambda}{L}\in [0,\frac{\pi}{2}]$ and $\E_{Q_X}[(X_1^2+X_2^2)^\frac{1}{2}]\leq \E_{Q_X}[X_1^2+X_2^2]^\frac{1}{2}=\sqrt{2}$. Hence, we obtain that
$$
\begin{aligned}
\xi(\mathbf{1}\{\hat{\beta}_P^Tx\geq 0\};1,\frac{m}{\pi})&=2\E_{X\sim N(0,I_d)}[|\eta^P-\frac{1}{2}|\mathbf{1}\{X\in D\}]\\
&\leq \frac{\sqrt{2}U}{2L^2}\Lambda^2\\
&=\frac{\sqrt{2}U}{2L^2}\kappa_P^2 (\sqrt{\frac{s\log d}{n_P}}+\Delta)^2.
\end{aligned}
$$
To summarize, we have the bounds$\xi(\mathbf{1}\{\hat{\beta}_P^Tx\geq 0\};1,\frac{m}{\pi})\leq \frac{\Lambda}{L}$ on the event $E_2$ and $\xi(\mathbf{1}\{\hat{\beta}_P^Tx\geq 0\};1,\frac{m}{\pi})\leq 1$ on the event $E_2^c$. As a result,
$$
\begin{aligned}
 &\E_{\mathcal{D}_P}\xi(\mathbf{1}\{\hat{\beta}_P^Tx\geq 0\};1,\frac{m}{\pi})\\
 = &\E_{\mathcal{D}_P}[\xi(\mathbf{1}\{\hat{\beta}_P^Tx\geq 0\};1,\frac{m}{\pi})|E_2]\pr_{\mathcal{D}_P}(E_2)+\E_{\mathcal{D}_P}[\xi(\mathbf{1}\{\hat{\beta}_P^Tx\geq 0\};1,\frac{m}{\pi})|E_2^c]\pr_{\mathcal{D}_P}(E_2^c)\\
 \leq & \frac{\sqrt{2}U}{2L^2}\kappa_P^2 (\sqrt{\frac{s\log d}{n_P}}+\Delta)^2\times 1 + 1\times C_P (\frac{s\log d}{n_Q}\land\frac{s\log d}{n_P})\\
 = & \frac{\sqrt{2}U}{2L^2}\kappa_P^2 (\sqrt{\frac{s\log d}{n_P}}+\Delta)^2+ C_P (\frac{s\log d}{n_Q}\land\frac{s\log d}{n_P})\\
 \lesssim &\frac{s\log d}{n_P}+\Delta^2.
\end{aligned}
$$
\end{proof}

\subsection{Proof of Theorem \ref{thm:logisticLower}}
Given an arbitrary classifier $\hat{f}$ based on $\mathcal{D}_Q$ and $\mathcal{D}_P$. Our goal is to show that
$$\sup_{(Q,P)\in\Pi^{LR}}\E\mathcal{E}_Q(\hat{f})\gtrsim
\left(\frac{s\log d}{n_P} + \Delta^2\right)\land \frac{s\log d}{n_Q}.$$
\begin{proof}[First part of Proof]
As part of the proof, we would like to first show that $$\sup_{(Q,P)\in\Pi^{LR}}\E\mathcal{E}_Q(\hat{f})\gtrsim
\frac{s\log d}{n_P}\land \frac{s\log d}{n_Q}.$$ The idea is to suppose $$\beta_P=\beta_Q=\beta,\quad P=Q=Q_\beta$$ which implies $\angle ( \beta_Q,\beta_P)=0$, and then reveal the lower bound provided by the combined sample $\mathcal{D}_Q\cup \mathcal{D}_P$. The proof is based on Fano's lemma on the vector angle.

We consider a class of $\beta\in\mathbb{R}^d$ such that
$$\|\beta\|_0\leq s,\ \beta_1=1,\ |\beta_j|\in\{0,C\sqrt{\log d/(n_Q\lor n_P)},-C\sqrt{\log d/(n_Q\lor n_P)}\}\  (\forall 2\leq j\leq d),$$ where the constant $C>0$ will be specified later. Denote such a class of $\beta$ as $\mathcal{H}$. The construction of the class of $\beta$ is inspired and similar to \cite{raskutti2009minimaxRO}. By Lemma 5 of \cite{raskutti2009minimaxRO}, we see that there exists a subset of $\mathcal{H}$, named $\tilde{\mathcal{H}}$, such that
$|\tilde{\mathcal{H}}|\geq \exp(\frac{s-1}{2}\log \frac{d-s}{(s-1)/2})$
and $$\sum_{j=2}^d\mathbf{1}\{\beta_j\neq\beta_j'\}\geq \frac{s-1}{2},\quad \forall \beta,\beta'\in\tilde{\mathcal{H}}.$$ We provide the following two observations on the elements of $\tilde{\mathcal{H}}:$
\begin{itemize}
    \item For any $\beta\in\tilde{\mathcal{H}}$, its norm is bounded by the following inequality:
\begin{equation}
1\leq \|\beta\|_2\leq 1+C\sqrt{s\log d/(n_Q\lor n_P)}\leq M,
\label{eq:paraLowerPart1NormBound}
\end{equation}
for some $M>0$ large enough since we supposed that $s\log d/(n_Q\lor n_P)\lesssim 1$.
\item For any $\beta\neq \beta'\in\tilde{\mathcal{H}}$, their angle could be bounded by
\begin{equation}
\begin{aligned}
\angle ( \beta,\beta')\geq \sin\angle ( \beta,\beta')\geq & \frac{(\sum_{j=2}^d\mathbf{1}\{\beta_j\neq\beta_j'\})^{\frac{1}{2}}C\sqrt{\log d/(n_Q\lor n_P)}}{\|\beta\|}\\
\geq &  \frac{C\sqrt{s\log d/(n_Q\lor n_P)}}{2M}.
\end{aligned}
\label{eq:paraLowerPart1AngleBound}
\end{equation}
This inequality holds since as long as $\beta_j\neq\beta_j'$ for some $2\leq j\leq d$, based on the construction of $\tilde{\mathcal{H}}$ we have $|\beta_j-\beta_j'|\geq C\sqrt{\log d/(n_Q\lor n_P)}$.
\end{itemize}

Next, we define the random variable $Z$ as the $\beta\in\tilde{\mathcal{H}}$ that attains the minimum of $\mathcal{E}_{Q_\beta}(\hat{f})$, i.e.
$$Z:=\argmin_{\beta\in\tilde{\mathcal{H}}}\mathcal{E}_{Q_\beta}(\hat{f}).$$
For any $\beta'\neq Z,\beta'\in\tilde{\mathcal{H}}$, by Lemma \ref{lemma:risk2angle}, we have that
$$\mathcal{E}_{Q_Z}(\hat{f})+\mathcal{E}_{Q_{\beta'}}(\hat{f})\geq \frac{\sigma(M)(1-\sigma(M))}{20\pi}\angle ( Z,\beta')^2.$$
by \eqref{eq:paraLowerPart1NormBound}. The choice of $Z$ further indicates that 
$$
\begin{aligned}
   \mathcal{E}_{Q_{\beta'}}(\hat{f})\geq & \frac{\sigma(M)(1-\sigma(M))}{40\pi}\angle ( Z,\beta')^2\\
   \geq & \frac{C\sigma(M)(1-\sigma(M))}{80\pi M} \sqrt{s\log d/(n_Q\lor n_P)}.
\end{aligned}
$$
We denote the constant term $\frac{C\sigma(M)(1-\sigma(M))}{80\pi M}$ by $C_M$ for simplicity. Hence,
$$\E_{\mathcal{D}_{\beta'}}\mathcal{E}_{Q_{\beta'}}(\hat{f})\geq C_M^2 s\log d/(n_Q\lor n_P)\times \pr_{\mathcal{D}_{\beta'}}(Z\neq \beta'),$$ and so
\begin{equation}
\begin{aligned}
  \sup_{(Q,P)\in\Pi}\E \mathcal{E}_{Q}(\hat{f})\geq &\max_{\beta_Q=\beta_P=\beta'\in\tilde{\mathcal{H}}}\E_{\mathcal{D}_{\beta'}}\mathcal{E}_{Q_{\beta'}}(\hat{f})\\
  \geq &C_M^2 s\log d/(n_Q\lor n_P) \max_{\beta'\in\tilde{\mathcal{H}}}\pr_{\mathcal{D}_{\beta'}}(Z\neq \beta')\\
  \geq & C_M^2 s\log d/(n_Q\lor n_P) \frac{1}{|\tilde{\mathcal{H}}|}\sum_{\beta'\in\tilde{\mathcal{H}}}\pr_{\mathcal{D}_{\beta'}}(Z\neq \beta').
\end{aligned}
\label{eq:paraLowerkeyBound}
\end{equation}
The remaining work is to lower bound $\frac{1}{|\tilde{\mathcal{H}}|}\sum_{\beta'\in\tilde{\mathcal{H}}}\pr_{\mathcal{D}_{\beta'}}(Z\neq \beta')$ by some positive constant. By Fano's lemma where the total sample size is $n_Q+n_P$ (Note that we suppose $P=Q$), we have
\begin{equation}
\begin{aligned}
\frac{1}{|\tilde{\mathcal{H}}|}\sum_{\beta'\in\tilde{\mathcal{H}}}\pr_{\mathcal{D}_{\beta'}}(Z\neq \beta')\geq& 1-\frac{\log 2+(n_Q+n_P)\max_{\beta\neq \beta'\in\tilde{\mathcal{H}}}KL(Q_\beta,Q_{\beta'})}{\log |\tilde{\mathcal{H}}|}\\
\geq & 1-\frac{\log 2+(n_Q+n_P)\max_{\beta\neq \beta'\in\tilde{\mathcal{H}}}KL(Q_\beta,Q_{\beta'})}{c\cdot s\log d}
\end{aligned}  
\end{equation}
since
$\log |\tilde{\mathcal{H}}|\geq \frac{s-1}{2}\log \frac{d-s}{(s-1)/2}\geq c\cdot s\log d$ for some $c>0$ small enough. For any $\beta\neq \beta'\in\tilde{\mathcal{H}}$, their Kullback-Leibler divergence satisfies
$$
\begin{aligned}
KL(Q_\beta,Q_{\beta'})&=\E_{Q_\beta}[\psi''(X^T(t\beta+(1-t)\beta'))(X^T(\beta-\beta'))^2]\\
&\leq C_\psi \E_{Q_\beta}[(X^T(\beta-\beta'))^2]\\
&\leq C_\psi \|\beta-\beta'\|^2\leq C_\psi C^2(s\log d)/ (n_Q\lor n_P)
\end{aligned}
$$
for the logistic regression link function $\psi(u)=\log(1+e^u)$, some constant $t\in[0,1]$ and $$C_\psi:=\|\psi''\|_\infty.$$ Therefore,
$$
\begin{aligned}
\frac{\log 2+(n_Q+n_P)\max_{\beta\neq \beta'\in\tilde{\mathcal{H}}}KL(Q_\beta,Q_{\beta'})}{c\cdot s\log d}\leq \frac{\log 2+C_\psi C^2(s\log d) \frac{(n_Q+n_P)}{n_Q\lor n_P}}{c\cdot s\log d}\leq \frac{1}{2}
\end{aligned}
$$
by choosing $C$ to be small enough. With such a choice of $C$, we see that \eqref{eq:paraLowerkeyBound} further reduces to
$$
\begin{aligned}
  \sup_{(Q,P)\in\Pi}\E \mathcal{E}_{Q}(\hat{f}) \geq & C_M^2 s\log d/(n_Q\lor n_P) \frac{1}{|\tilde{\mathcal{H}}|}\sum_{\beta'\in\tilde{\mathcal{H}}}\pr_{\mathcal{D}_{\beta'}}(Z\neq \beta')\\
  \geq & \frac{1}{2}C_M^2 s\log d/(n_Q\lor n_P)\\
  =&\frac{1}{2}C_M^2 \frac{s\log d}{n_Q}\land \frac{s\log d}{n_P},
\end{aligned}
$$
which is just our desired result.
\end{proof}
\begin{proof}[Second part of Proof] Next, we would like to show that
$$\sup_{(Q,P)\in\Pi^{LR}}\E\mathcal{E}_Q(\hat{f})\gtrsim
\Delta^2\land \frac{s\log d}{n_Q}.$$ 
    
First, we fix $\beta_P=h$ in the sense that
$$h_1=1,\quad h_j=0\quad  (\forall 2\leq j\leq d).$$ Next, we consider a class of $\beta$ such that
$$\|\beta\|_0\leq s,\ \beta_{1}=1,\ |\beta_{j}|\in\{0,C_\Delta\sqrt{\log d/n_Q}\land \frac{\sin\Delta}{\sqrt{s}},-C_\Delta\sqrt{\log d/n_Q}\land \frac{\sin\Delta}{\sqrt{s}}\}\  (\forall 2\leq j\leq d),$$ where the constant $C_\Delta\geq 0$ will be specified later. Denote such a class of $\beta$ as $\mathcal{H}_\Delta$.

We want to verify that for any $\beta\in \mathcal{H}_\Delta, (Q_\beta,P_{h})\in\Pi^{LR}$. Since it is obvious that $\|\beta\|_0\leq s$ by definition, it suffices to show that $\angle ( \beta,h)\leq \Delta$ from the calculation of the angle sine. It holds that
\begin{equation}
 \sin\angle ( \beta,h)\leq \frac{(\sum_{j=2}^d\mathbf{1}\{\beta\neq h\})^{\frac{1}{2}}(\frac{\sin\Delta}{\sqrt{s}})}{\|h\|}\leq \sin \Delta, 
 \label{eq:paraLowerAngleVerify1}
\end{equation}
since $\sum_{j=2}^d\mathbf{1}\{\beta\neq h\}\leq s,\|h\|=1$, and $|\beta-h|\leq \frac{\sin\Delta}{\sqrt{s}}$ if $\beta_j\neq h_j$. Also, the dot product satisfies
\begin{equation}
 \beta^Th=1>0, 
 \label{eq:paraLowerAngleVerify2}
\end{equation}
which indicates that $\angle ( \beta,h)<\frac{\pi}{2}$. Based on \eqref{eq:paraLowerAngleVerify1} and \eqref{eq:paraLowerAngleVerify2}, we conclude that $$\angle ( \beta,h)\leq \Delta,\quad \forall \beta\in\mathcal{H}_\Delta.$$

By following the same idea in the first part of the proof, Lemma 5 of \cite{raskutti2009minimaxRO} shows that there exists a subset of $\mathcal{H}_\Delta$, named $\tilde{\mathcal{H}}_\Delta$, such that
$|\tilde{\mathcal{H}}_\Delta|\geq \exp(\frac{s-1}{2}\log \frac{d-s}{(s-1)/2})$
and $$\sum_{j=2}^d\mathbf{1}\{\beta_j\neq\beta_j'\}\geq \frac{s-1}{2},\quad \forall \beta,\beta'\in\tilde{\mathcal{H}}_\Delta.$$ Following the same procedure of \eqref{eq:paraLowerPart1NormBound} and \eqref{eq:paraLowerPart1AngleBound}, we claim that
\begin{itemize}
    \item For any $\beta\in\tilde{\mathcal{H}}_\Delta$, its norm is bounded by the following inequality:
\begin{equation}
1\leq \|\beta\|_2\leq 1+C_\Delta\sin\Delta\leq 1+C_\Delta.
\label{eq:paraLowerPart2NormBound}
\end{equation}
\item For any $\beta\neq \beta'\in\tilde{\mathcal{H}}_\Delta$, their angle sine could be bounded by
\begin{equation}
\begin{aligned}
\angle ( \beta,\beta')\geq \sin\angle ( \beta,\beta')\geq & \frac{(\sum_{j=2}^d\mathbf{1}\{\beta_j\neq\beta_j'\})^{\frac{1}{2}}(C_\Delta\sqrt{\log d/n_Q}\land \frac{\sin\Delta}{\sqrt{s}})}{\|\beta\|}\\
\geq &  \frac{C_\Delta\sqrt{s\log d/n_Q}\land \sin\Delta}{2(1+C_\Delta)}.
\end{aligned}
\label{eq:paraLowerPart2AngleBound}
\end{equation}
This inequality holds since as long as $\beta_j\neq\beta_j'$ for some $2\leq j\leq d$, based on the construction of $\tilde{\mathcal{H}}_\Delta$ we have $|\beta_j-\beta_j'|\geq C_\Delta\sqrt{\log d/n_Q}\land \frac{\sin\Delta}{\sqrt{s}}$.
\end{itemize}

Define the random variable $Z$ as the $\beta\in\tilde{\mathcal{H}}_\Delta$ that attains the minimum of $\mathcal{E}_{Q_\beta}(\hat{f})$, i.e.
$$Z:=\argmin_{\beta\in\tilde{\mathcal{H}}_\Delta}\mathcal{E}_{Q_\beta}(\hat{f}).$$
For any $\beta'\neq Z,\beta'\in\tilde{\mathcal{H}}_\Delta$, by Lemma \ref{lemma:risk2angle}, we have that
$$\mathcal{E}_{Q_Z}(\hat{f})+\mathcal{E}_{Q_{\beta'}}(\hat{f})\geq \frac{\sigma(1+C_\Delta)(1-\sigma(1+C_\Delta))}{20\pi}\angle ( Z,\beta')^2.$$
by \eqref{eq:paraLowerPart2NormBound}. The choice of $Z$ further indicates that 
$$
\begin{aligned}
   \mathcal{E}_{Q_{\beta'}}(\hat{f})\geq & \frac{\sigma(1+C_\Delta)(1-\sigma(1+C_\Delta))}{40\pi}\angle ( Z,\beta')^2\\
   \geq & \frac{\sigma(1+C_\Delta)(1-\sigma(1+C_\Delta))}{80\pi (1+C_\Delta)} (C_\Delta^2s\log d/n_Q)\land \sin^2\Delta.
\end{aligned}
$$
We denote the constant term $\frac{\sigma(1+C_\Delta)(1-\sigma(1+C_\Delta))}{80\pi}$ by $C'_\Delta$ for simplicity. Hence,
$$\E_{\mathcal{D}_{\beta'}}\mathcal{E}_{Q_{\beta'}}(\hat{f})\geq C'_\Delta ((C_\Delta^2s\log d/n_Q)\land \sin^2\Delta)\times \pr_{\mathcal{D}_{\beta'}}(Z\neq \beta'),$$ and so
\begin{equation}
\begin{aligned}
  \sup_{(Q,P)\in\Pi}\E \mathcal{E}_{Q}(\hat{f})\geq &\max_{\beta'\in\tilde{\mathcal{H}}_\Delta,\beta_P=h}\E_{\mathcal{D}_{\beta'}}\mathcal{E}_{Q_{\beta'}}(\hat{f})\\
  \geq &C'_\Delta ((C_\Delta^2s\log d/n_Q)\land \sin^2\Delta)\max_{\beta'\in\tilde{\mathcal{H}}_\Delta}\pr_{\mathcal{D}_{\beta'}}(Z\neq \beta')\\
  \geq & C'_\Delta ((C_\Delta^2s\log d/n_Q)\land \sin^2\Delta) \frac{1}{|\tilde{\mathcal{H}}_\Delta|}\sum_{\beta'\in\tilde{\mathcal{H}}_\Delta}\pr_{\mathcal{D}_{\beta'}}(Z\neq \beta').
\end{aligned}
\label{eq:paraLowerkeyBound2}
\end{equation}
By Fano's lemma where the total sample size is $n_Q$, we have
\begin{equation}
\begin{aligned}
\frac{1}{|\tilde{\mathcal{H}}_\Delta|}\sum_{\beta'\in\tilde{\mathcal{H}}_\Delta}\pr_{\mathcal{D}_{\beta'}}(Z\neq \beta')\geq& 1-\frac{\log 2+n_Q\max_{\beta\neq \beta'\in\tilde{\mathcal{H}}_\Delta}KL(Q_\beta,Q_{\beta'})}{\log |\tilde{\mathcal{H}}|}\\
\geq & 1-\frac{\log 2+n_Q\max_{\beta\neq \beta'}KL(Q_\beta,Q_{\beta'})}{c\cdot s\log d}
\end{aligned}  
\end{equation}
since
$\log |\tilde{\mathcal{H}}_\Delta|\geq \frac{s-1}{2}\log \frac{d-s}{(s-1)/2}\geq c\cdot s\log d$ for some $c>0$ small enough. For any $\beta\neq \beta'\in\tilde{\mathcal{H}}_\Delta$, their Kullback-Leibler divergence satisfies
$$
\begin{aligned}
KL(Q_\beta,Q_{\beta'})&=\E_{Q_\beta}[\psi''(X^T(t\beta+(1-t)\beta'))(X^T(\beta-\beta'))^2]\\
&\leq C_\psi \E_{Q_\beta}[(X^T(\beta-\beta'))^2]\\
&\leq C_\psi \|\beta-\beta'\|^2\leq C_\psi C_\Delta^2\frac{s\log d}{n_Q}
\end{aligned}
$$
for the logistic regression link function $\psi(u)=\log(1+e^u)$ and some constants $t\in[0,1]$. Therefore,
$$
\begin{aligned}
\frac{\log 2+(n_Q+n_P)\max_{\beta\neq \beta'}KL(Q_\beta,Q_{\beta'})}{c\cdot s\log d}\leq \frac{\log 2+C_\psi C_\Delta^2 s\log d}{c\cdot s\log d}\leq \frac{1}{2}
\end{aligned}
$$
by choosing $C_\Delta$ to be small enough. With such a choice of $C_\Delta$, we see that \eqref{eq:paraLowerkeyBound} further reduces to
$$
\begin{aligned}
  \sup_{(Q,P)\in\Pi}\E \mathcal{E}_{Q}(\hat{f}) \geq & C'_\Delta ((C_\Delta^2s\log d/n_Q)\land \sin^2\Delta)\max_{\beta'\in\tilde{\mathcal{H}}_\Delta}\pr_{\mathcal{D}_{\beta'}}(Z\neq \beta')\\
  \geq & \frac{1}{2}C'_\Delta ((C_\Delta^2s\log d/n_Q)\land \sin^2\Delta)\\
  = &\frac{1}{2}C'_\Delta C_\Delta^2 \frac{s\log d}{n_Q} \land C'_\Delta \sin^2\Delta\\
  \geq &\left(\frac{1}{2}C'_\Delta C_\Delta^2 \frac{s\log d}{n_Q}\right) \land\left(  \frac{4C'_\Delta}{\pi^2}\Delta^2\right),
\end{aligned}
$$
which is just our desired result.

\end{proof}
Theorem \ref{thm:logisticLower} is achieved simply by combining the first and second parts of the proof above. 
\section{Auxiliary Results}
\label{sec:appAux}
\subsection{Proof of Proposition \ref{prop:sliceQ}}
In this subsection, $\varepsilon(z;\gamma,C_\gamma)\equiv 0$ in the parameter setting of $\Pi^{NP}$.
\begin{proof}[Proof of the first statement]
The direction of $\{Q:(Q,P)\in\Pi^{NP}\}\subset \Pi_Q^{NP}$ is obvious from the definitions.
Suppose $Q\in\Pi^{NP}_Q$. We always assume that $P_X=Q_X$ in this part. It is then obvious that $P_X=Q_X\in\mathcal{S}(\mu^+,\mu^-,c_\mu,r_\mu)$.

If $\beta=0$, then $\beta_P=0$, and $\eta^Q,\eta^P$ can be noncontinuous. Simply setting $$\eta^P=\frac{1}{2}+\frac{C_\gamma}{2^\gamma}\mathrm{sgn}\left(\eta^Q-\frac{1}{2}\right)$$ satisfies $(Q,P)\in\Pi^{NP}$ provided that $C_{\beta_P}\geq \frac{C_\gamma}{2^\gamma}$.

If $\beta>0$, and $\eta^Q$ does not hit $\frac{1}{2}$ on $\Omega$, then the continuity of $\eta^Q$ indicates that $\mathrm{sgn}(\eta^Q-\frac{1}{2})$ is a fixed constant. In this case, either setting $\eta^P\equiv 1$ or $\eta^P\equiv 0$ satisfies $(Q,P)\in\Pi^{NP}$.

At last, we prove the first statement under the conditions that $\beta>0$ and $\eta^Q$ hits $\frac{1}{2}$ on $\Omega$. If $\beta_P=0$, then we can again set $\eta^P=\frac{1}{2}+\frac{C_\gamma}{2^\gamma}\mathrm{sgn}\left(\eta^Q-\frac{1}{2}\right)$ to satisfy that $(Q,P)\in\Pi^{NP}$, so we consider $\beta_P>0$ below.

For any $x\in\Omega$, define $d(x)$ as the point $x'\in\Omega,\eta^Q(x')=\frac{1}{2}$ that is closest to $x$, i.e.
$$d(x):=\argmin_{x'\in\Omega,\eta^Q(x')=\frac{1}{2}}\|x-x'\|.$$
This function is well-defined since $\eta^Q$ hits $\frac{1}{2}$ on $\Omega$. Define $\eta^P$ as
$$\eta^P=\frac{1}{2}+C\mathrm{sgn}\left(\eta^Q-\frac{1}{2}\right)\|x-d(x)\|^{\beta_P},$$ where the constant $C>0$ will be specified later.

By the definition of $P$, we have $(\eta^P-\frac{1}{2})(\eta^Q-\frac{1}{2})\geq 0$. Moreover, we have 
$$
\begin{aligned}
 |\eta^P(x)-\frac{1}{2}|&\geq C\|x-d(x)\|^{\beta_P}\\
 &\geq C C_\beta^{-\frac{\beta_P}{\beta}}\left(C_\beta\|x-d(x)\|^\beta \right)^{\frac{\beta_P}{\beta}}\\
 &\geq C C_\beta^{-\frac{\beta_P}{\beta}}|\eta^Q(x)-\frac{1}{2}|^{\frac{\beta_P}{\beta}}\\
 &\geq C C_\beta^{-\frac{\beta_P}{\beta}} 2^{\beta_P/\beta-\gamma}|\eta^Q(x)-\frac{1}{2}|^\gamma\\
 &\geq C_\gamma |\eta^Q(x)-\frac{1}{2}|^\gamma.
\end{aligned}
$$ We conclude that $\Omega^+(\gamma,C_\gamma)=\Omega$ provided that $C\geq C_\gamma C_\beta^{\frac{\beta_P}{\beta}}2^{\gamma-\beta_P/\beta}$. Hence, it suffices to obtain that $$\eta^P\in \mathcal{H}(\beta_P,C_{\beta_P})$$ to show $(Q,P)\in\Pi^{NP}(\alpha,C_\alpha,\gamma,C_\gamma,0,\beta,\beta_P,C_\beta,C_{\beta_P},\mu^+,\mu^-,c_\mu,r_\mu)$. For any $x,x'\in\mathbb{R}^d$, we assume that $\|x-d(x)\|\geq \|x'-d(x')\|$ without loss of generality. If $(\eta^Q(x)-\frac{1}{2})(\eta^Q(x')-\frac{1}{2})\geq 0$, we have
$$
\begin{aligned}
|\eta^P(x)-\eta^P(x')|&\leq C\||x-d(x)\|^{\beta_P}-\|x'-d(x')\|^{\beta_P}|\\
&\leq C|\|x-d(x)||-||x'-d(x')\||^{\beta_P}\\
&\leq C\|x-x'\|^{\beta_P},
\end{aligned}
$$
so $\eta^P\in \mathcal{H}(\beta_P,C_{\beta_P})$ if $C_{\beta_P}\geq C$.
If $(\eta^Q(x)-\frac{1}{2})(\eta^Q(x')-\frac{1}{2})<0$, since $\eta^P$ is continuous, there exists $\lambda\in(0,1)$ such that $$\eta^Q(\lambda x+(1-\lambda)x')=\frac{1}{2}.$$ We have
$$ 
\begin{aligned}
|\eta^P(x)-\eta^P(x')|&\leq 2C\|x-d(x)\|^{\beta_P}\\
&\leq 2C\|x-(\lambda x+(1-\lambda)x')\|^{\beta_P}\\
&\leq 2C\|x-x'\|^{\beta_P},
\end{aligned}
$$
so $\eta^P\in \mathcal{H}(\beta_P,C_{\beta_P})$ if $C_{\beta_P}\geq 2C$. Combining all cases above, we conclude that for any $Q\in\Pi^{NP}_Q$, there exists some $P$ such that $(Q,P)\in\Pi^{NP}$ provided that $$C_{\beta_P}\geq \max\{C_\gamma 2^{-\gamma},2C_\gamma C_\beta^{\frac{\beta_P}{\beta}}2^{\gamma-\beta_P/\beta}\}.$$
\end{proof}
\begin{proof}[Proof of the second statement]
Let $\Omega=[0,1]^d$. Since $\beta_P>\gamma\beta\geq 0$, $\eta^P$ is continuous. 
Define
$$\eta^Q=\frac{1}{2}+C_\beta(x_1-\frac{1}{2})^\beta.$$ When $x_1=\frac{1}{2}$ for any $x\in\Omega$, it holds that $\eta^Q(x)=\frac{1}{2}$. Define $\Omega_1=\{x\in\Omega:x_1=\frac{1}{2}\}$. We claim $\eta^P=\frac{1}{2}$ on $x\in\Omega_1$. Otherwise, it is trivial to see that there will be a small ball in $\Omega$ on which $(\eta^P-\frac{1}{2})(\eta^Q-\frac{1}{2})<0$, which contradicts the fact that $\varepsilon(z;\gamma, C_\gamma)\equiv 0$.

Define the unit vector on the first coordinate as $e_1:=(1,0,0,\cdots,0)$. We see that for any $t\in[0,1]$,
$$\eta^Q(e_1)=\frac{1}{2},\quad \eta^Q(te_1)=\frac{1}{2}+C_\beta(t-\frac{1}{2})^\beta,$$
\begin{equation}
 \eta^P(e_1)=\frac{1}{2},\quad |\eta^P(te_1)-\frac{1}{2}|\geq C_\gamma|\eta^Q(te_1)-\frac{1}{2}|^\gamma\geq C_\gamma C^\gamma_\beta|t-\frac{1}{2}|^{\gamma\beta}.
 \label{eq:prop2secondContradict}
\end{equation}
However, since $\eta^P\in\mathcal{H}(\beta_P,C_{\beta_P})$, we have 
$$ |\eta^P(te_1)-\frac{1}{2}|=|\eta^P(te_1)-\eta^P(e_1)|\leq C_{\beta_P}t^{\beta_P},$$
which contradicts with \eqref{eq:prop2secondContradict} by choosing $t>0$ to be small enough.
\end{proof}

\subsection{Proof of Proposition \ref{prop:nonParaGeneralLowerBound01}}
\label{sec:prop5proof}
\begin{proof}
The proof is similar to the proof of Theorem 3.2 in \cite{ttcai2020classification}, which shows that
\begin{equation}
\inf_{\hat{f}}\sup_{\substack{(Q,P)\in\Pi^{NP}_0\\\Omega=\Omega_P,\beta_P=0,C_{\beta_P}=1}}\E\mathcal{E}_Q(\hat{f})\\
\gtrsim  n_P^{-\frac{\beta(1+\alpha)}{2\gamma\beta+d}}\land n_Q^{-\frac{\beta(1+\alpha)}{2\beta+d}}.
\label{eq:prop5part1raw}
\end{equation}
without imposing any smoothness condition on $\eta^P$. Our objective is to show that the minimax lower bound \eqref{eq:prop5part1raw} could apply to a broader class:
$$
\inf_{\hat{f}}\sup_{\substack{(Q,P)\in\Pi^{NP}_0\\\gamma\beta \geq \beta_P}}\E\mathcal{E}_Q(\hat{f})\\
\gtrsim  n_P^{-\frac{\beta(1+\alpha)}{2\gamma\beta+d}}\land n_Q^{-\frac{\beta(1+\alpha)}{2\beta+d}}
$$
with some modifications on the original proof. The only extra work is to verify that our construction $P_\sigma,\ \sigma\in\{0,1\}^m$ also satisfies the smoothness condition for $\eta^P$ with parameters $(\beta_P,C_{\beta_P})$, i.e., $(Q,P)\in\mathcal{H}(\beta,\beta_P,C_\beta,C_{\beta_P})$, given that $\beta_P<\gamma\beta$. 

Define the quantities
$$r=c_r (n_P^{-\frac{1}{2\gamma\beta+d}}\land n_Q^{-\frac{1}{2\beta+d}}),\quad w=c_wr^d,\quad m=\lfloor c_m r^{\alpha\beta-d}\rfloor,$$ where the constants $c_r,c_w,c_m$ will be specified later in the proof.

We consider a packing $\{x_k\}_{k=1,\cdots,m}$ with radius $2r$ in $[0,1]^d$. For an example of such construction, we divide $[0,1]^d$ into uniform small cubes with side length as $6r$, which forms a grid with at least $\lfloor (6r)^{-1}\rfloor^d$ small cubes. Since $m\ll\lfloor (6r)^{-1}\rfloor^d$, we suppose that $c_m$ is small enough such that $m<\lfloor (6r)^{-1}\rfloor^d$. Therefore, we could assign exactly one sphere with radius $2r$ in one cube without intersection, which forms a packing of $\{x_k\}_{k=1,\cdots,m}$ with radius $2r$ in $[0,1]^d$. Also, define $B_c$ as the compliment of these $m$ balls, i.e.
$B_c:=[0,1]^d/\left(\bigcup_{k=1}^m B(x_k,2r)\right)$.

For any $\sigma\in\{1,-1\}^m$, we consider the regression functions $\eta^Q_\sigma(x)$ and $\eta^P_\sigma(x)$ defined as follows:
$$
  \eta^Q_\sigma(x)=
\begin{cases}
\frac{1}{2}+\sigma_k C_\beta r^{\beta}g^{\beta}(\frac{\|x-x_k\|}{r})\quad &\text{if $x\in B(x_k,2r)$ for some $k=1,\cdots,m$,}\\
\frac{1}{2}\quad &\text{otherwise,}
\end{cases}
$$

$$
\eta^P_\sigma(x)=
\begin{cases}
\frac{1}{2}+\sigma_k (C_{\beta_P}\lor C_\gamma C_\beta^\gamma) r^{\gamma\beta}g^{\gamma\beta}(\frac{\|x-x_k\|}{r})\quad &\text{if $x\in B(x_k,2r)$ for some $k=1,\cdots,m$,}\\
\frac{1}{2}\quad &\text{otherwise,}
\end{cases}
$$
where the function $g(\cdot)$ is defined as $g(x)=\min\{1,2-x\}$ on $x\in[0,2]$. 

The construction of the marginal distributions $Q_{\sigma,X}$ and $P_{\sigma,X}$ is as follows. Let $Q_{\sigma,X}=P_{\sigma,X}$, both of which have the density function $\mu(\cdot)$ defined as
$$\mu(x)=
\begin{cases}
\frac{w}{\lambda[B(x_k,r)]}\quad &\text{if $x\in B(x_k,r)$ for some $k=1,\cdots,m$,}\\
\frac{1-mw}{1-m\lambda[B(x_k,2r)]}\quad &\text{if $x\in B_c$,}\\
0\quad &\text{otherwise.}
\end{cases}
$$

Given the the construction of $(Q_\sigma,P_\sigma)$, we next verify that $(Q_\sigma,P_\sigma)$ belongs to $\Pi^{NP}_0$ for any $\sigma\in\{1,-1\}^m$. The fact that $\Omega^+(\gamma,C_\gamma)=\Omega$ is trivial by the construction of the regression functions $\eta^Q_\sigma(x)$ and $\eta^P_\sigma(x)$, so we omit its verification.

\vspace{\baselineskip}
\noindent\textbf{Verify Margin Assumption:} We have that
$$
\begin{aligned}
Q_\sigma(0<|\eta_\sigma^Q-\frac{1}{2}|<t)&=m Q_\sigma(0<C_\beta r^\beta g^\beta(\frac{\|X-x_1\|}{r})\leq t)\\
&=m Q_\sigma(0<g(\frac{\|X-x_1\|}{r})\leq (\frac{t}{C_\beta r^\beta})^{\frac{1}{\beta}})\\
&=mw\mathbf{1}\{t\geq C_\beta r^\beta\}\\
&\leq c_mc_w r^{\alpha\beta}\mathbf{1}\{t\geq C_\beta r^\beta\}\\
&\leq C_\alpha t^\alpha
\end{aligned}
$$
given that $c_m$ is small enough. Therefore, $Q_\sigma\in\mathcal{M}(\alpha,C_\alpha)$.

\vspace{\baselineskip}
\noindent\textbf{Verify Smoothness Assumption:}
It is easy to see that for any $a,b\in[0,2]$ we have
$$
|g^{\beta}(a)-g^{\beta}(b)|\leq |a-b|^\beta
$$
and
$$
|g^{\gamma\beta}(a)-g^{\gamma\beta}(b)|\leq
\begin{cases}
|a-b|^{\gamma\beta}\quad &\gamma\beta\leq 1,\\
1-(1-|a-b|)^{\gamma\beta}\leq \gamma\beta |a-b| \quad &\gamma\beta>1.
\end{cases}
$$
Thus, for any $x,x'\in B(x_k,2r)$, we obtain from the triangular inequality $\|x-x_k\|-\|x'-x_k\|\leq \|x-x'\|$ that
\begin{equation}
    |r^{\beta}g^{\beta}(\|x-x_k\|/r))-r^{\beta}g^{\beta}(\|x'-x_k\|/r))|\leq\|x-x'\|^{\beta}
    \label{eq:prop5smooth1}
\end{equation}
and
\begin{equation}
|r^{\gamma\beta}g^{\gamma\beta}(\|x-x_k\|/r))-r^{\gamma\beta}g^{\gamma\beta}(\|x'-x_k\|/r))|
\leq
\begin{cases}
\|x-x'\|^{\gamma\beta}\quad &\gamma\beta\leq 1,\\
\gamma\beta\cdot r^{\gamma\beta} \|x-x'\| \quad &\gamma\beta>1.
\end{cases} 
\label{eq:prop5smooth2}
\end{equation}
Suppose that $c_r$ is small enough such that for any $n_Q,n_P\geq 1$, we have $$\max\{\gamma\beta\cdot r^{\gamma\beta}(4r)^{1-\beta_P},(4r)^{\gamma\beta-\beta_P}\}\leq 1,$$ then we further deduces that
\begin{equation}
\begin{aligned}
&\|x-x'\|^{\gamma\beta}\leq \|x-x'\|^{\beta_P} (4r)^{\gamma\beta-\beta_P}\leq \|x-x'\|^{\beta_P}\\
& \gamma\beta\cdot r^{\gamma\beta} \|x-x'\|\leq \gamma\beta\cdot r^{\gamma\beta}(4r)^{1-\beta_P}\|x-x'\|^{\beta_P}\leq \|x-x'\|^{\beta_P}.
\end{aligned}
\label{eq:prop5smooth3}
\end{equation}
On one hand, by the definition of $\eta^Q_\sigma$ and \eqref{eq:prop5smooth1}, we see that $$|\eta^Q_\sigma(x)-\eta^Q_\sigma(x')|\leq C_{\beta}\|x-x'\|^{\beta},\quad \forall x,x\in[0,1]^d.$$
 One the other hand, by the definition of $\eta^P_\sigma$, \eqref{eq:prop5smooth2} and \eqref{eq:prop5smooth3}, we see that
$$|\eta^P_\sigma(x)-\eta^P_\sigma(x')|\leq C_{\beta_P}\|x-x'\|^{\gamma\beta}\leq C_{\beta_P}\|x-x'\|^{\beta_P},\quad \forall x,x\in[0,1]^d.$$
Hence, $(Q,P)\in\mathcal{H}(\beta,\beta_P,C_\beta,C_{\beta_P})$.

\vspace{\baselineskip}
\noindent\textbf{Verify Strong Density Condition:}
If $x\in B(x_k,r)$ for some $k=1,\cdots,m$, we have
$$\mu(x)=\frac{w}{\lambda[B(x_k,r)]}=c_w/\pi_d.$$
If $x\in B_c$, we have
$$\mu(x)=\frac{1-c_w \lfloor c_m r^{\alpha\beta-d}\rfloor r^d}{1-2^d\pi_d\lfloor c_m r^{\alpha\beta-d}\rfloor r^d}\overset{n_Q\rightarrow \infty}{\longrightarrow}1.$$
Therefore, we could set $c_w\in[\pi_d\mu^-,\pi_d\mu^+]$ to satisfy the condition $(Q_\sigma,P_\sigma)\in\mathcal{S}(\mu^+,\mu^-,c_\mu,r_\mu)$.

Putting all verification steps above together, we conclude that $(Q_\sigma,P_\sigma)\in\Pi^{NP}_0$. We finish the proof of this part by applying Assouad's lemma to $(Q_\sigma,P_\sigma),\ \forall \sigma\in\{1,-1\}^m$.

If $\sigma,\sigma'\in\{1,-1\}^m$ differ only at one coordinate, i.e.
$$\sigma_k=-\sigma'_k,\quad \sigma_l=\sigma'_l\ (\forall l\neq k),$$ we have the Hellinger distance bound as
$$
\begin{aligned}
H^2(Q_\sigma,Q_{\sigma'})&=\frac{1}{2}\int \left(\sqrt{\eta^Q_\sigma(X)}-\sqrt{\eta^Q_{\sigma'}(X)}\right)^2+\left(\sqrt{1-\eta^Q_\sigma(X)}-\sqrt{1-\eta^Q_{\sigma'}(X)}\right)^2 dQ_X\\
&=\int_{B(x_k,r)}\frac{w}{\lambda[B(x_k,r)]}\left(\sqrt{\frac{1}{2}+C_\beta r^\beta}-\sqrt{\frac{1}{2}-C_\beta r^\beta}\right)^2 dx\\
&= \frac{1}{2}w(1-\sqrt{1-2C^2_\beta r^{2\beta}})\\
&\leq C^2_\beta wr^{2\beta},
\end{aligned}
$$
$$
\begin{aligned}
H^2(P_\sigma,P_{\sigma'})&=\frac{1}{2}\int \left(\sqrt{\eta^P_\sigma(X)}-\sqrt{\eta^P_{\sigma'}(X)}\right)^2+\left(\sqrt{1-\eta^P_\sigma(X)}-\sqrt{1-\eta^P_{\sigma'}(X)}\right)^2 dP_X\\
&=\int_{B(x_k,r)}\frac{w}{\lambda[B(x_k,r)]}\left(\sqrt{\frac{1}{2}+(C_{\beta_P}\lor C_\gamma C_\beta^\gamma) r^{\gamma\beta}}-\sqrt{\frac{1}{2}-(C_{\beta_P}\lor C_\gamma C_\beta^\gamma) r^{\gamma\beta}}\right)^2 dx\\
&= \frac{1}{2}w(1-\sqrt{1-2(C^2_{\beta_P}\lor C^2_\gamma C_\beta^{2\gamma}) r^{2\gamma\beta}})\\
&\leq (C^2_{\beta_P}\lor C^2_\gamma C_\beta^{2\gamma}) r^{2\gamma\beta}.
\end{aligned}
$$
Recall that $r=c_r (n_P^{-\frac{1}{2\gamma\beta+d}}\land n_Q^{-\frac{1}{2\beta+d}})$. By the property of Hellinger distance, we have
$$
\begin{aligned}
H^2(\pr^\sigma_{\mathcal{D}_Q}\times\pr^\sigma_{\mathcal{D}_P},\pr^{\sigma'}_{\mathcal{D}_Q}\times\pr^{\sigma'}_{\mathcal{D}_P})\leq &n_Q H^2(Q_\sigma,Q_{\sigma'})+n_P H^2(P_\sigma,P_{\sigma'})\\
\leq & C^2_\beta w n_Q r^{2\beta}+(C^2_{\beta_P}\lor C^2_\gamma C_\beta^{2\gamma}) n_P r^{2\gamma\beta}\\
\leq & C^2_\beta c_r^{2\beta+d}c_w+(C^2_{\beta_P}\lor C^2_\gamma C_\beta^{2\gamma})c_r^{2\gamma\beta+d}c_w\\
\leq & \frac{\sqrt{2}}{4}
\end{aligned}
$$
provided that $c_r$ is small enough. This further indicates that
\begin{equation}
TV(\pr^\sigma_{\mathcal{D}_Q}\times\pr^\sigma_{\mathcal{D}_P},\pr^{\sigma'}_{\mathcal{D}_Q}\times\pr^{\sigma'}_{\mathcal{D}_P})\leq \sqrt{2}H^2(\pr^\sigma_{\mathcal{D}_Q}\times\pr^\sigma_{\mathcal{D}_P},\pr^{\sigma'}_{\mathcal{D}_Q}\times\pr^{\sigma'}_{\mathcal{D}_P})\leq \frac{1}{2}.
\label{eq:prop5part1H2toTV}
\end{equation}
For any empirical classifier $\hat{f}$, we have
$$
\begin{aligned}
    \mathcal{E}_{Q_\sigma}(\hat{f})+\mathcal{E}_{Q_{\sigma'}}(\hat{f})&=2\E_{Q_\sigma}[|\eta^{Q_\sigma}(X)-\frac{1}{2}|\mathbf{1}\{\hat{f}(X)\neq f^*_{Q_\sigma}(X)\}]\\
&+2\E_{Q_{\sigma'}}[|\eta^{Q_{\sigma'}}(X)-\frac{1}{2}|\mathbf{1}\{\hat{f}(X)\neq f^*_{Q_{\sigma'}}(X)\}]\\
&\geq 2\int_{B(x_k,r)}\frac{w}{\lambda[B(x_k,r)]}C_\beta r^\beta (\mathbf{1}\{\hat{f}(X)\neq f^*_{Q_\sigma}(X)\}+\mathbf{1}\{\hat{f}(X)\neq f^*_{Q_{\sigma'}}(X)\})dx\\
&= 2C_\beta w r^\beta.
\end{aligned}
$$
Combining this lower bound with \eqref{eq:prop5part1H2toTV}, the Assouad's lemma shows that
$$\sup_{\substack{(Q,P)\in\Pi^{NP}_0\\\Omega=\Omega_P,\gamma\beta \geq \beta_P}}\E\mathcal{E}_Q(\hat{f})\geq \sup_{\substack{(Q_\sigma,P_\sigma)\\ \sigma\in\{1,-1\}^m}}\E\mathcal{E}_{Q_\sigma}(\hat{f})\geq \frac{1}{2}C_\beta mwr^\beta\gtrsim r^{\beta(1+\alpha)}\gtrsim  n_P^{-\frac{\beta(1+\alpha)}{2\gamma\beta+d}}\land n_Q^{-\frac{\beta(1+\alpha)}{2\beta+d}}.$$
\end{proof}

\subsection{Proof of Proposition \ref{prop:nonParaGeneralLowerBound02}}
\label{sec:prop6proof}
\begin{proof}

Our objective is to show that the minimax lower bound satisfies
$$
\inf_{\hat{f}}\sup_{\substack{(Q,P)\in\Pi^{NP}_0\\\gamma\beta < \beta_P}}\E\mathcal{E}_Q(\hat{f})\\
\gtrsim  n_P^{-\frac{\beta_P(1+\alpha)/\gamma}{2\beta_P+d}}\land n_Q^{-\frac{\beta_P(1+\alpha)/\gamma}{2\beta_P/\gamma+d}}.
$$
Define the quantities
$$r_Q=c_r (n_P^{-\frac{\beta_P/\gamma\beta}{2\beta_P+d}}\land n_Q^{-\frac{1}{2\beta+d\gamma\beta/\beta_P}}),\quad r_P=r_Q^{\gamma\beta/\beta_P}=c_r^{\gamma\beta/\beta_P} (n_P^{-\frac{1}{2\beta_P+d}}\land n_Q^{-\frac{1}{2\beta_P/\gamma+d}}), $$
$$w=c_wr_Q^d,\quad m_0=\lfloor c_{m_0}\frac{r_P^d}{r_Q^d} \rfloor, \quad m=\lfloor c_m r_P^{\alpha\beta_P/\gamma-d}\rfloor,$$ where the constants $c_r,c_w,c_{m_0},c_m$ will be specified later in the proof.

Similar to the proof of the case of $\gamma\beta\geq \beta_P$, we consider a packing $\{x_k\}_{k=1,\cdots,m}$ with radius $2r_P$ in $[0,1]^d$. Define $B_c$ as the compliment of these $m$ balls, i.e., $$B_c:=[0,1]^d/\left(\bigcup_{k=1}^m B(x_k,2r_P)\right).$$

Next, we further consider a packing $\{x_{k,l}\}_{l=1,\cdots,m_0}$ for any $k=1,\cdots,m$ with radius $2r_Q$. By scaling the radius of the balls, the feasibility of our considered packing reduces to finding a packing $\{x_{l}\}_{l=1,\cdots,m_0}$ of $B(0,1)$ with radius $r_Q/r_P$. We denote $r_Q/r_P$ by $R$, and $m_0$ then becomes $\lfloor c_{m_0}R^{-d} \rfloor$.

We provide an example of the reduced form of packing as follows. Consider the inscribed cube of $B(0,1)$ with diagonal length $2$ and side length $\frac{2}{\sqrt{d}}$. We divide this inscribed cube into uniform small cubes with side length as $6R$, which forms a grid with at least $$\lfloor \frac{2}{\sqrt{d}}(6R)^{-1}\rfloor^d\geq \lfloor c\frac{2}{\sqrt{d}}(6R)^{-d}\rfloor$$ small cubes for some constant $c>0$. Therefore, provided that $c_{m_0}\leq \frac{2c}{6^d\sqrt{d}}$ is small enough, we could assign exactly one sphere with radius $2R$ in one cube without intersection, which forms a packing $\{x_l\}_{l=1,\cdots,m_0}$ with radius $2R$. Hence, we can find a packing $\{x_{k,l}\}_{l=1,\cdots,m_0}$ for any $k=1,\cdots,m$ with radius $2r_Q$.

For any $\sigma\in\{1,-1\}^m$, we consider the regression functions $\eta^Q_\sigma(x)$ and $\eta^P_\sigma(x)$ defined as follows:
\begin{equation}
\eta^Q_\sigma(x)=
\begin{cases}
\frac{1}{2}\left(1+\sigma_k (C_\beta \land (\frac{C_{\beta_P}}{C_\gamma})^{\frac{1}{\gamma}}) r_Q^{\beta}g^{\beta}(\frac{\|x-x_{k,l}\|}{r_Q})\right)\quad &\text{if $x\in B(x_{k,l},2r_Q)$ for some $k,l$,}\\
\frac{1}{2}\quad &\text{otherwise,}
\end{cases}
\label{eq:prop5part2etaQ}
\end{equation}

\begin{equation}
\eta^P_\sigma(x)=
\begin{cases}
\frac{1}{2}\left(1+\sigma_k  C_{\beta_P} r_P^{\beta_P}g^{\beta_P}(\frac{\|x-x_k\|}{r_P})\right)\quad &\text{if $x\in B(x_k,2r_P)$ for some $k$,}\\
\frac{1}{2}\quad &\text{otherwise,}
\end{cases}
\label{eq:prop5part2etaP}
\end{equation}
where the function $g(\cdot)$ is defined as $g(x)=\min\{1,2-x\}$ on $x\in[0,2]$. 

The construction of the marginal distributions $Q_{\sigma,X}$ and $P_{\sigma,X}$ is as follows. Let $Q_{\sigma,X}=P_{\sigma,X}$, both of which have the density function $\mu(\cdot)$ defined as
$$\mu(x)=
\begin{cases}
\frac{w}{\lambda[B(x_k,r_Q)]}\quad &\text{if $x\in B(x_{k,l},r)$ for some $k,l$,}\\
\frac{1-mm_0w}{1-m\lambda[B(x_k,2r_P)]}\quad &\text{if $x\in B_c$,}\\
0\quad &\text{otherwise.}
\end{cases}
$$

Given the the construction of $(Q_\sigma,P_\sigma)$, we next verify that $(Q_\sigma,P_\sigma)$ belongs to $\Pi^{NP}_0$ for any $\sigma\in\{1,-1\}^m$.

\vspace{\baselineskip}
\noindent\textbf{Verify Margin Assumption:} We have that
$$
\begin{aligned}
Q_\sigma(0<|\eta_\sigma^Q-\frac{1}{2}|<t)&=mm_0 Q_\sigma(0<(C_\beta \land (\frac{C_{\beta_P}}{C_\gamma})^{\frac{1}{\gamma}}) r_Q^\beta g^\beta(\frac{\|X-x_{1,1}\|}{r_Q})\leq 2t)\\
&=mm_0 Q_\sigma(0<g(\frac{\|X-x_1\|}{r_Q})\leq (\frac{2t}{(C_\beta \land (\frac{C_{\beta_P}}{C_\gamma})^{\frac{1}{\gamma}}) r_Q^\beta})^{\frac{1}{\beta}})\\
&=mm_0 w\mathbf{1}\{t\geq \frac{1}{2}(C_\beta \land (\frac{C_{\beta_P}}{C_\gamma})^{\frac{1}{\gamma}}) r_Q^\beta\}\\
&\leq c_m c_{m_0} c_w r_P^{\alpha\beta_P/\gamma}\mathbf{1}\{t\geq \frac{1}{2}(C_\beta \land (\frac{C_{\beta_P}}{C_\gamma})^{\frac{1}{\gamma}}) r_Q^\beta\}\\
&= c_m c_{m_0} c_w r_Q^{\alpha\beta}\mathbf{1}\{t\geq \frac{1}{2}(C_\beta \land (\frac{C_{\beta_P}}{C_\gamma})^{\frac{1}{\gamma}}) r_Q^\beta\}\\
&\leq C_\alpha t^\alpha
\end{aligned}
$$
given that $c_m$ is small enough. Therefore, $Q_\sigma\in\mathcal{M}(\alpha,C_\alpha)$.

\vspace{\baselineskip}
\noindent\textbf{Verify Smoothness Assumption:}
It is easy to see that for any $a,b\in[0,2]$ we have
$$
|g^{\beta}(a)-g^{\beta}(b)|\leq |a-b|^\beta,\quad |g^{\beta_P}(a)-g^{\beta_P}(b)|\leq |a-b|^{\beta_P}.
$$
Thus, for any $x,x'\in B(x_{k,l},2r_Q)$, we obtain from the triangular inequality $\|x-x_k\|-\|x'-x_k\|\leq \|x-x'\|$ that
\begin{equation}
    |r_Q^{\beta}g^{\beta}(\|x-x_{k,l}\|/r_Q))-r_Q^{\beta}g^{\beta}(\|x'-x_{k,l}\|/r_Q))|\leq\|x-x'\|^{\beta}.
    \label{eq:prop5smooth4}
\end{equation}
By the definition of $\eta^Q_\sigma$ and \eqref{eq:prop5smooth4}, we see that $$|\eta^Q_\sigma(x)-\eta^Q_\sigma(x')|\leq C_{\beta}\|x-x'\|^{\beta},\quad \forall x,x\in[0,1]^d.$$
In addition, for any $x,x'\in B(x_{k},2r_P)$, we obtain from the triangular inequality $\|x-x_k\|-\|x'-x_k\|\leq \|x-x'\|$ that
\begin{equation}
    |r_P^{\beta_P}g^{\beta_P}(\|x-x_{k}\|/r_P))-r_P^{\beta_P}g^{\beta_P}(\|x'-x_{k}\|/r_P))|\leq\|x-x'\|^{\beta_P}.
    \label{eq:prop5smooth5}
\end{equation}
By the definition of $\eta^P_\sigma$ and \eqref{eq:prop5smooth5}, we see that $$|\eta^P_\sigma(x)-\eta^P_\sigma(x')|\leq C_{\beta_P}\|x-x'\|^{\beta_P},\quad \forall x,x\in[0,1]^d.$$
Hence, $(Q,P)\in\mathcal{H}(\beta,\beta_P,C_\beta,C_{\beta_P})$.

\vspace{\baselineskip}
\noindent\textbf{Verify Strong Density Condition:}
If $x\in B(x_{k,l},r)$ for some $k=1,\cdots,m$ and $l=1,\cdots,m_0$, we have
$$\mu(x)=\frac{w}{\lambda[B(x_{k,l},r)]}=c_w/\pi_d.$$
If $x\in B_c$, we have
$$
\begin{aligned}
    \mu(x)=&\frac{1-c_w \lfloor c_{m_0}\frac{r_P^d}{r_Q^d} \rfloor \lfloor c_m r_P^{\alpha\beta_P/\gamma-d}\rfloor r_Q^d}{1-2^d\pi_d\lfloor c_m r_P^{\alpha\beta_P/\gamma-d}\rfloor r_P^d}\\
    =&\frac{1-c_w \lfloor c_{m_0}r_Q^{d\gamma\beta/\beta_P-d} \rfloor \lfloor c_m r_Q^{\alpha\beta-d\gamma\beta/\beta_P}\rfloor r_Q^d}{1-2^d\pi_d\lfloor c_m r_P^{\alpha\beta-d\gamma\beta/\beta_P}\rfloor r_Q^{d\gamma\beta/\beta_P}}\overset{n_Q\rightarrow \infty}{\longrightarrow}1.
\end{aligned}
$$
Therefore, we could set $c_w\in[\pi_d\mu^-,\pi_d\mu^+]$ to satisfy the condition $(Q_\sigma,P_\sigma)\in\mathcal{S}(\mu^+,\mu^-,c_\mu,r_\mu)$.

\vspace{\baselineskip}
\noindent\textbf{Verify Signal Transfer Set:}
We show that $\Omega^+(\gamma,C_\gamma)=\Omega$ for our construction of $(Q_\sigma,P_\sigma)$. By the construction of the marginal distribution pair $(Q^\sigma_X,P^\sigma_X)$ and the definition of $\eta^Q_\sigma$ and $\eta^P_\sigma$ in \eqref{eq:prop5part2etaQ} and \eqref{eq:prop5part2etaP}, for any $x\in\Omega=B_c\cup \left(\bigcup_{k=1}^m B(x_{k,l},r_Q)\right)$,
\begin{itemize}
    \item If $x\in B(x_{k,l},r)$ for some $k=1,\cdots,m$ and $l=1,\cdots,m_0$, we have
    $$\eta^Q_\sigma(x)-\frac{1}{2}=\sigma_k(C_\beta \land (\frac{C_{\beta_P}}{C_\gamma})^{\frac{1}{\gamma}}) r_Q^{\beta},\quad \eta^P_\sigma(x)-\frac{1}{2}=\sigma_kC_{\beta_P} r_P^{\beta_P}=\sigma_k C_{\beta_P} r_Q^{\gamma\beta}.$$ Therefore,
    $$\mathrm{sgn}\left(\eta^Q(x)-\frac{1}{2}\right)\times(\eta^P(x)-\frac{1}{2})\geq C_\gamma|\eta^Q(x)-\frac{1}{2}|^\gamma.$$
    \item If $x\in B_c$, we have $\eta^Q_\sigma(x)=\eta^P_\sigma(x)=\frac{1}{2}$. Therefore,
    $$\mathrm{sgn}\left(\eta^Q(x)-\frac{1}{2}\right)\times(\eta^P(x)-\frac{1}{2})\geq C_\gamma|\eta^Q(x)-\frac{1}{2}|^\gamma.$$
\end{itemize}

Putting all verification steps above together, we conclude that $(Q_\sigma,P_\sigma)\in\Pi^{NP}_0$. We finish the proof of this part by applying Assouad's lemma to $(Q_\sigma,P_\sigma),\ \forall \sigma\in\{1,-1\}^m$.

If $\sigma,\sigma'\in\{1,-1\}^m$ differ only at one coordinate, i.e.
$$\sigma_k=-\sigma'_k,\quad \sigma_l=\sigma'_l\ (\forall l\neq k),$$ we have the Hellinger distance bound as
$$
\begin{aligned}
H^2(Q_\sigma,Q_{\sigma'})&=\frac{1}{2}\int \left(\sqrt{\eta^Q_\sigma(X)}-\sqrt{\eta^Q_{\sigma'}(X)}\right)^2+\left(\sqrt{1-\eta^Q_\sigma(X)}-\sqrt{1-\eta^Q_{\sigma'}(X)}\right)^2 dQ_X\\
&=\sum_{l=1}^{m_0}\int_{B(x_{k,l},r_Q)}\frac{w}{\lambda[B(x_{k,l},r_Q)]}\left(\sqrt{\frac{1}{2}+(C_\beta \land (\frac{C_{\beta_P}}{C_\gamma})^{\frac{1}{\gamma}}) r_Q^{\beta}}-\sqrt{\frac{1}{2}-(C_\beta \land (\frac{C_{\beta_P}}{C_\gamma})^{\frac{1}{\gamma}}) r_Q^{\beta}}\right)^2 dx\\
&= \frac{1}{2}m_0w(1-\sqrt{1-2(C^2_\beta \land (\frac{C_{\beta_P}}{C_\gamma})^{\frac{2}{\gamma}}) r_Q^{2\beta}})\\
&\leq (C^2_\beta \land (\frac{C_{\beta_P}}{C_\gamma})^{\frac{2}{\gamma}}) m_0w r_Q^{2\beta},
\end{aligned}
$$
$$
\begin{aligned}
H^2(P_\sigma,P_{\sigma'})&=\frac{1}{2}\int \left(\sqrt{\eta^P_\sigma(X)}-\sqrt{\eta^P_{\sigma'}(X)}\right)^2+\left(\sqrt{1-\eta^P_\sigma(X)}-\sqrt{1-\eta^P_{\sigma'}(X)}\right)^2 dP_X\\
&=\sum_{l=1}^{m_0}\int_{B(x_{k,l},r_Q)}\frac{w}{\lambda[B(x_{k,l},r_Q)]}\left(\sqrt{\frac{1}{2}+C_{\beta_P} r_P^{\beta_P}}-\sqrt{\frac{1}{2}-C_{\beta_P} r_P^{\beta_P}}\right)^2 dx\\
&= \frac{1}{2}m_0w(1-\sqrt{1-2C^2_{\beta_P} r_P^{2\beta_P}})\\
&\leq m_0w C^2_{\beta_P} r_P^{2\beta_P}.
\end{aligned}
$$
By the property of Hellinger distance, we have
$$
\begin{aligned}
H^2(\pr^\sigma_{\mathcal{D}_Q}\times\pr^\sigma_{\mathcal{D}_P},\pr^{\sigma'}_{\mathcal{D}_Q}\times\pr^{\sigma'}_{\mathcal{D}_P})\leq &n_Q H^2(Q_\sigma,Q_{\sigma'})+n_P H^2(P_\sigma,P_{\sigma'})\\
\leq & (C^2_\beta \land (\frac{C_{\beta_P}}{C_\gamma})^{\frac{2}{\gamma}}) m_0w n_Q r_Q^{2\beta}+m_0w C^2_{\beta_P} n_P r_P^{2\beta_P}\\
\leq & c_wc_{m_0}(C^2_\beta \land (\frac{C_{\beta_P}}{C_\gamma})^{\frac{2}{\gamma}}) c_r^{2\beta}+c_wc_{m_0} C^2_{\beta_P} c_r^{2\gamma\beta}\\
\leq & \frac{\sqrt{2}}{4}
\end{aligned}
$$
provided that $c_r$ is small enough. This further indicates that
\begin{equation}
TV(\pr^\sigma_{\mathcal{D}_Q}\times\pr^\sigma_{\mathcal{D}_P},\pr^{\sigma'}_{\mathcal{D}_Q}\times\pr^{\sigma'}_{\mathcal{D}_P})\leq \sqrt{2}H^2(\pr^\sigma_{\mathcal{D}_Q}\times\pr^\sigma_{\mathcal{D}_P},\pr^{\sigma'}_{\mathcal{D}_Q}\times\pr^{\sigma'}_{\mathcal{D}_P})\leq \frac{1}{2}.
\label{eq:prop5part2H2toTV}
\end{equation}
For any empirical classifier $\hat{f}$, we have
$$
\begin{aligned}
    \mathcal{E}_{Q_\sigma}(\hat{f})+\mathcal{E}_{Q_{\sigma'}}(\hat{f})&=2\E_{Q_\sigma}[|\eta^{Q_\sigma}(X)-\frac{1}{2}|\mathbf{1}\{\hat{f}(X)\neq f^*_{Q_\sigma}(X)\}]\\
&+2\E_{Q_{\sigma'}}[|\eta^{Q_{\sigma'}}(X)-\frac{1}{2}|\mathbf{1}\{\hat{f}(X)\neq f^*_{Q_{\sigma'}}(X)\}]\\
&\geq 2\sum_{l=1}^{m_0}\int_{B(x_{k,l},r_Q)}\frac{w}{\lambda[B(x_{k,l},r_Q)]} (C_\beta \land (\frac{C_{\beta_P}}{C_\gamma})^{\frac{1}{\gamma}})m_0w r_Q^{\beta}dx\\
&= 2(C_\beta \land (\frac{C_{\beta_P}}{C_\gamma})^{\frac{1}{\gamma}})m_0w r_Q^{\beta}.
\end{aligned}
$$
Combining this lower bound with \eqref{eq:prop5part2H2toTV}, the Assouad's lemma shows that
$$
\begin{aligned}
\sup_{\substack{(Q,P)\in\Pi^{NP}_0\\\Omega=\Omega_P,\gamma\beta < \beta_P}}\E\mathcal{E}_Q(\hat{f})\geq \sup_{\substack{(Q_\sigma,P_\sigma)\\ \sigma\in\{1,-1\}^m}}\E\mathcal{E}_{Q_\sigma}(\hat{f})&\geq \frac{1}{2}(C_\beta \land (\frac{C_{\beta_P}}{C_\gamma})^{\frac{1}{\gamma}})mm_0w r_Q^{\beta}\\
&\gtrsim r_Q^{\beta(1+\alpha)}\gtrsim  n_P^{-\frac{\beta_P(1+\alpha)/\gamma}{2\beta_P+d}}\land n_Q^{-\frac{\beta_P(1+\alpha)/\gamma}{2\beta_P/\gamma+d}}. 
\end{aligned}
$$
\end{proof}
\subsection{Proof of Lemma \ref{lemma:LRUparaBoundQ} and \ref{lemma:LRUparaBoundP}}
\label{sec:logisticUpperLemmaProof}
\begin{proof}[Proof of Lemma \ref{lemma:LRUparaBoundQ}]
 For any $i=1,\cdots, n_Q$, define the residual as $\varepsilon_i:=\psi'(\beta_Q^TX_i)-Y_i\in[-1,1]$. Define $V_{ij}:=X_{ij}\varepsilon_i$, where $X_{ij}$ is the $j-$th covariate of the $i$-th observation $X_i$.
For any $t\in\mathbb{R}$, we compute the cumulant function
$$
\begin{aligned}
  \log \E_{\mathcal{D}_Q}[\exp(tV_{ij})|X_i]&=\log \left(\E_{\mathcal{D}_Q}[\exp(t)X_{ij}Y_i]\exp(-tX_{ij}\psi'(\beta_Q^TX_i))\right)\\
  &= \psi(tX_{ij}+\beta_Q^TX_i)-\psi(\beta_Q^TX_i)-\psi'(\beta_Q^TX_i)tX_{ij}.
\end{aligned}
$$
Hence, by the second-order Taylor series expansion, we have
\begin{equation}
 \log \E_{\mathcal{D}_Q}[\exp(tV_{ij})|X_i]=\frac{t^2}{2}X_{ij}^2\psi''(\beta_Q^TX_i+\xi_itX_{ij})\leq \frac{t^2}{8}(\frac{1}{n_Q}\sum_{i=1}^{n_Q}X_{ij}^2).
 \label{eq:LRUcumulant}
\end{equation}
for some $\xi_i\in[0,1]$.
Since $\E_{\mathcal{D}_Q}[X_{ij}^2]=1$ and $X_{ij}^2$ is sub-exponential since $X_{ij}\sim N(0,1)$, the tail bound for the independent sum of sub-exponential random variables reads
$$\pr_{\mathcal{D}_Q}(\frac{1}{n_Q}\sum_{i=1}^{n_Q}X_{ij}^2\geq 2)\leq 2\exp(-n_Q/4).$$
We denote the event $E:=\{\max_{j=1,\cdots,d}\frac{1}{n_Q}\sum_{i=1}^{n_Q}X_{ij}^2\leq 2\}$. The union bounds give
\begin{equation}
 \pr_{\mathcal{D}_Q}(E^c)\leq \sum_{j=1}^d\pr_{\mathcal{D}_Q}(\frac{1}{n_Q}\sum_{i=1}^{n_Q}X_{ij}^2\geq 2)\leq 2d\exp(-n_Q/4).
 \label{eq:LRUEcProb}
\end{equation}
By \eqref{eq:LRUcumulant}, we have that on the event $E$,
$$ \log \E_{\mathcal{D}_Q}[\exp(tV_{ij})|X_i]\leq \frac{t^2}{4},$$
and we obtain by the Chernoff bound that
$$\pr_{\mathcal{D}_Q}(|\frac{1}{n_Q}\sum_{i=1}^{n_Q}V_{ij}|\geq t|E)\leq 2\exp(-4n_Qt^2),$$ and the union bound is as follows:
\begin{equation}
\pr_{\mathcal{D}_Q}(\max_{j=1,\cdots,d}|\frac{1}{n_Q}\sum_{i=1}^{n_Q}V_{ij}|\geq t|E)\leq 2d\exp(-4n_Qt^2).
\label{eq:LRUVijBound0}
\end{equation}
We set $t=\frac{\sqrt{K+1}}{2}\sqrt{\frac{\log d}{n_Q}}$, then \eqref{eq:LRUVijBound0} becomes
\begin{equation}
\pr_{\mathcal{D}_Q}(\max_{j=1,\cdots,d}|\frac{1}{n_Q}\sum_{i=1}^{n_Q}V_{ij}|\geq \frac{\sqrt{K+1}}{2}\sqrt{\frac{\log d}{n_Q}}|E)\leq 2d^{-K}.
\label{eq:LRUVijBound1}
\end{equation}
Putting \eqref{eq:LRUEcProb} and \eqref{eq:LRUVijBound1} together, we have
\begin{equation}
\pr_{\mathcal{D}_Q}(\max_{j=1,\cdots,d}|\frac{1}{n_Q}\sum_{i=1}^{n_Q}V_{ij}|\geq \frac{\sqrt{K+1}}{2}\sqrt{\frac{\log d}{n_Q}})\leq 2d^{-K}+2d\exp(-n_Q/4).
\label{eq:LRUVijBoundQ}
\end{equation}
Denote
$$E_Q:=\max_{j=1,\cdots,d}|\frac{1}{n_Q}\sum_{i=1}^{n_Q}V_{ij}|\leq \frac{\sqrt{K+1}}{2}\sqrt{\frac{\log d}{n_Q}},$$
then we have $\pr_{\mathcal{D}_Q}(E_Q^c)\leq 2d^{-K}+2d\exp(-n_Q/4)$.

Define the empirical loss function as $$l_Q(\beta):=\frac{1}{n_Q}\sum_{i=1}^{n_Q} \left\{\log(1+e^{X_i^T\beta})-Y_iX_i^T\beta\right\}.$$
From straightforward calculation,we have
$$\nabla l_Q(\beta)=\frac{1}{n_Q}\sum_{i=1}^{n_Q}(\psi'(\beta^TX_i)-Y_i)X_i\Rightarrow \nabla l_Q(\beta_Q)=\frac{1}{n_Q}\sum_{i=1}^{n_Q}V_i,$$
$$\nabla^2 l_Q(\beta)=\frac{1}{n_Q}\sum_{i=1}^{n_Q}\psi''(\beta^TX_i)X_iX_i^T.$$
where $V_i=(\psi'(\beta_Q^TX_i)-Y_i)X_i$. We see that $l_Q(\beta)$ is convex since $\nabla^2 l_Q(\beta)$ is positive definite. In addition, we have
$$\|\nabla l_Q(\beta_Q)\|_\infty=\max_{j=1,\cdots,d}|\frac{1}{n_Q}\sum_{i=1}^{n_Q}V_{ij}|.$$
Define the error of the first-order Taylor series expansion of $l_Q$ at $\beta_Q$ as
$$\delta l_Q(v)=l_Q(\beta_Q+v)-l_Q(\beta_Q)-(\nabla l_Q(\beta_Q))^Tv$$ for any $v\in\mathbb{R}^d$. 

Define the support of $\beta_Q$ as $S$ with the cardinality $|S|\leq s$. Denote $\{1,\cdots,d\}/S$ by $S^c$. Since we choose $\lambda_Q\geq \sqrt{K+1}\sqrt{\frac{\log d}{n_Q}}$, on the event $E_Q$ we have $\lambda_Q\geq 2\|\nabla l_Q(\beta_Q)\|_\infty$. Proposition 5.3 of \cite{fan2020StatisticalFO} then implies that on the event $E_Q$, we have
$$\|(\hat{\beta}_Q-\beta_Q)_{S^c}\|_1\leq 3 \|(\hat{\beta}_Q-\beta_Q)_{S}\|_1\Rightarrow \|\hat{\beta}_Q-\beta_Q\|\leq 4 \|(\hat{\beta}_Q-\beta_Q)_{S}\|_1.$$

Suppose that $n_Q$ is large enough such that $n_Q\geq 64\kappa_2^2s\log d$, which is feasible as we assumed $n_Q\gg s\log d$. From Proposition 2 of the full version of \cite{negahban2009RSC}, there exists constants $\kappa_1,\kappa_2,c_1,c_2>0$ such that for any $v\in\mathbb{R}^d$ satisfying that $\|v\|_2\leq 1, \|v_{S^c}\|_1\leq 3\|v_{S}\|_1$, we have
\begin{equation}
\begin{aligned}
\delta l_Q(v)&\geq \kappa_1\|v\|^2_2-\kappa_1\kappa_2\sqrt{\frac{\log d}{n_Q}}\|v\|_1\|v\|_2\\
&\geq \kappa_1\|v\|_2(\|v\|_2-\frac{1}{8\sqrt{s}}\|v\|_1)\\
&\geq \kappa_1\|v\|_2(\|v\|_2-\frac{1}{2\sqrt{s}}\|v_{S}\|_1)\\
&\geq \kappa_1\|v\|_2(\|v\|_2-\frac{1}{2}\|v_{S}\|_2)\\
&\geq \kappa_1\|v\|_2(\|v\|_2-\frac{1}{2}\|v\|_2)\\
&=\frac{\kappa_1}{2}\|v\|_2^2
\end{aligned}
 \label{eq:LRURSC}
\end{equation}
with probability at least $1-c_1\exp(-c_2n_Q)$ w.r.t. the distribution of $\mathcal{D}_Q$. See Appendix D.2 of the full version of \cite{negahban2009RSC} to get a clearer view of this statement.

Putting \eqref{eq:LRUVijBoundQ} and \eqref{eq:LRURSC} together, we obtain that with probability at least $1-2d^{-K}-2d\exp(-n_Q/4)-c_1\exp(-c_2n_Q)$ w.r.t. the distribution of $\mathcal{D}_Q$, we have
\begin{equation}
    \delta l_Q(v)\geq \frac{\kappa_1}{2}\|v\|_2^2\quad \forall \|v\|_2\leq 1, \|v\|_1\leq 3\|v_{S}\|_1,
    \label{eq:LRUfan1}
\end{equation}
and
\begin{equation}
    \lambda_Q\geq 2\|\nabla l_Q(\beta_Q)\|_\infty.
    \label{eq:LRUfan2}
\end{equation}
Applying Theorem 1 of the full version of \cite{negahban2009RSC} with \eqref{eq:LRUfan1} and \eqref{eq:LRUfan2}, we have that with probability at least $1-2d^{-K}-2d\exp(-n_Q/4)-c_1\exp(-c_2n_Q)$, 
\begin{equation}
\|\hat{\beta}_Q-\beta_Q\|\leq \frac{4\sqrt{s}}{\kappa_1}\lambda_Q=\frac{4c_Q}{\kappa_1}\sqrt{\frac{s\log d}{n_Q}}.
\label{eq:LRUbetaQBound}
\end{equation}
since $\|(\beta_Q)_{S^c}\|_1=0$. Denote the event $\{\|\hat{\beta}_Q-\beta_Q\|^2\leq \frac{4\sqrt{s}}{\kappa_1}\lambda_Q\}$ by $E_0$, we have $$\pr_{\mathcal{D}_Q}(E_0^c)\leq 2d^{-K}+2d\exp(-n_Q/4)+c_1\exp(-c_2n_Q).$$
By the theorem setting, we have $2d^{-K}\lesssim \frac{s\log d}{n_Q}\land\frac{s\log d}{n_P}$, so it remains to show that the term $2d\exp(-n_Q/4)+c_1\exp(-c_2n_Q)$ is asymptotically less than $\frac{s\log d}{n_Q}\land\frac{s\log d}{n_P}$. Given the condition $\log n_Q\gg d$, we could set $n_Q$ to be large enough such that
$$2d\exp(-n_Q/4)= \exp(-n_Q/4+\log d)\leq \exp(-n_Q/8).$$ Hence, it suffices to show that for any $c>0$, we have $\exp(-cn_Q)\ll \frac{s\log d}{n_Q}\land\frac{s\log d}{n_P}$.
Note that
$$n_Q\gg \log \frac{n_P}{s\log d}\Rightarrow \exp(cn_Q-\frac{n_P}{s\log d})\gg 1\Rightarrow \exp(-cn_Q)\ll \frac{s\log d}{n_P}.$$ Since it is trivial that $\exp(-cn_Q)\ll \frac{s\log d}{n_Q}$, we have $\exp(-cn_Q)\ll \frac{s\log d}{n_Q}\land\frac{s\log d}{n_P}$, which finishes the proof.
\end{proof}

\begin{proof}[Proof of Lemma \ref{lemma:LRUparaBoundP}]
For any $i=1,\cdots, n_P$, define the residual as $$\varepsilon^P_i:=\psi'(\beta_P^TX^P_i)-Y^P_i\in[-1,1].$$ Define $V^P_{ij}:=X^P_{ij}\varepsilon^P_i$, where $X^P_{ij}$ is the $j-$th covariate of the $i$-th observation $X^P_i$. Following the completely identical procedures as given in Lemma \ref{lemma:LRUparaBoundQ}, we have that
\begin{equation}
\pr_{\mathcal{D}_P}(\max_{j=1,\cdots,d}|\frac{1}{n_P}\sum_{i=1}^{n_P}V^P_{ij}|\geq \frac{\sqrt{K+1}}{2}\sqrt{\frac{\log d}{n_P}})\leq 2d^{-K}+2d\exp(-n_P/4).
\label{eq:LRUVijBoundP}
\end{equation}
Denote
$$E_P:=\max_{j=1,\cdots,d}|\frac{1}{n_Q}\sum_{i=1}^{n_P}V_{ij}|\leq \frac{\sqrt{K+1}}{2}\sqrt{\frac{\log d}{n_P}},$$
then we have $\pr_{\mathcal{D}_P}(E_P^c)\leq 2d^{-K}+2d\exp(-n_P/4)$. Similarly, define the empirical loss function as $$l_P(\beta):=\frac{1}{n_P}\sum_{i=1}^{n_P} \left\{\log(1+e^{(X^P_i)^T\beta})-Y^P_i(X^P_i)^T\beta\right\}.$$
Similarly, we have
$$\|\nabla l_P(\beta_Q)\|_\infty=\max_{j=1,\cdots,d}|\frac{1}{n_P}\sum_{i=1}^{n_P}V^P_{ij}|,$$
and
$$\nabla^2 l_P(\beta)=\frac{1}{n_P}\sum_{i=1}^{n_P}\psi''(\beta^TX^P_i)X^P_i(X^P_i)^T.$$
Define $$h=\frac{\|\beta_P\|}{\|\beta_Q\|}\beta_Q,$$ which could be seen as the ``rotated" $\beta_P$ in order to have the sparsity pattern with the same norm. Due to the angle constraint with the parameter $\Delta$ between $\beta_Q$ and $\beta_P$, we have 
$$\|\beta_P-h\|\leq \Delta \|\beta_P\|\leq U\Delta.$$
We expand the first-order Taylor series at $h$, which is different from the classical analysis in Lemma \ref{lemma:LRUparaBoundQ}. Define the error of the first-order Taylor series expansion of $l_P$ at $h$ as
$$\delta l_P(v)=l_P(h+v)-l_P(h)-(\nabla l_P(h))^Tv$$ for any $v\in\mathbb{R}^d$. Define the support of $\beta_Q$ as $S$. Since we choose $\lambda_P\geq \sqrt{K+1}\sqrt{\frac{\log d}{n_P}}$, on the event $E_P$ we have $\lambda_P\geq 2\|\nabla l_P(\beta_P)\|_\infty$. 

Let $F(v)=l_P(h+v)-l(v)+\lambda_P(\|h+v\|_1-\|h\|_1)$ and $\hat{v}=\hat{\beta}_P-h$. Then $F(\hat{v})\leq 0$. It could be observed that on the event $$E_P\cap \{\|\frac{1}{n_P}\sum_{i=1}^{n_P}X^P_i(X^P_i)^T\|_2\leq 2\},$$  we have
\begin{equation}
\begin{aligned}
l_P(h+v)-l(v)&\geq -|\nabla l_P(h)^Tv|\\
&\geq -|\nabla l_P(\beta_P)^Tv|-|(\nabla l_P(\beta_P)-\nabla l_P(h))^Tv|\\
&\geq -\|\nabla l_P(\beta_P)\|_\infty \|v\|_1-\|\nabla l_P(\beta_P)-\nabla l_P(h)\|_2\|v\|_2\\
&\geq -\frac{\lambda_P}{2} \|v\|_1-\max\{\|\nabla^2l_P\|_2\}\cdot\|\beta_P-h\|_2\cdot\|v\|_2\\
&\geq -\frac{\lambda_P}{2} \|v\|_1-\frac{1}{4}\|\frac{1}{n_P}\sum_{i=1}^{n_P}X^P_i(X^P_i)^T\|_2\cdot\|\beta_P-h\|_2\cdot\|v\|_2\\
&\geq -\frac{\lambda_P}{2} \|v\|_1-\frac{U\Delta}{2}\|v\|_2.
\end{aligned}
\label{eq:logisticUpperCrucial}
\end{equation}

The first inequality is due to convexity of $l_p$. The third inequality is because that
$$u^Tv\leq \|u\|_\infty\|v\|_1,\ u^Tv\leq \|u\|_2\|v\|_2$$
for any vector pair $u,v$. The fourth inequality is due to the first-order Taylor expansion for the derivative and the matrix norm inequality $$\|Au\|_2\leq\|A\|_2\|u\|_2$$ for any matrix $A$ and vector $u$.

On the other hand, as $h_S=h$ by definition of $h$, we have
\begin{equation}
\begin{aligned}
\|h+v\|_1-\|h\|_1&= \|h+v_S+v_{S^c}\|_1-\|h\|_1\\
&\geq \|h+v_{S^c}\|_1-\|v_S\|_1-\|h\|_1\\
&= \|h\|_1+\|v_{S^c}\|_1-\|v_S\|_1-\|h\|_1\\
&=\|v_{S^c}\|_1-\|v_S\|_1.
\end{aligned}
\label{eq:logisticUpperCrucial2}
\end{equation}
Combining \eqref{eq:logisticUpperCrucial}, \eqref{eq:logisticUpperCrucial2}, and $F(\hat{v})\leq 0$, we have
\begin{equation}
U\Delta\|\hat{v}\|_2\geq \lambda_P (\|\hat{v}_{S^c}\|_1-3\|\hat{v}_{S}\|_1)\Longrightarrow \|\hat{v}\|_1\leq 4\|\hat{v}_S\|_1+\frac{U\Delta}{\lambda_P}\|\hat{v}\|_2.
\label{eq:logisticUpperCrucial3}
\end{equation}

We apply Proposition 2 of the full version of \cite{negahban2009RSC} again. Suppose that $n_P$ is large enough such that $n_P\geq 64\kappa_2^2s\log d$, which is feasible as we assumed $n_P\gg s\log d$. From Proposition 2 of the full version of \cite{negahban2009RSC}, there exists constants $\kappa_1,\kappa_2,c_1,c_2>0$ such that
\begin{equation}
\begin{aligned}
\delta l_P(\hat{v})&\geq \kappa_1\|\hat{v}\|^2_2-\kappa_1\kappa_2\sqrt{\frac{\log d}{n_P}}\|\hat{v}\|_1\|\hat{v}\|_2\\
&\geq \kappa_1\|\hat{v}\|_2(\|\hat{v}\|_2-\kappa_2\sqrt{\frac{\log d}{n_P}}\|\hat{v}\|_1)\\
&\geq \kappa_1\|\hat{v}\|_2(\|\hat{v}\|_2-4\kappa_2\sqrt{\frac{\log d}{n_P}}\|\hat{v}_S\|_1-\kappa_2\sqrt{\frac{\log d}{n_P}}\frac{U\Delta}{\lambda_P}\|\hat{v}\|_2)\\
&\geq \kappa_1\|\hat{v}\|_2(\|\hat{v}\|_2-\frac{1}{2\sqrt{s}}\|\hat{v}_{S}\|_1-\frac{\kappa_2 U\Delta}{\sqrt{K+1}}\|\hat{v}\|_2)\\
&\geq \kappa_1\|\hat{v}\|_2(\|\hat{v}\|_2-\frac{1}{2}\|\hat{v}_{S}\|_2-\frac{1}{4}\|\hat{v}\|_2)\\
&\geq \kappa_1\|\hat{v}\|_2(\|\hat{v}\|_2-\frac{1}{2}\|\hat{v}\|_2-\frac{1}{4}\|\hat{v}\|_2)\\
&=\frac{\kappa_1}{4}\|\hat{v}\|_2^2
\end{aligned}
\label{eq:LRURSC2}
\end{equation}
with probability at least $1-c_1\exp(-c_2n_Q)$ w.r.t. the distribution of $\mathcal{D}_Q$. The fourth inequality is due to
$n_P\geq 64\kappa_2^2s\log d$,
and
$\lambda_P\geq \sqrt{K+1}\sqrt{\frac{\log d}{n_P}}$. Recall that in the lemma we assumed that $\Delta$ is smaller than a constant, and in \eqref{eq:LRURSC2} we specifically assume that $\frac{\kappa_2 U\Delta}{\sqrt{K+1}} \leq \frac{1}{4}$, i.e., $\Delta\leq \frac{\sqrt{K+1}}{4\kappa_2 U}$.

From the proof of Theorem 1 of \cite{bickel2008sampleCovariance}, by applying Lemma A.3 of \cite{bickel2008sampleCovariance} and the union bound, we have
$$\pr_{\mathcal{D}_P}(\|\frac{1}{n_P}\sum_{i=1}^{n_P}X^P_i(X^P_i)^T-I_d\|\geq t)\leq 3d \exp(c_3 n t^2)$$ for some constant $c_3>0$ and any $t>0$. Hence,
\begin{equation}
\pr_{\mathcal{D}_P}(\|\frac{1}{n_P}\sum_{i=1}^{n_P}X^P_i(X^P_i)^T\|\> 2)\leq \pr_{\mathcal{D}_P}(\|\frac{1}{n_P}\sum_{i=1}^{n_P}X^P_i(X^P_i)^T-I_d\|\geq 1) \leq 3d \exp(c_3 n).
\label{eq:LRUSampleCovarianceBoundP}
\end{equation}

Putting \eqref{eq:LRUVijBoundP}, \eqref{eq:LRURSC2}, and \eqref{eq:LRUSampleCovarianceBoundP} together, we obtain that with probability at least $1-2d^{-K}-2d\exp(-n_P/4)-c_1\exp(-c_2n_P)-3d \exp(c_3 n)$ w.r.t. the distribution of $\mathcal{D}_P$, we have
\begin{equation}
    \delta l_P(\hat{v})\geq \frac{\kappa_1}{4}\|\hat{v}\|_2^2
    \label{eq:LRUfan3}
\end{equation}
and
\begin{equation}
    \lambda_P\geq 2\|\nabla l_P(\beta_Q)\|_\infty.
    \label{eq:LRUfan4}
\end{equation}
With a similar analysis as the one in Lemma \ref{lemma:LRUparaBoundQ} based on the theorem setting, we have $2d^{-K}+2d\exp(-n_P/4)+c_1\exp(-c_2n_P)+3d \exp(c_3 n)\leq C_P \left(\frac{s\log d}{n_Q}\land \frac{s\log d}{n_P}\right)$. Therefore, it remains to show that $\|\hat{\beta}_P-\beta_P\|_2\lesssim \sqrt{s}\lambda_P+\Delta$ when \eqref{eq:LRUfan3} and \eqref{eq:LRUfan4} hold.

Applying \eqref{eq:logisticUpperCrucial} and \eqref{eq:logisticUpperCrucial2} again, we have
$$
\begin{aligned}
0\geq F(\hat{v})&\geq \nabla l_P(h)^T\hat{v}+\frac{\kappa_1}{4}\|\hat{v}\|_2^2+\lambda_P(\|h+\hat{v}\|_1-\|h\|_1)\\
&\geq  \frac{\kappa_1}{4}\|\hat{v}\|_2^2-\frac{\lambda_P}{2} \|v\|_1-\frac{U\Delta}{2}\|v\|_2 +\lambda_P(\|v_{S^c}\|_1-\|v_S\|_1)\\
&= \frac{\kappa_1}{4}\|\hat{v}\|_2^2-\frac{\lambda_P}{2} (\|v_{S^c}\|_1+\|v_S\|_1)-\frac{U\Delta}{2}\|v\|_2 +\lambda_P(\|v_{S^c}\|_1-\|v_S\|_1)\\
&\geq \frac{\kappa_1}{4}\|\hat{v}\|_2^2-\frac{3\lambda_P}{2} \|v_S\|_1-\frac{U\Delta}{2}\|v\|_2\\
&\geq \frac{\kappa_1}{4}\|\hat{v}\|_2^2-\frac{3\sqrt{s}\lambda_P}{2} \|v_S\|_2-\frac{U\Delta}{2}\|v\|_2\\
&\geq \frac{\kappa_1}{4}\|\hat{v}\|_2^2-\left(\frac{3\sqrt{s}\lambda_P}{2} +\frac{U\Delta}{2}\right)\|v\|_2
\end{aligned}
$$
given \eqref{eq:LRUfan3} and \eqref{eq:LRUfan4}, which implies $$\|\hat{v}\|_2\leq \frac{4}{\kappa_1}(\frac{3\sqrt{s}\lambda_P}{2} +\frac{U\Delta}{2})\lesssim \sqrt{s}\lambda_P+\Delta.$$ The proof is finished by observing that
$$\|\hat{\beta}_P-\beta_P\|_2\leq \|\hat{v}\|_2+\|h-\beta_P\|_2\leq \|\hat{v}\|_2+ U\Delta\lesssim \sqrt{s}\lambda_P+\Delta.$$
\end{proof}

\subsection{Proofs of Auxiliary Results}
\begin{lemma}
    If the distribution pair $(Q,P)$ belongs to $\Pi_{BA}^{NP}$ defined in Section \ref{sec:nonPara} with corresponding parameters, then $(Q,P)$ satisfies Assumption \ref{assum:ambiguity} with
    $$\varepsilon(z;\gamma,C_\gamma/2)=\left(C_\alpha z^{1+\alpha}\right)\land \left(2^{\frac{1+\alpha}{\gamma}}C_\alpha C_\gamma^{-\frac{1+\alpha}{\gamma}}\Delta^{\frac{1+\alpha}{\gamma}}\right).$$
    \label{lemma:BAtoAPB}
\end{lemma}
\begin{proof}
Suppose $(Q,P)\in\Pi_{BA}^{NP}$. If $x\in\Omega$ satisfies $|\eta^Q(x)-\frac{1}{2}|\geq (\frac{2\Delta}{C_\gamma})^{\frac{1}{\gamma}}$, then by the definition of $\Pi_{BA}^{NP}$ we have
$$\Delta\leq \frac{C_\gamma}{2}|\eta^Q(x)-\frac{1}{2}|^\gamma,$$
$$\Longrightarrow s(x)\geq C_\gamma|\eta^Q(x)-\frac{1}{2}|^\gamma-\Delta\geq \frac{C_\gamma}{2}|\eta^Q(x)-\frac{1}{2}|^\gamma.$$
which indicates that $x\in\Omega^+(\gamma,C_\gamma/2)$ and $\Omega^-(\gamma,C_\gamma/2)\subset\{x\in\Omega:|\eta^Q(x)-\frac{1}{2}|< (\frac{2\Delta}{C_\gamma})^{\frac{1}{\gamma}}\}$. Therefore,
$$
\begin{aligned}
    & Q(X\in\Omega^-(\gamma,C_\gamma/2),0<|\eta^Q(X)-\frac{1}{2}|\leq z)\\
\leq& Q(0<|\eta^Q(X)-\frac{1}{2}|\leq (\frac{2\Delta}{C_\gamma})^{\frac{1}{\gamma}}\land z)\\
\leq& C_\alpha((\frac{2\Delta}{C_\gamma})^{\frac{\alpha}{\gamma}}\land z^{\alpha}).
\end{aligned}
$$
Moreover, since $$\sup_{\substack{x\in\Omega^-(\gamma,C_\gamma),|\eta^Q(x)-\frac{1}{2}|\leq z}}|\eta^Q(x)-\frac{1}{2}|\leq (\frac{2\Delta}{C_\gamma})^{\frac{1}{\gamma}}\land z,$$ we have
$$
\begin{aligned}
&\E_{(X,Y)\sim Q}\left[|\eta^Q(X)-\frac{1}{2}|\mathbf{1}\{s(X)\leq C_\gamma|\eta^Q(X)-\frac{1}{2}|^\gamma\leq C_\gamma z^\gamma\}\right]\\
   \leq &\left(\sup_{\substack{x\in\Omega^-(\gamma,C_\gamma),|\eta^Q(x)-\frac{1}{2}|\leq z}}|\eta^Q(x)-\frac{1}{2}|\right) \cdot Q(X\in\Omega^-(\gamma,C_\gamma/2),0<|\eta^Q(X)-\frac{1}{2}|\leq z) \\
   \leq &C_\alpha\left((\frac{2\Delta}{C_\gamma})^{\frac{\alpha}{\gamma}}\land z^{\alpha}\right)\cdot \left((\frac{2\Delta}{C_\gamma})^{\frac{1}{\gamma}}\land z\right).
\end{aligned}
$$
Therefore, Assumption \ref{assum:ambiguity} holds with
    $$
    \begin{aligned}
    \varepsilon(z;\gamma,C_\gamma/2)&=C_\alpha\left((\frac{2\Delta}{C_\gamma})^{\frac{\alpha}{\gamma}}\land z^{\alpha}\right)\cdot \left((\frac{2\Delta}{C_\gamma})^{\frac{1}{\gamma}}\land z\right)\\
    &=\left(C_\alpha z^{1+\alpha}\right)\land \left(2^{\frac{1+\alpha}{\gamma}}C_\alpha C_\gamma^{-\frac{1+\alpha}{\gamma}}\Delta^{\frac{1+\alpha}{\gamma}}\right).
    \end{aligned}$$

\end{proof}

\begin{lemma}
Suppose there exists a continuous function $p(\cdot;\gamma,C_\gamma):[0,1]\rightarrow[0,1]$ such that for any $\eta\in[0,1]$,
$$Q(\mathrm{sgn}\left(\eta-\frac{1}{2}\right)\times(\eta^P-\frac{1}{2})\geq C_\gamma|\eta-\frac{1}{2}|^\gamma|\eta^Q=\eta)\leq p(\eta;\gamma,C_\gamma).$$
Then Assumption \ref{assum:ambiguity} holds with
$$
\begin{aligned}
\varepsilon(z;\gamma,C_\gamma)&=\int_{\frac{1}{2}-z}^{\frac{1}{2}+z}|\eta-\frac{1}{2}|p(\eta;\gamma,C_\gamma)dF^Q_{\eta}(\eta)\leq C_\alpha z^{1+\alpha}\times\sup_{\eta\in[\frac{1}{2}-z,\frac{1}{2}+z]}p(\eta;\gamma,C_\gamma)
\end{aligned}
$$
where $F^Q_{\eta}$ is defined as the cumulative distribution function of $\eta^Q$ w.r.t. $Q_X$.
\label{lemma:CAPD}
\end{lemma}
\begin{proof}
The first equality holds by observing the transform using Funibi's Theorem:
 $$
 \begin{aligned}
      &\int_{\Omega^-(\gamma,C_\gamma)}|\eta^Q(X)-\frac{1}{2}|\mathbf{1}\{0<|\eta^Q(X)-\frac{1}{2}|\leq z\}dQ_X\\
      =&\int_{\frac{1}{2}-z}^{\frac{1}{2}+z}|\eta-\frac{1}{2}|Q_{\eta^P}(\mathrm{sgn}\left(\eta-\frac{1}{2}\right)\times(\eta^P-\frac{1}{2})\geq C_\gamma|\eta-\frac{1}{2}|^\gamma|\eta^Q=\eta)dF^Q_{\eta}(\eta),
 \end{aligned}
$$
where $Q_{\eta^P}$ is the marginal distribution of $\eta^P$ with respect to $Q$. The second equality holds by the provided lemma condition.
\end{proof}

We present a lemma that establishes a high probability uniform bound on the distance between any point and its $K$-nearest neighbors. This result improves upon Lemma 9.1 in \cite{ttcai2020classification} by providing a tighter bound that leverages the Hoeffding's inequality, and the rest of our proof is similar. The proof follows a similar approach, so we only provide a more concise presentation here.
\begin{lemma}[$K$-NN Distance Bound]
There exists a constant $c_D>0$ such that with probability at least $1-c_D\frac{n_Q}{k_Q}\exp(-2k_Q)$ w.r.t. the distribution of $X_{1:n_Q}$, for all $x\in\Omega\subset[0,1]^d$,
\begin{equation}
\|X_{(k_Q)}(x)-x\|\leq c_D(\frac{k_Q}{n_Q})^{\frac{1}{d}}.
\label{eq:knnBoundQ}
\end{equation}
In addition, with probability at least $1-c_D\frac{n_P}{k_P}\exp(-2k_P)$ w.r.t. the distribution of $X^P_{1:n_P}$, for all $x\in\Omega^P\subset[0,1]^d$,
\begin{equation}
\|X^P_{(k_P)}(x)-x\|\leq c_D(\frac{k_P}{n_P})^{\frac{1}{d}}.
\label{eq:knnBoundP}
\end{equation}
\label{lemma:knnBound}
\end{lemma}
\begin{proof}
It suffices to prove \eqref{eq:knnBoundQ} since the proof of \eqref{eq:knnBoundP} can be obtained by symmetrically replacing the relevant quantities.

Let $(Q,P)\in\Pi^{NP}$ and take any $x\in\Omega$ and $r<r_\mu$. Since $\frac{dQ_X}{d\lambda}\geq \mu^-$, we have
\begin{equation}
    Q(X\in B(x,r)) \geq \mu^- \lambda (B(x,r)\cap \Omega)\geq c_\mu \mu^- \pi_d r^d,
    \label{eq:knnDensityLB}
\end{equation}
 where $\pi_d=\lambda(B(0,1))$ is the volume of a $d$-dimensional unit sphere.

If $k_Q \geq (\frac{1}{4}\land \frac{c_\mu \mu^- \pi_d r_\mu^d}{2})$, then $c_D(\frac{k_Q}{n_Q})^{\frac{1}{d}}\geq c_D (\frac{1}{4})^{\frac{1}{d}}$, and we can set $c_D$ to be large enough to satisfy \eqref{eq:knnBoundQ} using the trivial bound $\|X_{(k_Q)}(x)-x\|\leq d^{\frac{1}{2}}$. Therefore, we only need to consider the case when $k_Q\leq (\frac{1}{4}\land \frac{c_\mu\mu^-\pi_d r_\mu^d}{2})$.

Set $r_0=(\frac{2k_Q}{c_\mu \mu^-\pi_dn_Q})^{\frac{1}{d}}$. From $k_Q\leq \frac{c_\mu\mu^-\pi_d r_\mu^d}{2} n_Q$ we have $r_0< r_\mu$. Therefore, \eqref{eq:knnDensityLB} tells
$$Q(X\in B(x,r_0))\geq \frac{2k_Q}{n_Q}.$$

Let $S(x)=\sum_{i=1}^{n_Q}\mathbf{1}\{X_i\in B(x,r_0)\}$. Since $X_1,\cdots,X_{n_Q}$ are independent, $S$ follows a binomial distribution with parameters $n_Q$ and $\frac{2k_Q}{n_Q}$. By Hoeffding's inequality,
\begin{equation}
\pr_{X_{1:n_Q}}(S(x)<k_Q)=\pr_{X_{1:n_Q}}(S-\E_{\mathcal{D}_Q}[S]<-k_Q)\leq \exp(-\frac{2k_Q^2}{k^Q})=\exp(-2k_Q).
\label{eq:knnHoeffding}
\end{equation}
Suppose that $M$ balls with radius $r_0$ centered at $x_1,\cdots,x_M$ satisfy that 
$$[0,1]^d\subset \bigcup_{m=1}^MB(x_m,r_0).$$ It is feasible to find such $M$ balls with $M\leq Cr_0^{-d}$ for some $C>0$ large enough. By the union bound,
\begin{equation}
\pr_{X_{1:n_Q}}(\min_{1\leq m\leq M}\{S(x_m)\}<k_Q)\leq M\exp(-2k_Q)\leq Cr_0^{-d}\exp(-2k_Q).
\label{eq:knnHoeffdingUnion}
\end{equation}

For any $x\in\Omega$, there exists some $1\leq m'\leq M$ such that $x_{m'}\in B(x,r_0)$. Note that $S(x_{m'})\geq k_Q$ implies that $$\|X_{(k_Q)}(x)-x\|\leq 2r_0.$$ Therefore, we have $$\pr_{X_{1:n_Q}}(\forall x\in\Omega,\|X_{(k_Q)}(x)-x\|\leq 2r_0)\geq \pr_{X_{1:n_Q}}(\min_{1\leq m\leq M}\{S(x_m)\}\geq k_Q)\geq 1- Cr_0^{-d}\exp(-2k_Q),$$ i.e., with probability at least $1-\frac{2C}{c_\mu\mu^-\pi_d}\frac{n_Q}{k_Q}\exp(-2k_Q)$ w.r.t. the distribution of $X_{1:n_Q}$, we have $$\|X_{(k_Q)}(x)-x\|\leq 2(\frac{2}{c_\mu\mu^-\pi_d})^{\frac{1}{d}}(\frac{k_Q}{n_Q})^{\frac{1}{d}}.$$ This completes the proof by setting $c_D$ large enough.
\end{proof}

\begin{lemma}
For any empirical classifier $\hat{f}$ and $\alpha,\beta\in\mathbb{R}^d$ such that $\|\alpha\|$ and $\|\beta\|$ are bounded between $c$ and $C$ for some constant $C>c>0$, and $\angle ( \alpha,\beta)\in[0,\pi/2]$. Then we have $$\mathcal{E}_{Q_\alpha}(\hat{f})+\mathcal{E}_{Q_\beta}(\hat{f})\geq \frac{c\sigma(C)(1-\sigma(C))}{20\pi}\angle (\alpha,\beta)^2.$$
\label{lemma:risk2angle}
\end{lemma}
\begin{proof}
Without loss of generality and due to the symmetry property of of $N(0,I_d)$, we could rotate $\alpha$ and $\beta$ at the same time so that
$$\alpha=(\|\alpha\|,0,0,\cdots, 0),\quad \beta=(\|\beta\|\cos \angle (\alpha,\beta),\|\beta\|\sin \angle (\alpha,\beta),0,0,\cdots,0).$$ Therefore, we assume that only the first coordinate of $\alpha$ and the first and second coordinates of $\beta$ can be non-zero.

    For any $\lambda\in[0,1]$, we define $w_\lambda:=\lambda\alpha+(1-\lambda)\beta$. Define $\lambda_0,\lambda_1\in[0,1]$ as the value satisfying that
    $$\angle ( \alpha,w_{\lambda_0})=\angle ( w_{\lambda_1},\beta)=\angle ( \alpha,\beta)/4.$$ For simplicity, we abbreviate $w_{\lambda_0}$ as $w_0$ and $w_{\lambda_1}$ as $w_1$.
Define the area $$D:=\{x\in\mathbb{R}^d: (w_0^Tx)(w_1^Tx)<0,\frac{1}{4}\leq x_1^2+x_2^2\leq 1\}.$$ It is easy to see that $D\subset\{x\in\mathbb{R}^d: (\alpha^Tx)(\beta^Tx)<0,\frac{1}{2}\leq x_1^2+x_2^2\leq 1\}$. Moreover, we see that if $x\in D$, we have 
$$|\alpha^T x|=\|\alpha\|x_1\geq \|\alpha\|\sin(\frac{\angle (\alpha,\beta)}{4})(x^1+x^2)^{\frac{1}{2}},$$
$$|\beta^Tx|\geq \|\beta\|\sin(\frac{\angle (\alpha,\beta)}{4})(x^1+x^2)^{\frac{1}{2}}$$ since the angle between $(x_1,x_2,0\cdots,0)$ and normal planes of $\alpha$ and $\beta$ are both in $[\frac{\angle (\alpha,\beta)}{4},\frac{3\angle (\alpha,\beta)}{4}]$. Therefore, since $\sigma'(t)=\sigma(t)(1-\sigma(t))$ decreases with growing $|t|$, and $|\alpha^T x|\leq \|\alpha\|\leq C$ on $x\in D$, we have that if $x\in D$.
$$
\begin{aligned}
  |\eta^{Q_\alpha}(x)-\frac{1}{2}|&\geq |\sigma(\sin(\frac{\angle (\alpha,\beta)}{4})\|\alpha\|(x_1^2+x_2^2)^\frac{1}{2})-\frac{1}{2}|\\
  &\geq \sigma(C)(1-\sigma(C))\sin(\frac{\angle (\alpha,\beta)}{4})\|\alpha\|(x_1^2+x_2^2)^\frac{1}{2}\\
  &\geq \frac{1}{2}c\sigma(C)(1-\sigma(C))\sin(\frac{\angle (\alpha,\beta)}{4}).
\end{aligned}
$$
Similarly,
$$|\eta^{Q_\beta}(x)-\frac{1}{2}|\geq \frac{1}{2}c\sigma(C)(1-\sigma(C))\sin(\frac{\angle (\alpha,\beta)}{4}).$$
Hence, we have
$$
\begin{aligned}
\mathcal{E}_{Q_\alpha}(\hat{f})+\mathcal{E}_{Q_\beta}(\hat{f})&=2\E_{Q_\alpha}[|\eta^{Q_\alpha}(X)-\frac{1}{2}|\mathbf{1}\{(\hat{f}-\frac{1}{2})(\alpha^Tx)<0\}]\\
&+2\E_{Q_\beta}[|\eta^{Q_\beta}(X)-\frac{1}{2}|\mathbf{1}\{(\hat{f}-\frac{1}{2})(\beta^Tx)<0\}]\\
&\geq 2\E_{Q_\alpha}[|\eta^{Q_\alpha}(X)-\frac{1}{2}|\mathbf{1}\{(\hat{f}-\frac{1}{2})(\alpha^Tx)<0,x\in D\}]\\
&+2\E_{Q_\beta}[|\eta^{Q_\beta}(X)-\frac{1}{2}|\mathbf{1}\{(\hat{f}-\frac{1}{2})(\beta^Tx)<0,x\in D\}]\\
&\geq c\sigma(C)(1-\sigma(C))\sin(\frac{\angle (\alpha,\beta)}{4})\times\\
&\E_{X\sim N(0,I_d)}[\left(\mathbf{1}\{(\hat{f}-\frac{1}{2})(\alpha^Tx)<0\}+\mathbf{1}\{(\hat{f}-\frac{1}{2})(\beta^Tx)<0\}\right)\mathbf{1}\{x\in D\}].
\end{aligned}
$$
The change of the expectation operator in the last inequality is due to the fact that $|\eta^{Q_\alpha}(X)-\frac{1}{2}|\mathbf{1}\{(\hat{f}-\frac{1}{2})(\alpha^Tx)<0\}$ and $|\eta^{Q_\beta}(X)-\frac{1}{2}|\mathbf{1}\{(\hat{f}-\frac{1}{2})(\beta^Tx)<0\}$ do not depend on the distribution of the response value. 

If $x\in D$, then $(\alpha^Tx)(\beta^Tx)<0$, which implies $$\mathbf{1}\{(\hat{f}-\frac{1}{2})(\alpha^Tx)<0\}+\mathbf{1}\{(\hat{f}-\frac{1}{2})(\beta^Tx)<0\}=1.$$
Therefore, we could further bound $\mathcal{E}_{Q_\alpha}(\hat{f})+\mathcal{E}_{Q_\beta}(\hat{f})$ by
$$
\begin{aligned}
    \mathcal{E}_{Q_\alpha}(\hat{f})+\mathcal{E}_{Q_\beta}(\hat{f})&\geq c\sigma(C)(1-\sigma(C))\sin(\frac{\angle (\alpha,\beta)}{4}) Q_X(D)\\
    &\geq c\sigma(C)(1-\sigma(C))\sin(\frac{\angle (\alpha,\beta)}{4}) \angle (\alpha,\beta) Q(X_1^2+X_2^2\in[\frac{1}{4},1])\\
    &\geq \frac{c\sigma(C)(1-\sigma(C))}{20\pi}\angle (\alpha,\beta)^2.
\end{aligned}
$$
where $\Psi(\cdot)$ is the cumulative distribution function of the univariate standard normal distribution. Note that we utilized the numerical results that
$$\sin(\frac{\angle (\alpha,\beta)}{4})\geq \frac{1}{2\pi}\angle (\alpha,\beta),$$
and the cumulative distribution function of chi-squared distribution with $2$ degrees of freedom
$$Q(X_1^2+X_2^2\in[\frac{1}{4},1])=e^{-\frac{1}{2}}-e^{-\frac{1}{8}}\geq \frac{1}{10}.$$
\end{proof}
\end{appendix}
\end{document}